\documentclass[twoside,11pt]{article}

\usepackage{blindtext}
\usepackage{jmlr2e-edit}

\usepackage{amsmath}
\usepackage{natbib}
\usepackage{booktabs}
\usepackage{graphicx}
\usepackage{hyperref}
\newcommand{\soff}{s_{\mathrm{off}}}

\usepackage{lastpage}
\begin{document}

\title{Fractals made Practical:\\
Denoising Diffusion as Partitioned Iterated Function Systems}

\author{\name Ann Dooms \email Ann.Dooms@vub.be \\
       \addr Department of Mathematics \& Data Science\\
       Vrije Universiteit Brussel\\
       Pleinlaan 2, 1050 Brussels, BELGIUM
       }
\editor{}

\maketitle

\begin{abstract}
What is a diffusion model actually doing when it turns noise into a photograph?

We show that the deterministic DDIM reverse chain operates as a Partitioned Iterated Function System (PIFS) and that this framework serves as a unified design language for denoising diffusion model schedules, architectures, and training objectives. From the PIFS structure we derive three computable geometric quantities: a per-step contraction threshold $L^*_t$, a diagonal expansion function $f_t(\lambda)$ and a global expansion threshold $\lambda^{**}$. These quantities require no model evaluation and fully characterise the denoising dynamics. They structurally explain the two-regime behaviour of diffusion models: global context assembly at high noise via diffuse cross-patch attention and fine-detail synthesis at low noise via patch-by-patch suppression release in strict variance order. Self-attention emerges as the natural primitive for PIFS contraction. The Kaplan-Yorke dimension of the PIFS attractor is determined analytically through a discrete Moran equation on the Lyapunov spectrum. 

Through the study of the fractal geometry of the PIFS, we derive three optimal design criteria and show that four prominent empirical design choices (the cosine schedule offset, resolution-dependent logSNR shift, Min-SNR loss weighting, and Align Your Steps sampling) each arise as approximate solutions to our explicit geometric optimisation problems tuning theory into practice.

\end{abstract}

\medskip

\newpage 

\section{Introduction}
\label{sec:introduction}

Modern score-based diffusion models construct high-quality  images through a sequential denoising process. The theoretical foundation lies in continuous-time stochastic differential equations (SDEs) or probability flow ODEs, which provide powerful global guarantees of distributional convergence in $\mathcal{W}_2$ distance \citep{song2020score,ho2020ddpm}. This continuous perspective is elegant, but it treats the learned score network as a black-box, giving no real  structural insight in how a discrete sampling chain assembles global spatial context at early steps and synthesises localised fine detail at later ones and why is self-attention so effective as a generative primitive?

This paper answers both questions by establishing that the deterministic DDIM reverse chain $\Phi$ operates as a \emph{Partitioned Iterated Function System} (PIFS) introduced in \citet{jacquin1992image}. 
We show that trained diffusion models implicitly learn such a composition to reconstruct the data manifold and that the structure of this PIFS is directly responsible for both the two-regime dynamics proved by \citet{raya2023spontaneous} and the effectiveness of self-attention. The paper is organised around four contributions:

\begin{itemize}
    \item \textbf{Contraction structure (Section~\ref{sec:contraction}).}
    We derive two contractivity conditions on the single-step DDIM operator $\Phi_t$: an Euclidean condition~(EC) governed by a closed-form schedule threshold $L_t^*$ and a block-max condition~(PC) that separates diagonal from cross-patch behaviour. We show that score-matching training is the diffusion analogue of Barnsley's collage minimisation and establish an $L^2$--$\mathcal{W}_1$ bridge quantifying how training loss controls the Wasserstein distance to the PIFS fixed point. We propose a PIFS regulariser that directly enforces~(PC) during training.

    \item \textbf{Two-regime structure (Section~\ref{sec:regime_structure}).}
    The contraction study allows us to compute for each step $t$ constants $\kappa_t^\mathrm{diag}$ and $\delta_t^\mathrm{cross}$ from data statistics and architectural properties that structurally explain two qualitatively distinct phases of denoising. In the high-$t$ regime (Regime~I), 
    diffuse attention drives strong cross-patch coupling while a learned directional ``suppression field'' $S_{k,t}$ holds every diagonal block below a so-called expansion threshold. At the highest noise levels 
    the operator satisfies~(EC) with analytically computable constants. In the low-$t$ regime (Regime~II), 
    attention localises and suppression is released patch by patch in strict variance order, giving each release a precise geometric signature. The attention architecture is shown to control $\delta_t^\mathrm{cross}$ through a bound on the softmax weights, explaining why self-attention is the natural primitive for this PIFS structure.

    \item \textbf{Attractor geometry (Section~\ref{sec:attractor}).} We characterise the fractal dimension of the PIFS attractor via its Lyapunov spectrum. Under Gaussian data, the spectrum is fully determined by the per-step diagonal expansion function $f_t(\lambda)$ from which the Kaplan-Yorke dimension of the PIFS attractor is derived. The global expansion threshold $\lambda^{**}$ is the unique solution to the discrete Moran equation. The dimension of the attractor is positive if and only if the data contains at least one patch direction exceeding this threshold. For real data, the suppression field shifts the effective expansion as $f_t(\lambda_k) - S_{k,t}$, shrinking the set of expanding directions and reducing the attractor dimension.

\newpage
    \item \textbf{Practical implications (Section~\ref{sec:schedule_criteria}).}
The PIFS framework leads to the contraction threshold $L_t^*$, the per-step diagonal expansion function $f_t(\lambda)$ and the global expansion threshold $\lambda^{**}$ 
which require no model evaluation. They constitute a design language for noise schedules and training objectives. The cosine schedule offset of \citet{nichol2021improved}, the resolution-dependent logSNR shift of \citet{hoogeboom2023simple}, Min-SNR loss weighting \citep{hang2023efficient}, and the Align Your Steps sampling schedule \citep{sabour2024ays} each emerge as approximate solutions to one of three explicit geometric optimisation problems identified by this language.
\end{itemize}

\section{Background}\label{sec:background}

\subsection{Score-Based Diffusion Models}
GenAI evolves around building generative models that can produce new samples from an unknown data distribution $p_{\text{data}}(x)$ on $\mathbb{R}^n.$ The approach centres on learning the so-called score function and use it to guide a sampling process.

\begin{definition}
For a probability density function $p(x)$ on $\mathbb{R}^n$, the \emph{score function} is defined as 
$$s(x):=\nabla_x \log p(x) \in \mathbb{R}^n.$$
\end{definition}
The function points in the direction of steepest increase of log-probability. If we know s(x), we can generate samples from $p(x)$ using \emph{Langevin dynamics}. Starting from an arbitrary initial point \(x_0\), iterate
\[
x_{k+1} = x_k + \epsilon \, s(x_k) + \sqrt{2\epsilon} \, z_k,\qquad z_k \sim \mathcal{N}(0,I),
\]
where \(\epsilon > 0\) is a step size. Under mild conditions, as \(k\to\infty\) and \(\epsilon\to 0\), the distribution of \(x_k\) converges to \(p(x)\). This is a Markov chain Monte Carlo method.

But we do not know \(p_{\text{data}}\), so we cannot compute its score directly. We must estimate \(s(x)\) from data.

Let \(\{x_0^{(i)}\}_{i=1}^N\) be i.i.d. samples from \(p_{\text{data}}\). We wish to train a neural network \(s_\theta(x)\) (with parameters \(\theta\)) to approximate \(\nabla_x \log p_{\text{data}}(x)\). A naive approach would minimise
\[
\mathbb{E}_{p_{\text{data}}}\bigl[ \| s_\theta(x_0) - \nabla_x \log p_{\text{data}}(x_0) \|^2 \bigr],
\]
but this requires the true score. \emph{Score matching} \cite{hyvarinen2005score} provides an equivalent objective that avoids the unknown density:

\begin{theorem}[Score matching]
If \(s_\theta\) is differentiable and certain regularity conditions hold, then
$$\mathbb{E}_{p_{\text{data}}}\bigl[ \| s_\theta(x_0) - \nabla_x \log p_{\text{data}}(x_0) \|^2 \bigr] = \text{constant} + \mathbb{E}_{p_{\text{data}}}\bigl[ \| s_\theta(x_0) \|^2 + 2 \nabla_x \cdot s_\theta(x_0) \bigr].$$
\end{theorem}
The divergence term \(\nabla_x \cdot s_\theta\) can be expensive in high dimensions, so alternative methods are used in practice.

\emph{Denoising Score Matching (DSM)} by \cite{vincent2011connection} is a popular and tractable variant. Instead of directly targeting \(p_{\text{data}}\), we consider a perturbed distribution. In diffusion-based generative models, this perturbation is defined by a forward process that gradually adds noise to the data. For a sequence of timesteps \(t = 1,\dots,T\), we define the noised version of a clean image \(x_0\) as
\begin{equation}
x_t = \sqrt{\bar{\alpha}_t}\, x_0 + \sqrt{1 - \bar{\alpha}_t}\,\epsilon,\quad \epsilon \sim \mathcal{N}(0,I), \label{eq:forward}
\end{equation}
where \(\bar{\alpha}_t = \prod_{s=1}^t \alpha_s\) is the cumulative noise schedule, strictly decreasing from \(\bar{\alpha}_0 = 1\) (clean image) toward \(\bar{\alpha}_T \approx 0\) (pure noise). The residual noise variance at step \(t\) is \(v_t := 1 - \bar{\alpha}_t\). 
The specific choice of \(\{\bar\alpha_t\}\) defines the forward process. Two widely used schedules are the \emph{linear schedule} of \citet{ho2020ddpm} and the \emph{cosine schedule} of \citet{nichol2021improved}. 

The linear schedule sets \(\bar\alpha_t = \prod_{s=1}^t (1-\beta^{\text{lin}}_s)\) with \(\beta^{\text{lin}}_s\) increasing linearly from \(\beta^{\text{lin}}_1 = 10^{-4}\) to \(\beta^{\text{lin}}_T = 0.02\). This keeps \(\bar\alpha_t\) nearly constant at low noise (\(t\) near \(1\)) and drops sharply near \(t = T\), placing most of the signal-to-noise transition in the final steps.

The cosine schedule instead defines
\begin{equation}
\bar\alpha_t = \frac{f(t)}{f(0)},\qquad f(t) = \cos^2\left(\frac{t/T + \soff}{1+\soff}\cdot\frac{\pi}{2}\right), \label{eq:cosine_schedule}
\end{equation}
with offset parameter \(\soff \geq 0\); the default value \(\soff = 0.008\) is used in practice. The offset flattens the approach of \(\bar\alpha_t\) toward 1 near \(t = 0\), preventing the schedule from decaying too rapidly at the low-noise end. The log‑signal‑to‑noise ratio at step \(t\) is
\begin{equation}
\log\text{SNR}_t := \log\frac{\bar\alpha_t}{1-\bar\alpha_t}.
\label{eq:logsnr}
\end{equation}

We train a time‑dependent network \(\hat\varepsilon_\theta(x_t, t)\) to predict the noise \(\epsilon\) added to \(x_0\). The DSM loss becomes
\[
\mathcal{L}(\theta) = \mathbb{E}_{t\sim\mathcal{U}\{1,\dots,T\}}\; \mathbb{E}_{x_0\sim p_{\text{data}}}\; \mathbb{E}_{\epsilon\sim\mathcal{N}(0,I)}\left[ \left\| \hat\varepsilon_\theta(x_t, t) - \epsilon \right\|^2 \right],
\]
where \(x_t = \sqrt{\bar{\alpha}_t}x_0 + \sqrt{1-\bar{\alpha}_t}\epsilon\). Minimising this, forces \(\hat\varepsilon_\theta(x_t, t)\) to approximate the noise that would generate \(x_t\) from \(x_0\). The score is then recovered via
\[
\nabla_{x_t}\log p_t(x_t) \approx -\frac{\hat\varepsilon_\theta(x_t, t)}{\sqrt{1-\bar{\alpha}_t}}.
\]

Two popular discrete‑time samplers, relying on the same trained noise‑prediction network \(\hat\varepsilon_\theta\), which encodes the score function across all noise levels, arise from different discretisations of these reverse dynamics.

The \emph{stochastic sampler} from \emph{Denoising Diffusion Probabilistic Models} (DDPM) by \cite{ho2020ddpm} works with a Markov chain that approximates the reverse SDE. Starting from \(x_T\sim\mathcal{N}(0,I)\), for \(t=T,\dots,1\) they use
\[
x_{t-1} = \frac{1}{\sqrt{\alpha_t}}\Bigl(x_t - \frac{1-\alpha_t}{\sqrt{1-\bar{\alpha}_t}}\hat\varepsilon_\theta(x_t,t)\Bigr) + \sigma_t z,\quad z\sim\mathcal{N}(0,I),
\]
with \(\sigma_t\) typically chosen as \(\sqrt{1-\bar{\alpha}_{t-1} - \frac{1-\bar{\alpha}_t}{\alpha_t}}\). This update corresponds to a discretisation of the reverse SDE and retains stochasticity at each step.

The \emph{deterministic sampler} from \emph{Denoising Diffusion Implicit Models} (DDIM) by \cite{song2021ddim} works with a non‑Markovian formulation that yields a deterministic generative process. Discretising the probability flow ODE with an Euler step from \(t\) to \(t-1\) gives the DDIM update:
\[
x_{t-1} = \sqrt{\bar{\alpha}_{t-1}}\,\hat{x}_0(x_t,t) + \sqrt{1-\bar{\alpha}_{t-1}}\,\hat\varepsilon_\theta(x_t,t),
\]
where \(\hat{x}_0(x_t,t) := \frac{x_t - \sqrt{1-\bar{\alpha}_t}\,\hat\varepsilon_\theta(x_t,t)}{\sqrt{\bar{\alpha}_t}}\) is the network’s estimate of the clean image. Substituting \(\hat{x}_0\) yields the single‑step operator
\begin{equation}
\Phi_t(x) = \frac{\sqrt{\bar{\alpha}_{t-1}}}{\sqrt{\bar{\alpha}_t}}\,x + \left(\sqrt{1-\bar{\alpha}_{t-1}} - \frac{\sqrt{\bar{\alpha}_{t-1}}}{\sqrt{\bar{\alpha}_t}}\sqrt{1-\bar{\alpha}_t}\right)\hat\varepsilon_\theta(x,t)=\frac{\sqrt{\bar\alpha_{t-1}}}{\sqrt{\bar\alpha_t}}\,x
               + b_t\,\hat\varepsilon_\theta(x,t),
    \label{eq:diff3}
\end{equation}
which uses the score (via \(\hat\varepsilon_\theta\)) to move one step backward in time and where the \emph{score step coefficient} is given by
\begin{equation}
    b_t := \sqrt{1-\bar\alpha_{t-1}}
             - \frac{\sqrt{\bar\alpha_{t-1}}\sqrt{1-\bar\alpha_t}}{\sqrt{\bar\alpha_t}}.
    \label{eq:b_def}
\end{equation}
Because no noise is added, the process defines a deterministic mapping \(x_T \mapsto x_0 = \Phi(x_T)\) and can be run with fewer steps by sub‑sampling the timesteps. This sampler will be the object of our study.

\begin{proposition}[Sign of $b_t$]
\label{prop:b_sign}
For any strictly decreasing $\{\bar\alpha_t\}$ with $\bar\alpha_0=1$,
$b_t<0$ for all $t\geq 1$.
\end{proposition}

\begin{proof}
$b_t<0 \iff \sqrt{1-\bar\alpha_{t-1}}<\sqrt{\bar\alpha_{t-1}}\sqrt{1-\bar\alpha_t}/\sqrt{\bar\alpha_t}$, which is true as the schedule is strictly decreasing. 
\end{proof}

Because the sampler is deterministic, the full reverse chain defines a
\emph{deterministic map} $\Phi=\Phi_1\circ\cdots\circ\Phi_T: \mathbb{R}^n \to \mathbb{R}^n$ from noise to image, each $\Phi_t$ with distinct parameters $\bar\alpha_t$ and $b_t$ and
distinct score-network behaviour at each noise level.
The trained network $\hat\varepsilon_\theta$ approximates the \emph{optimal predictor}
$\varepsilon_\theta^*(x_t,t):=-\sqrt{v_t}\,\nabla_{x_t}\log p_t(x_t)$,
the true score of the noised marginal $p_t$.

The central observation motivating the PIFS connection is that each $\Phi_t$ in \eqref{eq:diff3} has two competing effects: the factor $\sqrt{\bar\alpha_{t-1}/\bar\alpha_t}>1$ expands distances, while $b_t<0$ counteracts this expansion through the score correction. Whether the net effect is contractive depends on how strongly the score Jacobian drives the system inward.

\subsection{Partitioned Iterated Function Systems}

Recall that a map $w:(X,d)\to(X,d)$ is a \emph{contraction} if
$d(w(x),w(y))\leq s\,d(x,y)$ for some $s\in[0,1)$, called \emph{contraction factor}, and all $x,y\in X$.

\begin{theorem}[Banach Fixed Point Theorem \citep{banach1922}]
\label{thm:banach}
Let $(X,d)$ be a complete metric space and $w$ a contraction with factor $s<1$.
There exists a unique $x^*\in X$ with $w(x^*)=x^*$ and
$\lim_{k\rightarrow \infty}w^k(x)=x^*$ for all $x\in X$, at rate $d(w^k(x), x^*) \leq s^kd(x, x^*)$.
\end{theorem}

An \emph{iterated function system} (IFS) $\{w_k\}_{k=1}^N$ consists of contractions
on $(X,d)$. The \emph{Hutchinson operator} $\mathcal{H}(A)=\bigcup_k w_k(A)$ is itself a contraction on $(\mathcal{C}(\mathbb{R}^2),d_H)$,
the complete space of non-empty compact subsets of $\mathbb{R}^2$ equipped with the Hausdorff metric
$d_H$. It hence has a unique fixed point attractor $A^*$ reached by iterating from any
starting compact set \citep{hutchinson1981fractals}.
This operator was the theoretical basis for fractal image compression. 

\begin{theorem}[Collage Theorem \citep{barnsley1988}]
For every compact set $A$
\begin{equation}
    d_H(A,A^*) \;\leq\; \frac{1}{1-s}\,d_H(A,\mathcal{H}(A)),
    \label{eq:collage}
\end{equation}
\end{theorem}
Indeed, by Barnsley's theorem  we have that whenever $A$ is close in Hausdorff distance to its image under the operator $\mathcal{H}$,  the distance between $A$ and the attractor of $\mathcal{H}$ is smaller than the distance to its ``collage'' of smaller copies of itself, the \emph{collage error}. This is applicable to sets with \emph{global self-similarity}, like fractals for which the collage error is zero. Global self-similar compact sets $A$ can hence be \emph{reconstructed} from any other compact set by iterating an IFS consisting of contractions describing the smaller copies for which the Hutchinson operator minimises the distance to $\mathcal{H}(A)$.

\medskip
Natural images, however, do not exhibit the global self-similarity as different regions of an image are typically unrelated, so a single IFS cannot represent them accurately.  \citep{jacquin1992image} showed that this limitation can be overcome by extending the Collage Theorem to a \emph{Partitioned IFS} (PIFS) $\{w_k\}_{k=1}^N$, which works at patch level as natural images exhibit local self-similarity. Indeed,  parts of an image look like other parts after shrinking combined with operations like rotations, reflections and translations.  

For a PIFS, the image is partitioned into non-overlapping patches called \emph{range blocks}  $\{R_k\}_{k=1}^M$ (the output
tiles to be encoded) and a set of larger \emph{domain blocks} $\{D_j\}$ (candidate
source tiles, drawn from a pool that covers
the image). Each range block $R_k$ is encoded by an affine map
$$w_k(x) = s_k A_k(\mathrm{DS}(D_{j(k)})) + o_k,$$ where $\mathrm{DS}$ is a spatial
downsampling operator that matches domain and range block sizes, $A_k$ is a geometric
symmetry (rotation or reflection), $s_k \in [0,1)$ is a luminance scaling factor
enforcing contractivity, and $o_k$ is a luminance offset. The encoder searches over
the pool of domain blocks and a discrete set of symmetries to find, for each range
block, the $(j(k), A_k, s_k, o_k)$ minimising the Euclidean distance between
$w_k(D_{j(k)})$ and $R_k$, called the \emph{collage error at the block level}.
The collection $\{w_k\}_{k=1}^N$ defines a PIFS whose attractor, reached by iterating the
system from any initial image approximates the original. Contractivity is guaranteed
per block 
rather than globally, enabling the system to capture local
texture repetition while allowing global structure to vary across blocks.

\section{Contraction Structure of the Denoising Operator}
\label{sec:contraction}

Asking whether the deterministic DDIM reverse chain $\Phi = \Phi_1\circ\cdots\circ\Phi_T$ is a PIFS means investigating
whether it is contractive on patch-level. As the operator is a composition of denoising steps, we will first concentrate on the $\Phi_t$. 

\newpage
The operator $\Phi_t$ is a contraction in
$(\mathbb{R}^n,\|\cdot\|_2)$ if and only  if $\|J_x\Phi_t\|_\mathrm{op}<1$ for all $x$, where $J_x\Phi_t(x):= \nabla_x\Phi_t(x) \in \mathbb{R}^{n\times n}$ is the \emph{Jacobian}. Recall that for a linear operator $A$, $\|A\|_\mathrm{op}:=\sup_{\|x\|_2=1}\|A(x)\|_2$ denotes
the \emph{operator norm}, which equals the largest singular value of its matrix. Differentiating
\eqref{eq:diff3} with respect to $x$ gives
\begin{equation}
    J_x\Phi_t (x)
    = \frac{\sqrt{\bar\alpha_{t-1}}}{\sqrt{\bar\alpha_t}}\,I
    + b_t\,J_x\hat\varepsilon_\theta(x,t),
    \label{eq:jacobian_phit}
\end{equation}
so the Jacobian of the denoising operator decomposes as a scaled identity plus a correction term entirely determined by the score network's Jacobian $J_x\hat\varepsilon_\theta$.
Every contractivity question about $\Phi_t$ therefore relates to an algebraic condition on $J_x\hat\varepsilon_\theta$.

\subsection{Contraction in the Euclidean Norm}
\label{sec:global_case}

Because of the Jacobian formula \eqref{eq:jacobian_phit}, $\Phi_t$ is contractive when
\begin{equation}
    \left\|\frac{\sqrt{\bar\alpha_{t-1}}}{\sqrt{\bar\alpha_t}}\,I
    + b_t\,J_x\hat\varepsilon_\theta(x,t)\right\|_\mathrm{op} < 1.
    \label{eq:euclidean_contraction_cond}
\end{equation}

\begin{theorem}[Euclidean Contraction]
\label{thm:euclidean_char}
Let $\hat\varepsilon_\theta(\cdot,t):\mathbb{R}^n\to\mathbb{R}^n$ be $C^1$, and suppose
its Jacobian admits, for every $x\in\mathbb{R}^n$, the decomposition
\begin{equation}
    J_x\hat\varepsilon_\theta(x,t) = \nu_t(x)\,I_n + R_t(x),
    \label{eq:score_jacobian_decomp}
\end{equation}
where $\nu_t:\mathbb{R}^n\to(0,\infty)$ is a positive scalar field,
$I_n$ is the $n\times n$ identity, and the residual
$R_t(x)\in\mathbb{R}^{n\times n}$ satisfies $\|R_t\|_\mathrm{op}\leq\delta_t$.
Define the \emph{contraction threshold}
\begin{equation}
    L_t^* := \frac{\sqrt{\bar\alpha_{t-1}/\bar\alpha_t}-1}{|b_t|} > 0,
    \label{eq:threshold}
\end{equation}
and set $\nu_t^\mathrm{min}:=\inf_{x\in\mathbb{R}^n}\nu_t(x)$.
Assume:
\begin{alignat}{2}
    &\textup{(C1)}\quad &\nu_t^\mathrm{min} &> L_t^* + \delta_t, \label{eq:C1}\\
    &\textup{(C2)}\quad &|b_t|\,\nu_t(x) &\leq \frac{\sqrt{\bar\alpha_{t-1}}}{\sqrt{\bar\alpha_t}}
    \quad\text{for all }x\in\mathbb{R}^n. \label{eq:C2}
\end{alignat}

Then $\Phi_t:\mathbb{R}^n\to\mathbb{R}^n$ is a contraction in
$(\mathbb{R}^n,\|\cdot\|_2)$ with contraction factor
\begin{equation}
    \kappa_t := 1 - |b_t|(\nu_t^\mathrm{min} - L_t^* - \delta_t) \in (0,1).
    \label{eq:kappa}
\end{equation}
\end{theorem}

\begin{proof}
Substituting \eqref{eq:score_jacobian_decomp} into \eqref{eq:jacobian_phit} gives
\begin{equation}
    J_x\Phi_t = c(x)\,I_n + b_t R_t(x),
    \qquad c(x) := \frac{\sqrt{\bar\alpha_{t-1}}}{\sqrt{\bar\alpha_t}} + b_t\nu_t(x).
    \label{eq:jacobian_decomp_proof}
\end{equation}
Since $b_t<0$, condition~\textup{(C2)} reads
$-b_t\nu_t(x)\leq\sqrt{\bar\alpha_{t-1}/\bar\alpha_t}$,
hence $c(x)\geq 0$.

As $\nu_t(x)\geq\nu_t^\mathrm{min}$, the scalar $c(x)$ is
decreasing in $\nu_t(x)$, so
\[
    c(x) \leq \frac{\sqrt{\bar\alpha_{t-1}}}{\sqrt{\bar\alpha_t}}
    - |b_t|\nu_t^\mathrm{min}
    = \kappa_t - |b_t|\delta_t.
\]
The triangle inequality applied to
$J_x\Phi_t = c(x)I_n + b_t R_t(x)$ gives
\[
    \|J_x\Phi_t\|_\mathrm{op} \leq c(x) + |b_t|\|R_t\|_\mathrm{op}
    \leq c(x) + |b_t|\delta_t \leq \kappa_t.
\]
Condition~\textup{(C1)} gives $\kappa_t<1$; condition~\textup{(C2)} and
$\nu_t^\mathrm{min}-L_t^*-\delta_t>0$ give $\kappa_t>0$.
We then get that
$\|\Phi_t(x)-\Phi_t(y)\|_2\leq\kappa_t\|x-y\|_2$ for all $x,y\in\mathbb{R}^n$.
\end{proof}

\begin{definition}[(EC): Euclidean Contraction]
\label{def:structured}
$\Phi_t$  satisfies \emph{(EC)} if $\hat\varepsilon_\theta(\cdot,t)$ is $C^1$, its Jacobian admits
decomposition~\eqref{eq:score_jacobian_decomp} for all $x\in\mathbb{R}^n$,
and the constants $(\nu_t^\mathrm{min},\delta_t)$ satisfy
conditions~\textup{(C1)} and~\textup{(C2)}.
\end{definition}
(EC) is hence a condition on the score network, where its two constants
$\nu_t^\mathrm{min}$ and $\delta_t$ can be derived from the score Jacobian,
while the threshold $L_t^*$ is determined entirely
by the noise schedule and carries no dependence on the network or the data.
Hence $\Phi_t$ is an Euclidean
contraction whenever (EC) holds.
The contraction factor $\kappa_t$ \eqref{eq:kappa} makes the dependence on the
score network explicit.

\begin{corollary}
\label{thm:fixed_point}
Suppose \textup{(EC)} holds for every $\Phi_t$  $(t\in\{1,\ldots,T\})$ with
respective factors $\kappa_t\in(0,1)$.
Then the composition
$\Phi=\Phi_1\circ\cdots\circ\Phi_T$
is a contraction on $(\mathbb{R}^n,\|\cdot\|_2)$ with contraction factor
$s:=\prod_{t=1}^T\kappa_t\in(0,1)$.
Consequently, $\Phi$ a unique fixed point $x^*\in\mathbb{R}^n$ such that for every starting point $x\in\mathbb{R}^n$
the iterates $\Phi^{\circ k}(x)$ (the $k$-fold composition of $\Phi$ with itself)
converge to $x^*$ at rate $\|\Phi^{\circ k}(x)-x^*\|_2\leq s^k\|x-x^*\|_2$.

Moreover, every singular value of $J_{x^*}\Phi\in\mathbb{R}^{n\times n}$ satisfies
$s_i(J_{x^*}\Phi)\leq s<1$. \end{corollary}

\begin{proof}
The composition is again a contraction 
with contraction factor $s=\prod_{t=1}^T\kappa_t\|x-y\|_2=s\|x-y\|_2$.
The singular-value bound follows from $\|J_{x^*}\Phi\|_\mathrm{op}\leq s<1$.
\end{proof}

\subsection{Contraction in the Block-Max Norm}
\label{sec:patch_local}
As a PIFS and diffusion work on patches, 
we now look at the \emph{block-max norm}
$\|x\|_{\infty,M}:=\max_k\|x^{(k)}\|_2$,
where $x^{(k)}\in\mathbb{R}^{n_k}$ denotes the vector of pixel values in patch $R_k$.

From \eqref{eq:jacobian_phit}, the Jacobian of $\Phi_t$ inherits the block structure
directly from the score network.

\begin{theorem}[Block-Max Contraction]
\label{thm:score_coupling_theorem}
Let $\hat\varepsilon_\theta(\cdot,t):\mathbb{R}^n\to\mathbb{R}^n$ be $C^1$.
Denote by $[J_x\hat\varepsilon_\theta]_{kj}\in\mathbb{R}^{n_k\times n_j}$ the
$(k,j)$ block of the Jacobian.
From \eqref{eq:jacobian_phit}, the corresponding blocks of $J_x\Phi_t$ are
\begin{equation}
    [J_x\Phi_t]_{kk}
    =\frac{\sqrt{\bar\alpha_{t-1}}}{\sqrt{\bar\alpha_t}}I_{n_k}
    +b_t[J_x\hat\varepsilon_\theta]_{kk}\in\mathbb{R}^{n_k\times n_k},
    \quad
    [J_x\Phi_t]_{kj}
    =b_t[J_x\hat\varepsilon_\theta]_{kj}\in\mathbb{R}^{n_k\times n_j},\;k\neq j.
    \label{eq:phi_blocks}
\end{equation}
Define the \emph{score-coupling field} $\mathcal{C}_{k,t}:\mathbb{R}^n\to\mathbb{R}_{\geq 0}$ by
\begin{equation}
    \mathcal{C}_{k,t}(x) := |b_t|\sum_{j\neq k}\bigl\|[J_{x}\hat\varepsilon_\theta]_{kj}\bigr\|_\mathrm{op}
    = \sum_{j\neq k}\|[J_x\Phi_t]_{kj}\|_\mathrm{op}.
    \label{eq:total_score_coupling}
\end{equation}
$\Phi_t:\mathbb{R}^n\to\mathbb{R}^n$ is a  contraction in
$(\mathbb{R}^n,\|\cdot\|_{\infty,M})$ if
\begin{equation}
    \|[J_x\Phi_t]_{kk}\|_\mathrm{op} + \mathcal{C}_{k,t}(x) < 1
    \qquad\text{holds for every patch }k\text{ and all }x\in\mathbb{R}^n.
    \label{eq:block_max_iff}
\end{equation}
Setting
\begin{equation}
    \kappa_t^\mathrm{diag}:=\max_{1\leq k\leq M}\sup_{x\in\mathbb{R}^n}\|[J_x\Phi_t]_{kk}\|_\mathrm{op},
    \qquad
    \delta_t^\mathrm{cross}:=\max_{1\leq k\leq M}\sup_{x\in\mathbb{R}^n}\mathcal{C}_{k,t}(x),
    \label{eq:kappa_diag_cross_def}
\end{equation}
the contraction factor in $(\mathbb{R}^n,\|\cdot\|_{\infty,M})$ is
\begin{equation}
    \kappa_t^\mathrm{pc} := \kappa_t^\mathrm{diag} + \delta_t^\mathrm{cross} < 1.
    \label{eq:kappa_loc}
\end{equation}
\end{theorem}

\begin{proof}
Fix $x,y\in\mathbb{R}^n$ and write $\delta_j:=x^{(j)}-y^{(j)}\in\mathbb{R}^{n_j}$.
Applying the fundamental theorem of calculus along the segment
$z(\tau):=y+\tau(x-y)$ to the $k$-th patch output of the $C^1$ map $\Phi_t$:
\[
[\Phi_t(x)]^{(k)}-[\Phi_t(y)]^{(k)}
=\int_0^1\Bigl([J_{z(\tau)}\Phi_t]_{kk}\delta_k
+\sum_{j\neq k}[J_{z(\tau)}\Phi_t]_{kj}\delta_j\Bigr)d\tau.
\]
Taking $\|\cdot\|_2$ and using the triangle inequality:
\begin{align*}
\|[\Phi_t(x)]^{(k)}-[\Phi_t(y)]^{(k)}\|_2
&\leq \Bigl(\|[J_z\Phi_t]_{kk}\|_\mathrm{op}
  + \mathcal{C}_{k,t}(z)\Bigr)\max_j\|\delta_j\|_2 \\
&\leq \kappa_t^\mathrm{pc}\,\|x-y\|_{\infty,M}.
\end{align*}
\end{proof}

\begin{definition}[(PC): Block-Max Contraction]
\label{def:structured_local}
$\Phi_t$ satisfies \emph{(PC)} whenever
$$\kappa_t^\mathrm{pc} = \kappa_t^\mathrm{diag} + \delta_t^\mathrm{cross} < 1,$$
with $$\kappa_t^\mathrm{diag}=\max_{1\leq k\leq M}\sup_{x\in\mathbb{R}^n}\|[J_x\Phi_t]_{kk}\|_\mathrm{op},
    \qquad
    \delta_t^\mathrm{cross}=\max_{1\leq k\leq M}\sup_{x\in\mathbb{R}^n}\mathcal{C}_{k,t}(x)= |b_t|\sum_{j\neq k}\bigl\|[J_{x}\hat\varepsilon_\theta]_{kj}\bigr\|_\mathrm{op}.$$
\end{definition}

\begin{corollary}
\label{thm:local_pifs}
Suppose \textup{(PC)} holds for every $\Phi_t$  $(t\in\{1,\ldots,T\})$ with
respective factors $\kappa_t^\mathrm{pc}\in(0,1)$.
Then the composition
$\Phi=\Phi_1\circ\cdots\circ\Phi_T$
is a  contraction on $(\mathbb{R}^n,\|\cdot\|_{\infty,M})$ with contraction factor
$$s^\mathrm{loc}:=\prod_{t=1}^T\kappa_t^\mathrm{pc}\in(0,1),$$
and therefore has a unique fixed point such that for every starting point $x\in\mathbb{R}^n$
the iterates $\Phi^{\circ k}(x)$ converge to $x^*$ at rate 
$\|\Phi^{\circ k}(x)-x^*\|_{\infty,M}\leq (s^\mathrm{loc})^k\|x-x^*\|_{\infty,M}$.
\end{corollary}

\begin{proof}
Each $\Phi_t$ is a  contraction on $(\mathbb{R}^n,\|\cdot\|_{\infty,M})$ with factor
$\kappa_t^\mathrm{pc}<1$ by Theorem~\ref{thm:score_coupling_theorem}.
The composition of finitely many contractions satisfies
\[
    \|\Phi(x)-\Phi(y)\|_{\infty,M}\leq\prod_{t=1}^T\kappa_t^\mathrm{pc}\|x-y\|_{\infty,M}
    = s^\mathrm{loc}\|x-y\|_{\infty,M}.
\]
Since all norms on $\mathbb{R}^n$ are equivalent, $(\mathbb{R}^n,\|\cdot\|_{\infty,M})$
is complete, and Theorem~\ref{thm:banach} gives the unique fixed point and
convergence rate.
\end{proof}
The theorem guarantees that repeatedly applying the PIFS $\Phi$ would eventually converge to a unique fixed point $x^*$. However in generative modelling we apply the composition $\Phi$ only once. We start from a noise vector $x_T$ and produce an image $x_0$ by applying the composition $\Phi_1\circ \dots \circ \Phi_T$. Different noise vectors generally map to different images under one application of $\Phi$, but if we would iterate $\Phi$ all inputs would eventually lead to the same output. 

The contraction in block-max norm guarantees stability and robustness: small perturbations in the input noise lead to controlled changes in the output's coarse structure. The interplay between diagonal contractivity ($\kappa_t^\mathrm{diag}$) and cross-patch coupling ($\delta_t^\mathrm{cross}$) in (PC) will show to be the lever through which the network can generate varied, realistic detail while maintaining global coherence.

\subsection{Learning the Geometry: Training as Collage Minimisation}
\label{sec:collage_learning}

Properties~(EC) and~(PC) are geometric conditions on the operator $\Phi$. We will show that the training objective actually drives the network toward satisfying them.
Score-matching turns out to be, in Barnsley's sense, the exact analogue of minimising collage error. 

\subsubsection{The Training Objective as Collage Error}

\begin{theorem}[Collage analogue for Denoising Training]
\label{prop:collage}
The denoising objective 
$$\mathcal{L}(\theta)=\mathbb{E}_{t,x_0,\varepsilon}[\|\hat\varepsilon_\theta(x_t,t)-\varepsilon\|_2^2]$$
equals, up to timestep reweighting,
\begin{equation}
    \sum_t \mathrm{SNR}_t\;\mathbb{E}_{x_0,x_t}\!\bigl[\|\hat{x}_0(x_t,t)-x_0\|_2^2\bigr]
    \qquad \text{with} \;\; \mathrm{SNR}_t=\frac{\bar\alpha_t}{1-\bar\alpha_t}.
    \label{eq:collage_objective}
\end{equation}
\end{theorem}

\begin{proof}
From the forward process~\eqref{eq:forward} and the reparameterisation
$\hat x_0=(x_t-\sqrt{1-\bar\alpha_t}\,\hat\varepsilon_\theta)/\sqrt{\bar\alpha_t}$,
one obtains $\hat\varepsilon(x_t,t)_\theta-\varepsilon=\frac{\sqrt{\bar\alpha_t}}{\sqrt{1-\bar\alpha_t}}(\hat x_0-x_0)$,
so $\|\hat\varepsilon_\theta (x_t,t)-\varepsilon\|_2^2=\mathrm{SNR}_t\|\hat x_0-x_0\|_2^2$.
Substituting gives \eqref{eq:collage_objective}, matching (B.5) of \citet{ho2020ddpm}.
\end{proof}

The following theorem makes the quantitative link between score-matching and convergence to the fixed point precise.

\subsubsection{Collage-Contraction Bridges}
\label{sec:collage_bridge}

\begin{theorem}[Collage-Contraction Bridge]
\label{thm:collage_bridge}
Under the assumptions of Corollary~\ref{thm:local_pifs}, let $x^*$ be the unique
fixed point of $\Phi=\Phi_1\circ\cdots\circ\Phi_T:\mathbb{R}^n\to\mathbb{R}^n$.
Let $d_t:=\|\Phi_t(x_t)-x_t\|_2$ be the one-step displacement at time $t$
and set $c_t:=\prod_{k=1}^{t-1}\kappa_k^{\text{pc}}$ (with $c_1:=1$). Then:
\begin{equation}
    \|x^*-x_T\|_2
    \leq \frac{1}{1-s^{\text{loc}}}\sum_{t=1}^{T} c_t\, d_t.
    \label{eq:collage_bridge}
\end{equation}
\end{theorem}

\begin{proof}
By Theorem~\ref{thm:banach} applied to $\Phi$,
$\|x_T-x^*\|_2\leq\frac{1}{1-s^{\text{loc}}}\|\Phi(x_T)-x_T\|_2$.
Then
\[
\Phi(x_T)-x_T
=\sum_{t=1}^T\bigl[(\Phi_1\circ\cdots\circ\Phi_t)(x_T)
-(\Phi_1\circ\cdots\circ\Phi_{t-1})(x_T)\bigr].
\]
The $t$-th term equals
$(\Phi_1\circ\cdots\circ\Phi_{t-1})(\Phi_t(x_t))-(\Phi_1\circ\cdots\circ\Phi_{t-1})(x_t)$.
Applying the Lipschitz bound of $\Phi_1\circ\cdots\circ\Phi_{t-1}$ (with factor $c_t$) gives
$\|\text{$t$-th term}\|_2\leq c_t\,d_t$.
Summing over $t$ and multiplying by $1/(1-s^{\text{loc}})$ yields \eqref{eq:collage_bridge}.
\end{proof}

\medskip
While Theorem~\ref{thm:collage_bridge} bounds the deterministic displacement of an individual sample path, generative modelling ultimately requires guarantees about probability measures. The following theorem lifts this trajectory-level bound to a distributional guarantee.

Concretely, the theorem bounds $\mathcal{W}_1(q(x_T),\delta_{x^*})$, the Wasserstein-$1$ distance from the \emph{prior} $q(x_T)$ (the distribution of the initial noise) to the Dirac mass at the fixed point~$x^*$.
This is the statistical analogue of Barnsley's Collage Theorem: minimising the score-matching objective during training not only pulls individual trajectories toward the fixed point but mathematically ensures the output measure aligns with~$\delta_{x^*}$.
A small $\mathcal{W}_1$ distance guarantees that the generated distribution inherits the fractal geometry of the PIFS fixed point $x^*$, which in turn approximates the intrinsic dimension of the true data manifold.

\begin{theorem}[$L^2$--$\mathcal{W}_1$ Collage Bridge]
\label{thm:l2_hausdorff}
Let
\begin{equation*}
\mathcal{L}_t(\theta)=\mathbb{E}_{x_t}\|\hat\varepsilon_\theta(x_t,t)-\varepsilon_\theta^*(x_t,t)\|_2^2
\end{equation*}
be the \emph{excess score-matching risk} at step $t$. 
Define the \emph{coupling weights} $B_t:=|b_t|$.
Then
\begin{equation}
    \mathcal{W}_1\!\bigl(q(x_T),\delta_{x^*}\bigr)
    \leq
    \frac{\sqrt{T}}{1-s^{\text{loc}}}
    \left(\sum_{t=1}^T\Bigl(\prod_{k<t}\kappa_k^{\text{pc}}\Bigr)^{\!2}
    \Bigl(\sqrt{\mathbb{E}[\|\mathcal{R}_t\|_2^2]}+B_t\sqrt{\mathcal{L}_t(\theta)}\Bigr)^{\!2}
    \right)^{\!1/2},
    \label{eq:l2_hausdorff}
\end{equation}
where $\delta_{x^*}$ denotes the Dirac measure concentrated at the fixed point $x^*$,
and $\mathcal{R}_t(x_t):=(\sqrt{\bar\alpha_{t-1}/\bar\alpha_t}-1)x_t+b_t\varepsilon_\theta^*(x_t,t)$
is the schedule-geometry displacement under the true score (independent of~$\theta$).
\end{theorem}

\begin{proof}
$\mathcal{W}_1(\mu,\delta_{x^*})=\mathbb{E}_{x_T\sim\mu}\|x_T-x^*\|_2$.
Apply Theorem~\ref{thm:collage_bridge} and take expectations; Cauchy-Schwarz over $t$
gives
\begin{equation*}
\mathbb{E}\|x^*-x_T\|_2
\leq\frac{\sqrt{T}}{1-s^{\text{loc}}}\Bigl(\sum_t c_t^2\,\mathbb{E}[d_t^2]\Bigr)^{1/2},
\quad c_t=\prod_{k<t}\kappa_k^{\text{pc}}.
\end{equation*}
Here $c_t$ is the Lipschitz factor of $\Phi_1\circ\cdots\circ\Phi_{t-1}$, the
composition of all steps applied \emph{after} step $t$ in the reverse chain
(steps with smaller index are applied last).
In particular $c_1=1$ (empty product), corresponding to the final applied step
$\Phi_1$, which receives the largest weight; early high-noise steps $\Phi_T$ carry
weight $c_T\approx s^{\text{loc}}< 1$.
Decompose the one-step displacement using \eqref{eq:diff3}:
\begin{equation*}
\Phi_t(x_t)-x_t
= \underbrace{\Bigl(\frac{\sqrt{\bar\alpha_{t-1}}}{\sqrt{\bar\alpha_t}}-1\Bigr)x_t
   +b_t\,\varepsilon_\theta^*(x_t,t)}_{=:\,\mathcal{R}_t(x_t)}
+\,b_t\,\bigl(\hat\varepsilon_\theta(x_t,t)-\varepsilon_\theta^*(x_t,t)\bigr).
\end{equation*}
The second term is the score approximation error $b_t\,\xi_t$ where
$\xi_t:=\hat\varepsilon_\theta(x_t,t)-\varepsilon_\theta^*(x_t,t)$.
$\mathcal{R}_t$ depends only on $x_t$ and the true score, so it is $\theta$-free.
Taking squared norms:
\begin{equation*}
d_t^2 = \|\mathcal{R}_t\|_2^2 + 2b_t\langle \mathcal{R}_t, \xi_t\rangle + b_t^2\|\xi_t\|_2^2.
\end{equation*}
For general $\theta$, the cross-term satisfies by Cauchy-Schwarz:
\begin{equation*}
|\mathbb{E}[\langle \mathcal{R}_t,\xi_t\rangle]|
\leq\sqrt{\mathbb{E}[\|\mathcal{R}_t\|_2^2]}\cdot\sqrt{\mathcal{L}_t(\theta)},
\end{equation*}
so
\begin{equation*}
\mathbb{E}[d_t^2]
\leq \mathbb{E}[\|\mathcal{R}_t\|_2^2] + 2B_t\sqrt{\mathbb{E}[\|\mathcal{R}_t\|_2^2]}\sqrt{\mathcal{L}_t(\theta)}
+ B_t^2\mathcal{L}_t(\theta)
= \Bigl(\sqrt{\mathbb{E}[\|\mathcal{R}_t\|_2^2]}+B_t\sqrt{\mathcal{L}_t(\theta)}\Bigr)^{\!2}.
\end{equation*}
Substituting $\mathbb{E}[d_t^2]^{1/2}\leq\sqrt{\mathbb{E}[\|\mathcal{R}_t\|_2^2]}+B_t\sqrt{\mathcal{L}_t(\theta)}$
into the Cauchy-Schwarz bound gives \eqref{eq:l2_hausdorff}.
\end{proof}

Since the excess risk
$\mathcal{L}_t(\theta)=\mathbb{E}\|\hat\varepsilon_\theta-\varepsilon_\theta^*\|_2^2$
is proportional to the reconstruction loss against the \emph{conditional mean}
$\mathbb{E}[x_0\mid x_t]$:
$\hat\varepsilon_\theta-\varepsilon_\theta^*=(\sqrt{\bar\alpha_t}/\sqrt{v_t})(\hat x_0-\mathbb{E}[x_0\mid x_t])$,
giving $\mathcal{L}_t(\theta)=(\bar\alpha_t/v_t)\mathbb{E}\|\hat x_0-\mathbb{E}[x_0\mid x_t]\|_2^2$,
the bound~\eqref{eq:l2_hausdorff} holds equally in terms of the reconstruction
objective, up to the substitution $B_t\to|b_t|\sqrt{\bar\alpha_t/v_t}$
inside the inner square root.

\subsubsection{PIFS Collage Regulariser and Enforcement of (PC)}
\label{sec:pifs_regularizer}

\begin{theorem}[PIFS Collage Regulariser]
\label{thm:pifs_loss}
Define the \emph{PIFS collage loss}
$\mathcal{L}_\mathrm{PIFS}(\theta):=\mathcal{L}(\theta)
+\mu_\mathrm{reg}\sum_{t,k,j\neq k}\|[J_{x}\hat\varepsilon_\theta]_{kj}\|_F^2$.
Then:
\begin{enumerate}
  \item[\textup{(a)}] $\mathcal{L}_\mathrm{PIFS}(\theta)=0$ if and only if the score
  is perfect and all off-diagonal blocks vanish.
  \item[\textup{(b)}] The gradient is computable in $O(M^2n_k)$ JVP/VJP operations
  without materialising the full Jacobian.
  \item[\textup{(c)}] Minimising $\mathcal{L}_\mathrm{PIFS}$ strictly decreases a
  Cauchy-Schwarz upper bound on $\mathcal{C}_{k,t}(x)$, directly widening the
  (PC) contraction margin $1-\kappa_t^\mathrm{diag}-\delta_t^\mathrm{cross}$.
\end{enumerate}
\end{theorem}

\begin{proof}
(a) Both terms are non-negative.
(b) Each $[J]_{kj}v$ (JVP) and $u^\top[J]_{kj}$ (VJP) costs one forward
pass; $O(n_k)$ per block gives $O(M^2n_k)$ total.
\textbf{(c).} Since $[J_x\Phi_t]_{kj}=b_t[J_x\hat\varepsilon_\theta]_{kj}$,
any decrease in $\sum_{j\neq k}\|[J]_{kj}\|_F^2$ directly decreases the
Cauchy-Schwarz upper bound
\[
\mathcal{C}_{k,t}(x)=|b_t|\sum_{j\neq k}\|[J_x\hat\varepsilon_\theta]_{kj}\|_\mathrm{op}
\leq|b_t|\Bigl(\sum_{j\neq k}\|[J]_{kj}\|_F^2\Bigr)^{1/2}\sqrt{M},
\]
reducing $\delta_t^\mathrm{cross}=\max_k\sup_x\mathcal{C}_{k,t}(x)$ and widening
the margin in condition~\eqref{eq:block_max_iff}.
\end{proof}

\section{Two-Regime Structure of the Denoising Chain}
\label{sec:regime_structure}

The preceding section established \emph{what} the geometric conditions are
(Properties~(EC) and~(PC)) and \emph{how} training enforces them through collage
minimisation. 
The remaining question is \emph{when} and \emph{why} these conditions hold
across the denoising chain.
This requires computing $\kappa_t^\mathrm{diag}$ and $\delta_t^\mathrm{cross}$
from actual data statistics and architectural properties.

\subsection{Regime~I: Global Context Assembly}
\label{sec:fine_detail}

Note that at high noise ($t$ near $T$, $\bar\alpha_t \approx 0$), the cross-patch coupling
$\delta_t^\mathrm{cross}=O(\bar\alpha_t)$ is automatically negligible. We now study what happens in the next steps.

\begin{theorem}[High-Noise Regime Contraction]
\label{thm:high_noise}
Suppose $\|x_0\|_2\leq M$ almost surely, and assume the score network equals the
true score: $\hat\varepsilon_\theta(x_t,t)=\varepsilon_\theta^*(x_t,t)=-\sqrt{v_t}\,\nabla_{x_t}\log p_t(x_t)$.
Define
\begin{equation}
    \nu_t^* := \frac{1}{\sqrt{v_t}}, \qquad
    \delta_t^* := \frac{C M^2\bar\alpha_t}{v_t^{3/2}},
    \label{eq:high_noise_params}
\end{equation}
where $C>0$ is an absolute constant depending only on the data distribution.
Then $\delta_t^*=O(\bar\alpha_t)$ as $t\to T$, and \textup{(EC)}
holds at timestep $t$ with $\nu_t^\mathrm{min}:=\nu_t^*$ and $\delta_t:=\delta_t^*$.
Consequently, $\Phi_t$ is an  Euclidean contraction with
$\nu_t^*-\delta_t^*>L_t^*$, at all sufficiently high-noise steps under standard schedules.
\end{theorem}

\begin{proof}
Since $\varepsilon_\theta^*=-\sqrt{v_t}\,\nabla_{x_t}\log p_t$,
\[
J_{x_t}\varepsilon_\theta^* = -\sqrt{v_t}\,\nabla_{x_t}^2\log p_t(x_t).
\]
Applying Stein's identity to the Gaussian-mixture marginal
$$p_t(x_t)=\int\mathcal{N}(x_t;\sqrt{\bar\alpha_t}x_0,v_tI)\,p(x_0)\,dx_0$$
yields the decomposition of the log-density Hessian:
\begin{equation*}
    \nabla_{x_t}^2\log p_t
    = \underbrace{-(1/v_t)I}_{\text{Gaussian term}}
    + \underbrace{(\bar\alpha_t/v_t^2)\,\mathrm{Cov}(x_0\mid x_t)}_{\text{posterior covariance}}
    + H_t(x_t),
\end{equation*}
where $H_t(x_t)$ collects all conditional cumulants of $x_0\mid x_t$ beyond second
order (so $H_t\equiv 0$ when $p(x_0\mid x_t)$ is Gaussian).
Multiplying through by $-\sqrt{v_t}$:
\begin{equation}
    J_{x_t}\varepsilon_\theta^*
    = \frac{1}{\sqrt{v_t}}I
    - \frac{\bar\alpha_t}{v_t^{3/2}}\mathrm{Cov}(x_0\mid x_t)
    + E_t(x_t),
    \label{eq:score_hessian_general}
\end{equation}
with $E_t(x_t):=-\sqrt{v_t}\,H_t(x_t)$.
This is decomposition~\eqref{eq:score_jacobian_decomp} with
$\nu_t(x_t):=1/\sqrt{v_t}$ and
$R_t(x_t):=-(\bar\alpha_t/v_t^{3/2})\mathrm{Cov}(x_0\mid x_t)+E_t(x_t)$.

Let us now bound $\|R_t\|_\mathrm{op}$.
Since $\|x_0\|_2\leq M$, we have $\|\mathrm{Cov}(x_0\mid x_t)\|_\mathrm{op}\leq M^2$,
so $\|(\bar\alpha_t/v_t^{3/2})\mathrm{Cov}(x_0\mid x_t)\|_\mathrm{op}\leq M^2\bar\alpha_t/v_t^{3/2}$.
$H_t$ is the tensor of conditional cumulants of order three and higher.
The leading (third-order) term satisfies
$$\|H_t^{(3)}(x_t)\|_\mathrm{op}
    \leq C_0\,(\bar\alpha_t/v_t^2)\,
    \sup_{x_t}\mathbb{E}[\|x_0-\mu_t\|_2^3\mid x_t]$$
for an absolute constant $C_0>0$, where $\mu_t=\mathbb{E}[x_0\mid x_t]$.
Since $\|x_0\|_2\leq M$ and $\|\mu_t\|_2\leq M$, the bounded-support bound
$\|x_0-\mu_t\|_2\leq 2M$ gives $\mathbb{E}[\|x_0-\mu_t\|_2^3\mid x_t]\leq 8M^3$, hence
\begin{equation}
    \|E_t(x_t)\|_\mathrm{op}
    \leq \frac{8C_0 M^3\bar\alpha_t}{v_t^{3/2}}
    = O(\bar\alpha_t)
    \quad\text{as }t\to T.
    \label{eq:Et_bound}
\end{equation}
Combining leads to $\delta_t^* := \|R_t\|_\mathrm{op} \leq CM^2\bar\alpha_t/v_t^{3/2}$
for a suitable absolute constant $C$ absorbing both terms.

Both correction terms vanish as $O(\bar\alpha_t)$, so $\delta_t^*\to 0$.
Meanwhile $\nu_t^*=1/\sqrt{v_t}\to 1$, and the contraction margin
$\nu_t^*-\delta_t^*-L_t^*$ is eventually positive under standard schedules.
\end{proof}

Thus (EC) holds at high noise for any data distribution with bounded
support, regardless of distributional shape. 

\begin{theorem}[Diagonal Block Expansion in the Fine-Detail Regime]\label{thm:fine_detail_expansion}
Let $p_\mathrm{data}=\mathcal{N}(0,\Sigma_0)$ with $\Sigma_0$ block-diagonal
with respect to the patch partition, $\Sigma_0=\mathrm{diag}(\Sigma_1,\ldots,\Sigma_M)$,
so that the patches $x_0^{(1)},\ldots,x_0^{(M)}$ are independent.
Assume $\hat\varepsilon_\theta$ equals the true score $\varepsilon_\theta^*$.
For each patch $R_k\in\{1,\ldots,M\}$, let
\begin{equation*}
    \lambda_k := \|\Sigma_k\|_\mathrm{op}
\end{equation*}
be the largest eigenvalue of the patch covariance and let $v_t=1-\bar\alpha_t$.
Define the \emph{per-step diagonal expansion function} $f_t:(0,\infty)\to\mathbb{R}$ by
\begin{equation}
    f_t(\lambda) :=
    \sqrt{\frac{\bar\alpha_{t-1}}{\bar\alpha_t}}
    - |b_t|\frac{\sqrt{v_t}}{\lambda\bar\alpha_t+v_t},
    \label{eq:diag_block_norm}
\end{equation}
and the \emph{per-step patch expansion threshold} 
\begin{equation}
    \lambda^*(t) :=
    \frac{\sqrt{v_t}\,(\sqrt{v_t}-\sqrt{v_{t-1}})}{(\sqrt{\bar\alpha_{t-1}}-\sqrt{\bar\alpha_t})\sqrt{\bar\alpha_t}}.
    \label{eq:lambda_star}
\end{equation}
Then the diagonal block spectral norm satisfies
$\|[J_{x}\Phi_t]_{kk}\|_\mathrm{op} = f_t(\lambda_k)$
for all $x\in\mathbb{R}^n$,
with
\begin{enumerate}
    \item[\textup{(a)}] $f_t(\lambda_k)<1 \iff \lambda_k<\lambda^*(t)$;
    \item[\textup{(b)}] $\lambda^*(t)>1$ for any strictly decreasing schedule, so
    unit-variance patches ($\lambda_k=1$) are always diagonally contractive;
    \item[\textup{(c)}] $f_t(\lambda)$ is strictly increasing in $\lambda$.
\end{enumerate}
\end{theorem}

\begin{proof}
Under the block-diagonal Gaussian assumption with perfect score,
the marginal score decomposes patch-wise. Differentiating the conditional mean
$\hat x_0^{(k)}(x_t,t)=\Sigma_k(\Sigma_k+(v_t/\bar\alpha_t)I_{n_k})^{-1}x_t^{(k)}/\sqrt{\bar\alpha_t}$
and substituting into \eqref{eq:jacobian_phit} gives, for any $x$:
\begin{equation*}
    [J_{x}\varepsilon_\theta^*]_{kk}
    =\frac{1}{\sqrt{v_t}}I_{n_k}
    -\frac{\sqrt{\bar\alpha_t}}{v_t}\Sigma_k\!\left(\Sigma_k+\frac{v_t}{\bar\alpha_t}I_{n_k}\right)^{-1}.
\end{equation*}
Combined with \eqref{eq:diff3} and $b_t<0$, the eigenvalue of $[J_{x}\Phi_t]_{kk}$
along an eigenvector of $\Sigma_k$ with eigenvalue $\lambda$ is precisely
$f_t(\lambda)$ as in \eqref{eq:diag_block_norm}.
Since
$\partial f_t/\partial\lambda=|b_t|\sqrt{v_t}\bar\alpha_t/(\lambda\bar\alpha_t+v_t)^2>0$,
$f_t$ is strictly increasing (part~(c)), and
$\|[J_{x}\Phi_t]_{kk}\|_\mathrm{op}=\max_i f_t(\mu_{k,i})=f_t(\lambda_k)$,
where $\mu_{k,i}$ are the eigenvalues of $\Sigma_k$.

At $\lambda=1$, using $\bar\alpha_t+v_t=1$ and $\bar\alpha_{t-1}+v_{t-1}=1$,
\[
f_t(1)=\sqrt{\bar\alpha_{t-1}\bar\alpha_t}+\sqrt{v_{t-1}v_t}.
\]
As $\sqrt{\bar\alpha_{t-1}\bar\alpha_t}\leq(\bar\alpha_{t-1}+\bar\alpha_t)/2$
and $\sqrt{v_{t-1}v_t}\leq(v_{t-1}+v_t)/2$,
we have that $f_t(1)\leq(\bar\alpha_{t-1}+\bar\alpha_t+v_{t-1}+v_t)/2=1$,
with strict inequality when $\bar\alpha_{t-1}\neq\bar\alpha_t$.
Thus $f_t(1)<1$ and $\lambda^*(t)>1$ (part~(b)).

Setting $f_t(\lambda^*)=1$ and solving for $\lambda^*$ yields \eqref{eq:lambda_star}.
Parts~(a) then follows from strict monotonicity of $f_t$.
\end{proof}

Hence under Gaussian data and the perfect
score, the diagonal block of $\Phi_t$ contracts whenever
$\lambda_k < \lambda^*(t)$. 

\subsubsection{Deviation of the Trained Score from the Gaussian Score}
\label{sec:learned_suppression}

For non-Gaussian data the trained score deviates from the Gaussian score. Let $\varepsilon_\theta^{\mathcal{N}}(x_t,t)$ denote the score of the Gaussian surrogate
$\mathcal{N}(0,\Sigma_0)$ and $\hat\varepsilon_\theta(x_t,t)$ the trained score.
Define the \emph{per-step deviation field}
\begin{equation}
    \Delta_t(x) := \hat\varepsilon_\theta(x,t) - \varepsilon_\theta^{\mathcal{N}}(x,t),
    \label{eq:score_deviation}
\end{equation}
and its Jacobian $\nabla_x\Delta_t(x)$. By linearity of the DDIM step~\eqref{eq:diff3}
in the score, the diagonal block of the full Jacobian decomposes as
\begin{equation}
    [J_x\Phi_t]_{kk}
    = [J_x\Phi_t^{\mathcal{N}}]_{kk}
    + b_t[\nabla_x\Delta_t(x)]_{kk},
    \label{eq:jacobian_decomposition}
\end{equation}
where $\Phi_t^{\mathcal{N}}$ is the DDIM step under the Gaussian score and $b_t < 0$.

\begin{definition}[Directional Suppression and Suppression Margin]
\label{def:suppression}
Let $\Sigma_k$ denote\\ the empirical patch covariance of patch $R_k$ over the
training distribution, with top eigenvalue $\lambda_k$ and top eigenvector
$v_k^{(1)}\in\mathbb{R}^{n_k}$. For patch $R_k$ at timestep $t$ and $x\in\mathbb{R}^n$, define the
\emph{directional suppression field} along $v_k^{(1)}$
\begin{equation}
    S_{k,t}(x) := |b_t|\,\bigl\langle v_k^{(1)},\,
    [\nabla_x\Delta_t(x)]_{kk}\,v_k^{(1)}\bigr\rangle,
    \label{eq:suppression}
\end{equation}
and the \emph{suppression margin} at $(x,t)$
\begin{equation}
    \gamma_{k,t}(x) := S_{k,t}(x) - (f_t(\lambda_k) - 1).
    \label{eq:suppression_margin}
\end{equation}
\end{definition}

Recall that the \emph{Rayleigh quotient} of a matrix $B$ along a unit vector $v$
is $\langle v, Bv\rangle$, which equals the eigenvalue of $B$ in direction $v$ when
$v$ is an eigenvector.

\begin{proposition}[Suppression Keeps Diagonal Rayleigh Quotient Below Unity]
\label{prop:suppression_diagonal}
Assume the patch covariance is isotropic, $\Sigma_k = \lambda_k I_{n_k}$.
For patch $R_k$ at timestep $t$ and $x\in\mathbb{R}^n$, if $\gamma_{k,t}(x) > 0$ then
$\langle v_k^{(1)}, [J_x\Phi_t]_{kk}\,v_k^{(1)}\rangle < 1$.
\end{proposition}

\begin{proof}
Under Gaussian $p_\mathrm{data}$ with $\Sigma_k = \lambda_k I_{n_k}$,
$[J_x\Phi_t^{\mathcal{N}}]_{kk}$ is a scalar multiple of the identity with
eigenvalue $f_t(\lambda_k)$ along every direction (Theorem~\ref{thm:fine_detail_expansion}).
From~\eqref{eq:jacobian_decomposition} and $b_t < 0$, taking the Rayleigh
quotient along $v_k^{(1)}$:
\begin{equation*}
    \bigl\langle v_k^{(1)},\,[J_x\Phi_t]_{kk}\,v_k^{(1)}\bigr\rangle
    = f_t(\lambda_k)
    - |b_t|\bigl\langle v_k^{(1)},\,[\nabla_x\Delta_t(x)]_{kk}\,v_k^{(1)}\bigr\rangle
    = f_t(\lambda_k) - S_{k,t}(x) < 1.
\end{equation*}
\end{proof}

The proposition identifies the precise condition under which a non-Gaussian trained
score can maintain diagonal contractivity. 

\subsection{Regime~II: Fine-Detail Synthesis}

\subsubsection{Stratified Crossover}
\label{sec:stratified_crossover}

All statements below are
understood to hold $q(x_t)$-almost surely; we write $\gamma_{k,t}$
for $\gamma_{k,t}(x_t)$ when the trajectory evaluation is clear from
context.

Denote
\begin{equation}
     \kappa_{k,t}^{\mathrm{diag}}
    \;:=\;
     \bigl\langle v_k^{(1)},\,[J_x\Phi_t]_{kk}\,v_k^{(1)}\bigr\rangle
     \;=\;
     f_t(\lambda_k) - S_{k,t}(x).
     \label{eq:kappa_decomp}
\end{equation}

\begin{definition}[Per-patch release time]
\label{def:release_time}
For patch $R_k$, we define the \emph{release time} as
\begin{equation}
    t_k^{\mathrm{rel}}
    \;:=\;
    \sup\bigl\{t \leq T : \gamma_{k,t}(x_t) \leq 0\bigr\},
    \label{eq:release_time}
\end{equation}
where $x_t \sim q(x_t)$.   
\end{definition}

Hence $t_k^{\mathrm{rel}}$ is the
largest timestep, encountered while running the reverse chain from
$t = T$, at which $\kappa_{k,t}^{\mathrm{diag}}$ first reaches or
exceeds~$1$.

\begin{definition}[(MM): Margin Monotonicity]
\label{def:property_mm}
The score network $\hat\varepsilon_\theta$ satisfies \emph{(MM)} at
timestep $t$ if the following two conditions hold.
\begin{enumerate}
    \item[\textup{(MM1)}] \textbf{Spectral monotonicity.}
    The suppression margin $\gamma_{k,t}$ is strictly increasing in
    the patch variance $\lambda_k$:
    \begin{equation}
        \frac{\partial \gamma_{k,t}}{\partial \lambda_k}
        \;=\;
        \frac{\partial S_{k,t}}{\partial \lambda_k}
        - f'_t(\lambda_k)
        \;>\; 0.
        \label{eq:mm1}
    \end{equation}
    By Theorem~\ref{thm:fine_detail_expansion},
    $f'_t(\lambda_k) = |b_t|\sqrt{v_t}\,\bar\alpha_t /
    (\lambda_k\bar\alpha_t + v_t)^2 > 0$,
    so \eqref{eq:mm1} requires
    $\partial S_{k,t}/\partial\lambda_k > f'_t(\lambda_k)$, so
    the learned suppression must grow with patch variance faster than
    the Gaussian baseline does.

    \item[\textup{(MM2)}] \textbf{Temporal decrease.}
    The margin is strictly decreasing along the reverse chain:
    $\partial_t \gamma_{k,t} < 0$ for every patch $k$.
    This guarantees that every patch eventually crosses from
    suppression ($\gamma_{k,t} > 0$) to release ($\gamma_{k,t} \leq 0$).
\end{enumerate}
\end{definition}

\begin{theorem}[Stratified Crossover]
\label{thm:stratified_crossover}
Assume \textup{(MM)}.
\begin{enumerate}
    \item[\textup{(a)}] \textbf{Release order.}
    If $\lambda_k < \lambda_{k'}$, then
    $t_k^{\mathrm{rel}} > t_{k'}^{\mathrm{rel}}$, so
    lower-variance patches release earlier in the reverse chain.

    \item[\textup{(b)}] \textbf{Sign reversal of the Spearman rank.}
    For any two patches with $\lambda_k < \lambda_{k'}$ and any
    timestep $t \in [t_k^{\mathrm{rel}},\, t_{k'}^{\mathrm{rel}})$:
    \[
        \kappa_{k,t}^{\mathrm{diag}} \;\geq\; 1
        \;>\;
        \kappa_{k',t}^{\mathrm{diag}},
    \]
    so the pair $(k,k')$ is discordant in the Spearman ranking of
    $(\lambda_\cdot,\kappa_{\cdot,t}^{\mathrm{diag}})$.
    Whether discordant pairs dominate globally, and hence
    $\rho(\lambda_k, \kappa_{k,t}^{\mathrm{diag}}) < 0$, depends on
    the number of released patches at time $t$.

    \item[\textup{(c)}] \textbf{Width of the stratified window.}
    To first order in $\lambda_k$, implicit differentiation of
    the release condition $\gamma_{k,t_k^{\mathrm{rel}}} = 0$ gives
    \begin{equation}
        \frac{\partial t_k^{\mathrm{rel}}}{\partial \lambda_k}
        \;=\;
        \frac{|\partial_\lambda \gamma_{k,t}|}{|\partial_t \gamma_{k,t}|}
        \;>\; 0,
        \label{eq:dtrel_dlambda}
    \end{equation}
    and the total spread of release times across patches satisfies
    \begin{equation}
        \Delta t^{\mathrm{strat}}
        \;\approx\;
        \frac{(\lambda_{\max} - \lambda_{\min})
              \cdot |\partial_\lambda \gamma_{k,t}|}
             {|\partial_t \gamma_{k,t}|},
        \label{eq:stratified_span}
    \end{equation}
    evaluated at a representative variance $\bar\lambda$.
    Expanding via $\gamma_{k,t} = S_{k,t} - (f_t(\lambda_k)-1)$:
    \begin{equation}
        \Delta t^{\mathrm{strat}}
        \;\approx\;
        \frac{(\lambda_{\max} - \lambda_{\min})
              \cdot \bigl(\partial_\lambda S_{k,t}
              - |b_t|\sqrt{v_t}\,\bar\alpha_t/
              (\bar\lambda\bar\alpha_t + v_t)^2\bigr)}
             {|\partial_t S_{k,t} - \partial_t f_t(\bar\lambda)|}.
        \label{eq:stratified_span_explicit}
    \end{equation}
    The derivatives $\partial_\lambda S_{k,t}$ and $\partial_t S_{k,t}$
    depend on the trained network and can be estimated empirically.

    \item[\textup{(d)}] \textbf{Restoration of the Gaussian ordering.}
    Once all patches have released, i.e.\ for $t < \min_k t_k^{\mathrm{rel}}$,
    the Rayleigh quotient satisfies
    $\kappa_{k,t}^{\mathrm{diag}} = f_t(\lambda_k) - S_{k,t} \geq 1$
    for every $k$.  As $S_{k,t} \to 0$ in deep Regime~II, the
    ranking of $\kappa_{k,t}^{\mathrm{diag}}$ across patches converges to
    that of $f_t(\lambda_k)$, which is strictly increasing in
    $\lambda_k$ by Theorem~\ref{thm:fine_detail_expansion}(d), so
    $\rho(\lambda_k, \kappa_{k,t}^{\mathrm{diag}})$ returns to positive values.
\end{enumerate}
\end{theorem}

\begin{proof}
(a)
By (MM1), $\partial_{\lambda_k}\gamma_{k,t} > 0$, so lower-variance
patches have strictly smaller margin.  By (MM2), $\partial_t\gamma_{k,t}
< 0$, so each margin decreases strictly over time and must eventually
reach zero.  The release condition $\gamma_{k,t_k^{\mathrm{rel}}} = 0$
is therefore met at a strictly larger value of $t$ for patches with
smaller $\lambda_k$, giving
$t_k^{\mathrm{rel}} > t_{k'}^{\mathrm{rel}}$ whenever
$\lambda_k < \lambda_{k'}$.

(b)
Fix $t \in [t_k^{\mathrm{rel}}, t_{k'}^{\mathrm{rel}})$.
Since $t \geq t_k^{\mathrm{rel}}$, patch $R_k$ has released:
$\gamma_{k,t} \leq 0$, equivalently $S_{k,t} \leq f_t(\lambda_k) - 1$,
so $\kappa_{k,t}^{\mathrm{diag}} = f_t(\lambda_k) - S_{k,t} \geq 1$.
Since $t < t_{k'}^{\mathrm{rel}}$, patch $R_{k'}$ remains suppressed:
$\gamma_{k',t} > 0$, so $\kappa_{k',t}^{\mathrm{diag}} < 1$ by
Proposition~\ref{prop:suppression_diagonal}.  Hence
$\kappa_{k,t}^{\mathrm{diag}} \geq 1 > \kappa_{k',t}^{\mathrm{diag}}$
despite $\lambda_k < \lambda_{k'}$, i.e.\ the pair is discordant.

(c)
Differentiate the identity
$\gamma_{k,t_k^{\mathrm{rel}}}(\lambda_k) = 0$
implicitly with respect to $\lambda_k$:
\[
    \partial_\lambda\gamma_{k,t}
    + \partial_t\gamma_{k,t}\cdot
      \frac{\partial t_k^{\mathrm{rel}}}{\partial\lambda_k}
    = 0,
    \qquad\text{so}\qquad
    \frac{\partial t_k^{\mathrm{rel}}}{\partial\lambda_k}
    = -\frac{\partial_\lambda\gamma_{k,t}}{\partial_t\gamma_{k,t}}
    = \frac{|\partial_\lambda\gamma_{k,t}|}{|\partial_t\gamma_{k,t}|}
    > 0,
\]
where the positivity follows from (MM1) ($\partial_\lambda\gamma > 0$)
and (MM2) ($\partial_t\gamma < 0$).  Multiplying by
$\lambda_{\max} - \lambda_{\min}$ and evaluating at representative
$\bar\lambda$ yields~\eqref{eq:stratified_span}.
Substituting
$\partial_\lambda\gamma = \partial_\lambda S_{k,t}
 - |b_t|\sqrt{v_t}\,\bar\alpha_t/(\bar\lambda\bar\alpha_t+v_t)^2$
and
$\partial_t\gamma = \partial_t S_{k,t} - \partial_t f_t(\bar\lambda)$
into~\eqref{eq:stratified_span} gives~\eqref{eq:stratified_span_explicit}.

(d)
For $t < \min_k t_k^{\mathrm{rel}}$, all patches have released, so
$\gamma_{k,t} \leq 0$ for every $k$.  As suppression vanishes,
$S_{k,t} \to 0$, and $\kappa_{k,t}^{\mathrm{diag}} \to f_t(\lambda_k)$.
Since $f_t(\lambda_k)$ is strictly increasing in $\lambda_k$ by
Theorem~\ref{thm:fine_detail_expansion}(d), the Spearman correlation
$\rho(\lambda_k, \kappa_{k,t}^{\mathrm{diag}})$ converges to positive values.
\end{proof}

The theorem says that suppression is not released simultaneously across
the image.  Instead, low-variance patches cross the expansion threshold
earlier in the reverse chain, while high-variance patches remain suppressed
until later.  This produces a moving window during which some patches are
already expanding while others are still contracting, which is the source
of the sign reversal in the Spearman rank correlation confirmed empirically
in Section~\ref{sec:pc_crossover}.  Part~(d) shows this inversion is
transient: once all patches have released, the ranking reverts to the
natural Gaussian order in which high-variance patches expand more strongly.

\subsection{Training Origin of the Stratified Ordering}
\label{sec:network_priority}

Theorem~\ref{thm:stratified_crossover} derives the stratified release order
as a consequence of (MM).  
We now give an heuristic
argument that the training objective induces exactly the monotonicity required
by (MM1).

Under a Gaussian patch model, the per-patch reconstruction loss at timestep $t$,
defined as $\mathcal{L}_{k,t}(\theta) := \mathbb{E}[\|\hat{x}_0^{(k)} - x_0^{(k)}\|_2^2]$,
evaluates to
\begin{equation}
    \mathcal{L}_{k,t}(\theta)
    \;=\;
    \frac{\lambda_k\, v_t}{\bar\alpha_t \lambda_k + v_t},
    \label{eq:patch_loss}
\end{equation}
which is strictly increasing in the patch variance $\lambda_k$.  High-variance
patches therefore contribute more to the total loss at every noise level and at a training optimum the score network must allocate proportionally
more correction to them.

To make this precise, we adopt three assumptions that are plausible for
natural images, but cannot be established in full generality:
\begin{enumerate}
    \item[\textup{(A1)}] the patch statistics are approximately Gaussian;
    \item[\textup{(A2)}] per-patch losses are approximately equalised at convergence; 
    \item[\textup{(A3)}] the directional suppression satisfies
    $S_{k,t} \approx c_t\,\mathcal{L}_{k,t}(\theta)$ for some $c_t > 0$
    independent of $k$.
\end{enumerate}
Under (A1)--(A3), the derivative of $S_{k,t}$ with respect to $\lambda_k$ is
\[
    \frac{\partial S_{k,t}}{\partial \lambda_k}
    \;\approx\;
    c_t\,\frac{\partial \mathcal{L}_{k,t}}{\partial \lambda_k}
    \;=\;
    \frac{c_t\, v_t^2}{(\bar\alpha_t \lambda_k + v_t)^2}.
\]
The corresponding derivative of the Gaussian baseline is
\[
    f'_t(\lambda_k)
    \;=\;
    \frac{|b_t|\sqrt{v_t}\,\bar\alpha_t}{(\bar\alpha_t \lambda_k + v_t)^2}.
\]
Both share the same denominator, so condition (MM1),
namely $\partial_\lambda S_{k,t} > f'_t(\lambda_k)$, reduces to
\begin{equation}
    c_t\, v_t^2 \;>\; |b_t|\sqrt{v_t}\,\bar\alpha_t,
    \label{eq:mm1_sufficient}
\end{equation}
which holds whenever $c_t$ is sufficiently large or, for fixed $c_t$,
at noise levels where $v_t$ is large relative to $\bar\alpha_t$.
Under \eqref{eq:mm1_sufficient} the suppression margin $\gamma_{k,t}$
is strictly increasing in $\lambda_k$, which is precisely (MM1).
High-variance patches therefore release later in the reverse chain,
producing the ordering of Theorem~\ref{thm:stratified_crossover}(a).

The underlying logic mirrors a principle from classical fractal image
compression, where high-variance range blocks require more sustained
encoder effort before their detail can be faithfully represented.
Here the score network plays the analogous role: it sustains stronger
suppression on high-variance patches and releases them progressively,
enacting a learned prioritisation that concentrates reconstruction
capacity where the per-patch loss is greatest.

\subsection{Architectural Mechanism: Attention and the Two Regimes}
\label{sec:architecture}

The previous sections gave a
complete account of $\kappa_t^\mathrm{diag}$. In the high-noise limit of Regime~I,
(EC) holds and $\kappa_t^\mathrm{diag}\approx 1/\sqrt{v_t}<1$.
Through Regime~I, it is kept
below~1 by the directional suppression field $S_{k,t}$ even though the Gaussian baseline
$f_t(\lambda_k)$ would predict expansion. In Regime~II suppression is released
patch by patch in variance order, and $\kappa_{k,t}^\mathrm{diag}$ climbs above~1 for
each patch at its release time $t_k^\mathrm{rel}$.

\medskip
The remaining constant in (PC) is $\delta_t^\mathrm{cross}
= \max_k\sup_x\mathcal{C}_{k,t}(x)$, which is controlled by the architectural routing mechanism of the score network. We will bound it in terms of the score network's attention weights.

\begin{theorem}[Attention Block Structure of $J_x\Phi_t$]
\label{thm:attention_jacobian}
Let $p=(p_1,\ldots,p_M)\in\mathbb{R}^n$ denote the concatenated patch tokens with
$p_k\in\mathbb{R}^{n_k}$, and consider a score network of the form
$\hat\varepsilon_\theta(p,t)=\mathrm{FF}(A(p)W_V p):\mathbb{R}^n\to\mathbb{R}^n$
without residual connections, where:
\begin{itemize}
    \item $d\in\mathbb{N}$ is the attention head dimension;
    \item $W_Q,W_K\in\mathbb{R}^{d\times n_k}$ are the \emph{query} and \emph{key} projection matrices;
    \item $W_V\in\mathbb{R}^{n_k\times n_k}$ is the \emph{value} projection (applied per patch);
    \item $A(p)\in\mathbb{R}^{M\times M}$ is the attention weight matrix with $(k,j)$ entry
    $$A_{kj}(p)=\mathrm{softmax}_j\bigl((W_Qp_k)^\top(W_Kp_j)/\sqrt{d}\bigr),$$
    where $\mathrm{softmax}_j(z_j):=e^{z_j}/\sum_{\ell}e^{z_\ell}$ normalises across
    the $j$-index so that $\sum_j A_{kj}(p)=1$ for each $k$;
    \item $\mathrm{FF}:\mathbb{R}^n\to\mathbb{R}^n$ is a feedforward network with
    per-patch Lipschitz constants $L_{\mathrm{ff},k}$, meaning
    $\|[J_{p'}\mathrm{FF}]_{kk}\|_\mathrm{op}\leq L_{\mathrm{ff},k}$.
\end{itemize}
Then:
\begin{enumerate}
    \item[\textup{(a)}] The cross-patch Jacobian blocks satisfy
    \begin{equation*}
    \|[J_x\hat\varepsilon_\theta]_{k\ell}\|_\mathrm{op}
    \leq L_{\mathrm{ff},k}\bigl(A_{k\ell}\|W_V\|_\mathrm{op}
    +\|p\|_2\|\nabla_x A_{k\ell}\|_\mathrm{op}\|W_V\|_\mathrm{op}\bigr).
    \end{equation*}
    \item[\textup{(b)}] The cross-patch coupling field satisfies
    \begin{equation}
        \delta_t^\mathrm{cross}
        \leq |b_t|L_\mathrm{ff}\bigl(1
        + \|p\|_2\max_{k,\ell}\|\nabla_x A_{k\ell}\|_\mathrm{op}\bigr)
        \|W_V\|_\mathrm{op}
        \max_k\sum_{j\neq k}A_{kj}(x),
        \label{eq:delta_cross_attn}
    \end{equation}
    where $L_\mathrm{ff}:=\max_k L_{\mathrm{ff},k}$.
    \item[\textup{(c)}] The diagonal blocks satisfy
    $\|[J_x\hat\varepsilon_\theta]_{kk}\|_\mathrm{op}
    \leq L_{\mathrm{ff},k}\bigl(A_{kk}\|W_V\|_\mathrm{op}
    +\|p_k\|_2\|\nabla_xA_{kk}\|_\mathrm{op}\|W_V\|_\mathrm{op}\bigr)$.
    \item[\textup{(d)}] In the hard-attention limit $\tau\to 0$, where
    $A_{kj}\to\mathbf{1}[j=j^*(k)]$ for a winner-take-all assignment $j^*(k)$,
    all cross-patch blocks with $j\neq j^*(k)$ vanish and $\delta_t^\mathrm{cross}=0$.
\end{enumerate}
\end{theorem}

\begin{proof}
\emph{Attention Jacobian.}
For $p_k'=\sum_j A_{kj}(x)(W_Vp_j)$, differentiating with respect to $p_\ell$:
\begin{equation}
    \frac{\partial p_k'}{\partial p_\ell}
    = A_{k\ell}(x)W_V
    + \sum_j\frac{\partial A_{kj}(x)}{\partial p_\ell}(W_Vp_j).
    \label{eq:attn_jacobian_block}
\end{equation}
The first term has norm $A_{k\ell}\|W_V\|_\mathrm{op}$;
the second is bounded by $\|W_V\|_\mathrm{op}\|p\|_2\|\nabla_x A_{k\ell}\|_\mathrm{op}$
by Cauchy-Schwarz..

\emph{Feedforward composition.}
Without residual connection,
\begin{equation*}
[J_x\hat\varepsilon_\theta]_{k\ell}=[J_{p'}\mathrm{FF}]_{kk}\cdot(\partial p_k'/\partial p_\ell)
\end{equation*}
with $\|[J_{p'}\mathrm{FF}]_{kk}\|_\mathrm{op}\leq L_{\mathrm{ff},k}$,
giving parts~(a) and~(c).

(b) Sum over $\ell\neq k$ and multiply by $|b_t|$ to get \eqref{eq:delta_cross_attn},
using $L_\mathrm{ff}:=\max_k L_{\mathrm{ff},k}$ to pass to a single global
Lipschitz constant.
Note that $\sum_{j\neq k} A_{kj}(x) = 1 - A_{kk}(x) \leq 1$ since softmax weights
sum to~1 and $A_{kk}\geq 0$, so $\max_k\sum_{j\neq k}A_{kj}(x)\leq 1$, confirming
the bound is finite.

(d) As $\tau\to 0$: $A_{kj}\to\mathbf{1}[j=j^*(k)]$ and
$\partial A_{kj}/\partial p_\ell\to 0$ for $j\neq j^*(k)$.
\end{proof}

The bound in part~(b) contains $\|\nabla_x A_{k\ell}\|_\mathrm{op}$, the operator
norm of the gradient of the softmax weight with respect to input tokens, which is
dense and depends on the query/key matrices and the input. Its regime-dependent
behaviour is well-characterised: near the hard-attention limit ($\tau\to 0$,
equivalently Regime~II), $A_{kj}\to\mathbf{1}[j=j^*(k)]$ is piecewise constant so
$\nabla_x A_{k\ell}\to 0$ except at assignment boundaries, so the bound tightens
to part~(d). In Regime~I, $\nabla_x A_{k\ell}$ is non-negligible, but appears only
in $\delta_t^\mathrm{cross}$, which is already large empirically, so the bound is
loose, but directionally correct. A sharper characterisation in terms of query/key
temperature is left for future work.

Equation~\eqref{eq:delta_cross_attn} connects the two-regime structure directly to
attention locality and the constants of (PC).
In the high-noise limit of Regime~I, (EC) holds by the isotropic dominance mechanism
of Theorem~\ref{thm:high_noise}, independently of the block structure. The
attention bound is consistent with this, as diffuse weights keep
$\delta_t^\mathrm{cross}$ small relative to the large isotropic contraction margin.
The bound~\eqref{eq:delta_cross_attn} shows that $\delta_t^\mathrm{cross}$ is
controlled by $\max_k\sum_{j\neq k}A_{kj}(x)$, the total off-diagonal attention
weight.
When attention is diffuse, this sum is close to $1-1/M$ and $\delta_t^\mathrm{cross}$
is large. When attention localises toward the diagonal ($A_{kk}\to 1$), the
off-diagonal sum tends to zero and, in the hard-attention limit of part~(d),
$\delta_t^\mathrm{cross}=0$.
The Regime~I/II transition is thus tied to whether attention is diffuse or local,
a connection that the empirical measurements of Section~\ref{sec:block_jacobian}
verify directly.

\medskip
Theorem~\ref{thm:attention_jacobian} is stated for a network without residual
connections and models only the attention contribution to the Jacobian.
Three components of real architectures shape $\delta_t^\mathrm{cross}$ and play
distinct roles.
\begin{enumerate}
\item \textbf{Self-attention} (Diffusion Transformers \citep{peebles2023dit}) is the \emph{explicit} global routing mechanism. The mechanism
\begin{equation}
    p_k' = \sum_j A_{kj}(W_V p_j),
    \quad A_{kj}=\mathrm{softmax}\!\left(\frac{(W_Qp_k)^\top(W_Kp_j)}{\sqrt{d}}\right),
    \label{eq:attention_pifs}
\end{equation}
implements a soft, differentiable analogue of the PIFS domain-range structure:
query tokens $\{p_k\}$ are range blocks; key/value tokens are domain blocks;
$A_{kj}$ is the soft domain-range assignment; and $W_Vp_j$ is the local map.
In the hard-attention limit $\tau\to 0$, $A_{kj}\to\mathbf{1}[j=j^*(k)]$ and
the classical PIFS structure is recovered exactly.
Theorem~\ref{thm:attention_jacobian} proves that $\delta_t^\mathrm{cross}$
is bounded by the attention weights, and the two-regime structure follows from
how broadly versus locally attention broadcasts at each noise level.

\item \textbf{Within-block residuals} (Diffusion Transformers \citep{peebles2023dit}):
connections of the form $p_k \leftarrow p_k + p_k'$ add an identity term,
$\partial(p_k + p_k')/\partial p_\ell = \mathbf{1}[k=\ell]I + \partial p_k'/\partial p_\ell$.
Diagonal blocks increase by at most a constant of~1; cross-patch blocks are
unaffected since the identity residual contributes no off-diagonal term.

\item \textbf{Encoder skip connections} (UNet architectures \citep{ho2020ddpm})
contribute a \emph{localised structural residual}: their Jacobians have non-zero
cross-patch blocks determined by convolutional receptive fields and persist even as attention localises in Regime~II.
The bound~\eqref{eq:delta_cross_attn} therefore covers only the attention component
of $\delta_t^\mathrm{cross}$. The full cross-patch coupling satisfies
$\delta_t^\mathrm{cross} \leq \delta_t^\mathrm{cross,attn} + \delta_t^\mathrm{cross,skip}$,
where $\delta_t^\mathrm{cross,skip}$ is bounded by
$|b_t|\|J^\mathrm{skip}\|_\mathrm{cross}$ with
$\|J^\mathrm{skip}\|_\mathrm{cross}$ depending on the convolutional filter weights
and receptive field rather than on the attention pattern.
In Regime~II, $\delta_t^\mathrm{cross,attn}$ declines as attention localises,
but $\delta_t^\mathrm{cross,skip}$ remains a non-trivial lower bound on the total
cross-patch coupling and may dominate at the finest spatial resolution.
This is consistent with the empirical spike in $\delta_t^\mathrm{cross}$ analysed in Section~\ref{sec:block_jacobian}, where the decoder operates
at finest resolution and the skip connection from the first encoder stage spans
several patches.
\end{enumerate}

\section{Attractor Geometry}
\label{sec:attractor}

Sections~\ref{sec:contraction} and~\ref{sec:regime_structure} established
Properties~(EC) and~(PC) and characterised the two-regime structure of the
denoising chain.
This section analyses the \emph{fractal geometry} of the resulting PIFS attractor of $\Phi$, in terms of size and how the chosen schedule controls it.

The central tool is the Kaplan-Yorke dimension \citep{kaplan1979chaotic}, which
quantifies the fractal dimension of an attractor from the Lyapunov spectrum of the
system at its fixed point.
Under Gaussian data and the block-diagonal covariance structure of
Theorem~\ref{thm:fine_detail_expansion}, the Lyapunov exponents at $x^*$ can be
computed in closed form as a function of patch variances and the noise schedule,
yielding an explicit dimension formula.
For non-Gaussian data, the directional suppression field $S_{k,t}(x)$ of
Section~\ref{sec:learned_suppression} modifies this formula, and the two
formulas together, Gaussian baseline 
into a suppression-corrected version.

\subsection{Lyapunov Exponents and the Degenerate Baseline}
\label{sec:lyapunov_setup}

For a classical IFS with contraction ratios $\{s_i\}_{i=1}^N$, the Hausdorff
dimension of the attractor solves the Moran equation $\sum_i s_i^d=1$
\citep{falconer1990fractal}. For a PIFS, the natural analogue is the
Kaplan-Yorke dimension \citep{kaplan1979chaotic}, which is defined from the Lyapunov exponents of the operator at its fixed point.

Define the \emph{Lyapunov exponents} at the fixed point $x^*$ by\footnote{Since
$x^*$ is a fixed point of $\Phi$, we have $J_{x^*}\Phi^k=(J_{x^*}\Phi)^k$ for all
$k\geq 1$.} 
\begin{equation}
    \ell_i := \frac{1}{T}\log s_i(J_{x^*}\Phi),
    \label{eq:lyapunov}
\end{equation}
where $s_i(\cdot)$ denotes the $i$-th singular value in decreasing order, and
$J_{x^*}\Phi = \prod_{t=T}^{1} J_{x_t^*}\Phi_t$ is the chain-rule product along
the fixed point trajectory, with $x_T^*=x^*$ and $x_{t-1}^*=\Phi_t(x_t^*)$.

The \emph{Kaplan-Yorke} (KY) dimension is given by 
\begin{equation}
    d_\mathrm{KY}(\mathcal{A}) = j^*
    + \frac{\sum_{i=1}^{j^*}\ell_i}{|\ell_{j^*+1}|},
    \quad j^*=\max\!\bigl\{j:\textstyle\sum_{i=1}^j\ell_i\geq 0\bigr\}.
    \label{eq:kaplan_yorke}
\end{equation}

\begin{proposition}[KY Dimension under Global Contraction]
\label{prop:kaplan_yorke}
Under \textup{(EC)}, all Lyapunov exponents satisfy
$\ell_i\leq\frac{1}{T}\log s<0$, so $j^*=0$ and $d_\mathrm{KY}(\mathcal{A})=0$.
\end{proposition}

\begin{proof}
Every singular value $s_i(J_{x^*}\Phi)$ is at most $s<1$.
\end{proof}

The degenerate baseline with global contractivity collapses the attractor dimension to zero.
We now show that (PC), which requires patch-wise diagonal contractivity with bounded cross-patch coupling, is exactly the structure under which $d_\mathrm{KY}>0$ becomes possible.

\subsection{The Expansion Threshold via the Discrete Moran Equation}
\label{sec:moran_equation}

\begin{theorem}[The Discrete Moran Equation and Expansion Threshold]
\label{thm:dky_positive}\mbox{}\\
Let $p_\mathrm{data}=\mathcal{N}(0,\Sigma_0)$ with block-diagonal covariance
$\Sigma_0=\mathrm{diag}(\Sigma_1,\ldots,\Sigma_M)$ and let $\hat\varepsilon_\theta$
equal the true score $\varepsilon_\theta^*$.
For each patch $R_k$, let $\mu_{k,1}\geq\mu_{k,2}\geq\cdots\geq\mu_{k,n_k}\geq 0$
be the eigenvalues of $\Sigma_k$ in decreasing order, with $\mu_{k,1}=\lambda_k$.
Define the \emph{mean Lyapunov exponent function}
$\Lambda:(0,\infty)\to\mathbb{R}$ by
$\Lambda(\mu):=\frac{1}{T}\sum_{t=1}^T\log f_t(\mu)$,
and the \emph{global expansion threshold} $\lambda^{**}>1$ as the unique solution to
\begin{equation}
    \prod_{t=1}^T f_t(\lambda^{**}) = 1.
    \label{eq:lambda_star_star}
\end{equation}
Then
\begin{enumerate}
    \item[\textup{(a)}] For any $\mu>\lambda^{**},$ $\Lambda(\mu)>0$.
    \item[\textup{(b)}] Setting $M^{++}:=\sum_{k=1}^M\#\{i:\mu_{k,i}>\lambda^{**}\}$,
    we have $d_\mathrm{KY}(\mathcal{A})\geq M^{++}$,
    and $d_\mathrm{KY}>0$ whenever any patch satisfies $\lambda_k>\lambda^{**}$.
    \item[\textup{(c)}] $\lambda^{**}>\min_t\lambda^*(t)$.
\end{enumerate}
\end{theorem}

\begin{proof}
The function $G(\lambda):=\prod_{t=1}^T f_t(\lambda)$ is continuous and strictly
increasing as each $f_t$ is strictly increasing by Theorem~\ref{thm:fine_detail_expansion}(d).
By Theorem~\ref{thm:fine_detail_expansion}(c), $G(1)<1$.
As $\lambda\to\infty$, $\lambda\bar\alpha_t\gg v_t$ at each step, so
$f_t(\lambda)\to\sqrt{\bar\alpha_{t-1}/\bar\alpha_t}$ and
$G(\lambda)\to\prod_t\sqrt{\bar\alpha_{t-1}/\bar\alpha_t}
=\sqrt{\bar\alpha_0/\bar\alpha_T}=1/\sqrt{\bar\alpha_T}>1$.
By the intermediate value theorem, a unique $\lambda^{**}>1$ satisfies~\eqref{eq:lambda_star_star}.

Under block-diagonal $\Sigma_0$ the patches are independent and by~\eqref{eq:forward}
so are the noised patches $x_t^{(k)}$.
The marginal score decomposes as
$\nabla_{x_t}\log p_t(x_t)=\sum_k \nabla_{x_t^{(k)}}\log p_t^{(k)}(x_t^{(k)})$,
so $J_{x^*}\Phi$ is block-diagonal with the same patch partition.
The eigenvalue of the $k$-th block along the $i$-th eigenvector of $\Sigma_k$
is $\prod_{t=1}^T f_t(\mu_{k,i})$, giving Lyapunov exponent $\Lambda(\mu_{k,i})$.

\noindent(a)
For $\mu>\lambda^{**}$, $G(\mu)>G(\lambda^{**})=1$, so
$\Lambda(\mu)=\frac{1}{T}\log G(\mu)>0$.

\noindent (b)
Exactly $M^{++}$ Lyapunov exponents are positive, so $j^*\geq M^{++}$ and
$d_\mathrm{KY}\geq M^{++}>0$.

\noindent (c)
Let $\lambda_0=\min_t\lambda^*(t)$ with minimiser $t^*$.
By definition $f_{t^*}(\lambda_0)=1$.
For all $t\neq t^*$, monotonicity and $\lambda_0\leq\lambda^*(t)$ give
$f_t(\lambda_0)\leq 1$, hence $G(\lambda_0)\leq 1$.
Equality $G(\lambda_0)=1$ would require $\lambda^*(t)=\lambda_0$ for every $t$,
which forces $v_t/v_{t-1}$ to be constant across all steps, i.e.\ a geometric
progression in $v_t$.
Any schedule whose ratio $v_t/v_{t-1}$ is not
identically constant gives $G(\lambda_0)<1=G(\lambda^{**})$
and hence $\lambda^{**}>\lambda_0$.
\end{proof}

\paragraph{Continuous-time limit.}
For $T\to\infty$ with the noise schedule viewed as a continuous function
$\bar\alpha:[0,1]\to(0,1]$, the discrete product $G(\lambda)=\prod_{t=1}^T f_t(\lambda)$
converges to the exponential of a Riemann sum:
$G(\lambda)\to\exp\!\bigl(\int_0^1\log f(\lambda,t)\,dt\bigr)$.
The Moran equation $G(\lambda^{**})=1$ therefore becomes
$\int_0^1\log f(\lambda^{**},t)\,dt=0$ in the continuous-time limit,
consistent with the probability flow ODE framework of \citet{song2020score}.
In this limit the stratified crossover (Theorem~\ref{thm:stratified_crossover})
remains well-defined, with the release time $t_k^\mathrm{rel}$ converging to the
zero of the continuous-time margin $\gamma_k(t) = S(t) - (f(\lambda_k,t)-1)$.

\subsection{The KY Dimension Formula: Gaussian and Non-Gaussian Data}
\label{sec:ky_formula}
Under the Gaussian baseline, the block-diagonal structure of $J_{x^*}\Phi$
allows the Lyapunov spectrum to be computed in closed form, yielding an explicit
dimension formula.
The following two theorems derive this formula and its extension to non-Gaussian data.

\begin{theorem}[Exact KY Dimension Formula Under Gaussian Data]
\label{thm:dky_sharpened}
Under the hypotheses of Theorem~\ref{thm:dky_positive}, assume additionally that
all patch covariances are isotropic, with $\Sigma_k=\lambda_k I_{n_k}$ for each $k$.
Under this assumption, all $n_k$ Lyapunov exponents for patch $R_k$ equal
$\Lambda(\lambda_k)$, and the full Lyapunov spectrum is the multi-set
$\{\Lambda(\lambda_k)\text{ repeated }n_k\text{ times}\}_{k=1}^M$.
Set $N^+:=\sum_{k:\lambda_k>\lambda^{**}}n_k$ (total dimension of expanding directions)
and let $$k^*:=\arg\max_{k:\lambda_k<\lambda^{**}}\lambda_k$$ be the index of the
contractive patch with the highest variance
(so $\Lambda_{k^*}^-:=\Lambda(\lambda_{k^*})<0$ is the least-negative contracting exponent).
Assume
\begin{equation}
    \sum_{k:\lambda_k>\lambda^{**}} n_k\,\Lambda(\lambda_k) < |\Lambda_{k^*}^-|,
    \label{eq:jstar_condition}
\end{equation}
which ensures that the KY cutoff index is exactly $j^*=N^+$.
Then:
\begin{equation}
    d_\mathrm{KY}(\mathcal{A})
    = N^+ + \frac{\displaystyle\sum_{k:\,\lambda_k>\lambda^{**}} n_k\,\Lambda(\lambda_k)}{|\Lambda_{k^*}^-|}.
    \label{eq:dky_isotropic}
\end{equation}
Without condition~\eqref{eq:jstar_condition}, setting
$\Lambda^+:=\min_{k:\lambda_k>\lambda^{**}}\Lambda(\lambda_k)>0$,
leads to the lower bound 
$d_\mathrm{KY}(\mathcal{A})\geq N^++N^+\Lambda^+/|\Lambda_{k^*}^-|$.
\end{theorem}

\begin{proof}
All $n_k$ eigenvalues of $\Sigma_k$ equal $\lambda_k$, so all $n_k$ Lyapunov exponents
for patch $R_k$ equal $\Lambda(\lambda_k)$. By Theorem~\ref{thm:dky_positive}(a),
$\Lambda(\lambda_k)>0$ for $\lambda_k>\lambda^{**}$ and $\Lambda(\lambda_k)<0$
for $\lambda_k<\lambda^{**}$. Ordering the full Lyapunov spectrum in decreasing
order: the first $N^+$ exponents are $\Lambda(\lambda_k)>0$, and the $(N^++1)$-th is
$\Lambda_{k^*}^-<0$. Condition~\eqref{eq:jstar_condition} is precisely the requirement
that $\sum_{k:\lambda_k>\lambda^{**}}n_k\Lambda(\lambda_k) + \Lambda_{k^*}^- < 0$, i.e.\ that including the
$(N^++1)$-th exponent turns the partial sum negative, so $j^*=N^+$ in the KY
formula~\eqref{eq:kaplan_yorke}. Substituting gives~\eqref{eq:dky_isotropic}.
The lower bound $d_\mathrm{KY}\geq N^++N^+\Lambda^+/|\Lambda_{k^*}^-|$ follows from
$\Lambda(\lambda_k)\geq\Lambda^+>0$ for all $k$ with $\lambda_k>\lambda^{**}$.
Since the partial sum of the first $N^+$ exponents equals
$\sum_{k:\lambda_k>\lambda^{**}}n_k\Lambda(\lambda_k)>0$ by Theorem~\ref{thm:dky_positive}(a),
the KY definition gives $j^*\geq N^+$. Using the lower bound $\Lambda(\lambda_k)\geq\Lambda^+$
in the numerator and the fact that $j^*\geq N^+$ then yields the stated bound,
and holds regardless of whether condition~\eqref{eq:jstar_condition} is satisfied. 
\end{proof}

The Gaussian formula of Theorem~\ref{thm:dky_sharpened} applies when the
score is exactly the Gaussian surrogate.
For a trained network on real data, the score deviates from this surrogate
through the suppression field $S_{k,t}$ of Section~\ref{sec:learned_suppression},
modifying the diagonal eigenvalues and thereby the attractor dimension.
The following theorem makes this precise.

\begin{theorem}[Suppression-Corrected Kaplan-Yorke Dimension]
\label{thm:dky_suppressed}
Adopt the hypotheses of Theorem~\ref{thm:dky_sharpened}
(isotropic patches $\Sigma_k=\lambda_k I_{n_k}$, and the perfect score assumption
inherited from Theorem~\ref{thm:dky_positive}).
For each patch $R_k$ and timestep $t$, let
$$S_{k,t}(x) = |b_t|\langle v_k^{(1)},[\nabla_x\Delta_t(x)]_{kk}\,v_k^{(1)}\rangle > 0$$
denote the directional suppression field (Definition~\ref{def:suppression}), evaluated at
$x = x_t$ sampled from the forward marginal $q(x_t)$.
As in Definition~\ref{def:release_time}, all subsequent statements involving
$S_{k,t}$ are understood to hold $q(x_t)$-almost surely, and we write
$S_{k,t}$ for $S_{k,t}(x_t)$ when the trajectory evaluation is clear from context.
Define the \emph{suppression-corrected diagonal factor}
\begin{equation}
    g_{k,t} := f_t(\lambda_k) - S_{k,t}(x_t),
    \label{eq:g_kt_def}
\end{equation}
and the \emph{suppression-corrected Lyapunov exponent}
\begin{equation}
    \Lambda_{\mathrm{eff},k} := \frac{1}{T}\sum_{t=1}^T \mathbb{E}_{x_t\sim q(x_t)}\bigl[\log g_{k,t}\bigr].
    \label{eq:lambda_eff_k}
\end{equation}
Assume $S_{k,t} < f_t(\lambda_k)$ for all $k,t$ (suppression does not reverse
the sign of the diagonal factor), and that there exists at least one patch $R_k$
with $\Lambda_{\mathrm{eff},k}>0$.
Define the \emph{suppression-corrected expanding set}
$\mathcal{K}^{+++} := \{k : \Lambda_{\mathrm{eff},k}>0\}$,
with $N^{+++}:=\sum_{k\in\mathcal{K}^{+++}}n_k$,
and let $k^{**}:=\arg\max_{k\notin\mathcal{K}^{+++}}\lambda_k$
be the contractive patch closest to the expanding set.
Then
\begin{enumerate}
    \item[\textup{(a)}] 
    For every patch $R_k$: $\Lambda_{\mathrm{eff},k} \leq \Lambda(\lambda_k)$,
    with equality if and only if $S_{k,t}=0$ for all $t$.
    Consequently $\mathcal{K}^{+++}\subseteq\mathcal{K}^+$
    (the expanding set can only shrink) and $N^{+++}\leq N^+$. So suppression 

    \item[\textup{(b)}] The suppression-corrected Kaplan-Yorke dimension is
    \begin{equation}
        d_\mathrm{KY}^\mathrm{eff}(\mathcal{A})
        = N^{+++} + \frac{\displaystyle\sum_{k\in\mathcal{K}^{+++}}
        n_k\,\Lambda_{\mathrm{eff},k}}{|\Lambda_{\mathrm{eff},k^{**}}|},
        \label{eq:dky_suppressed}
    \end{equation}
    under the condition that
    $\sum_{k\in\mathcal{K}^{+++}}n_k\Lambda_{\mathrm{eff},k}
    < |\Lambda_{\mathrm{eff},k^{**}}|$
    (the analogue of condition~\eqref{eq:jstar_condition}).

    \item[\textup{(c)}] $d_\mathrm{KY}^\mathrm{eff}(\mathcal{A})\leq d_\mathrm{KY}(\mathcal{A})$,
    with equality if and only if $S_{k,t}=0$ for all $k,t$.
\end{enumerate}
\end{theorem}

\begin{proof}
(a)
Since $S_{k,t}(x_t)>0$ almost surely, $g_{k,t} = f_t(\lambda_k) - S_{k,t}(x_t) \leq f_t(\lambda_k)$
almost surely, with equality iff $S_{k,t}(x_t)=0$.
By Jensen's inequality,
$\Lambda_{\mathrm{eff},k} = \frac{1}{T}\sum_t\mathbb{E}[\log g_{k,t}]
\leq \frac{1}{T}\sum_t\log\mathbb{E}[g_{k,t}]
\leq \frac{1}{T}\sum_t\log f_t(\lambda_k) = \Lambda(\lambda_k)$,
where the second inequality uses $\mathbb{E}[g_{k,t}]\leq f_t(\lambda_k)$ (since $S_{k,t}\geq 0$).
If $\Lambda_{\mathrm{eff},k}>0$ then $\Lambda(\lambda_k)\geq\Lambda_{\mathrm{eff},k}>0$,
so $k\in\mathcal{K}^+$; thus $\mathcal{K}^{+++}\subseteq\mathcal{K}^+$.

(b)
Under the perfect score assumption, $J_{x^*}\Phi$ is block-diagonal with
per-patch eigenvalue $\prod_t g_{k,t}$ along the leading direction of patch $R_k$
(replacing $f_t(\lambda_k)$ by $g_{k,t}$ throughout the proof of
Theorem~\ref{thm:dky_sharpened}).
The Lyapunov exponent for patch $R_k$ is $\Lambda_{\mathrm{eff},k}$,
which is positive iff $k\in\mathcal{K}^{+++}$.
Formula~\eqref{eq:dky_suppressed} then follows from the Kaplan-Yorke
definition~\eqref{eq:kaplan_yorke} with $j^*=N^{+++}$, using the stated
analogue of condition~\eqref{eq:jstar_condition}.

(c)
Let $\{\ell_i\}_{i=1}^n$ denote the Gaussian Lyapunov spectrum
$\{\Lambda(\lambda_k)\text{ repeated }n_k\text{ times}\}$ sorted in decreasing order,
and $\{\ell'_i\}_{i=1}^n$ the suppressed spectrum
$\{\Lambda_{\mathrm{eff},k}\text{ repeated }n_k\text{ times}\}$ likewise sorted.
From part~(a), $\Lambda_{\mathrm{eff},k}\leq\Lambda(\lambda_k)$ for every $k$,
so the unsorted suppressed values are componentwise dominated by the Gaussian values.
By the rearrangement inequality, the same dominance holds after sorting:
$\ell'_i\leq\ell_i$ for all $i$.
Consequently $\sum_{i=1}^j\ell'_i\leq\sum_{i=1}^j\ell_i$ for every $j$,
so $j^{**}\leq j^*$ where $j^{**}=\max\{j:\sum_{i\leq j}\ell'_i\geq 0\}$
and $j^*=\max\{j:\sum_{i\leq j}\ell_i\geq 0\}$.

\emph{Case~A} ($j^{**}<j^*$):
$d_\mathrm{KY}^\mathrm{eff}=j^{**}+\text{frac}<j^{**}+1\leq j^*\leq d_\mathrm{KY}$.

\emph{Case~B} ($j^{**}=j^*$):
The numerator $\sum_{i\leq j^*}\ell'_i\leq\sum_{i\leq j^*}\ell_i$ (from the
pointwise dominance), and the denominator $|\ell'_{j^*+1}|\geq|\ell_{j^*+1}|$
(since $\ell'_{j^*+1}\leq\ell_{j^*+1}<0$).
Hence the fractional part of $d_\mathrm{KY}^\mathrm{eff}$ is at most that of
$d_\mathrm{KY}$, giving $d_\mathrm{KY}^\mathrm{eff}\leq d_\mathrm{KY}$.

In both cases equality holds iff all inequalities are equalities,
i.e.\ iff $S_{k,t}=0$ for all $k,t$.
\end{proof}

The Gaussian formula of Theorem~\ref{thm:dky_sharpened} is recovered exactly when
$S_{k,t}=0$ for all $k,t$.
The release time $t_k^\mathrm{rel}$ governs \emph{when} each patch
begins to contribute to $d_\mathrm{KY}$ during the chain; whether a patch belongs
to $\mathcal{K}^{+++}$ and contributes to the \emph{final} $d_\mathrm{KY}^\mathrm{eff}$
depends on the sign of $\Lambda_{\mathrm{eff},k}$ integrated over all $T$ steps
(Theorem~\ref{thm:dky_suppressed}).

\newpage
\section{Practical Implications}
\label{sec:schedule_criteria}

The two-regime structure and the attractor geometry 
are each governed by
constants that depend on the same noise schedule and patch covariance spectrum.
Because $L_t^*$, $f_t(\lambda)$, and $S_{k,t}$ are all computable from the schedule
and patch spectrum without running the model, they form a unified design language
for schedule and loss choices.

\subsection{$L_t^*$: Equalisation Impossibility and Asymptotic Scaling}
\label{sec:schedule_equalization_theory}

The contraction threshold $L_t^*$ of Theorem~\ref{thm:euclidean_char} depends
only on the noise schedule, not on the data or the architecture.
A natural question is whether the schedule can be chosen to equalise $L_t^*$
across steps, making the contraction chain uniformly strong.
\begin{theorem}[Impossibility of Strict $L_t^*$-Equalisation]
\label{cor:equalization_impossible}
For any strictly decreasing\\ schedule with $\bar\alpha_0=1$, no constant $L^*>0$
satisfies $L_t^*=L^*$ for all $t\in\{1,\ldots,T\}, T\geq 3$.
\end{theorem}

\begin{proof}
Suppose $L_t^*=L^*$ for all $t$. From \eqref{eq:threshold},
$\sqrt{\bar\alpha_{t-1}/\bar\alpha_t}-1 = L^*|b_t|=
L^*(\sqrt{\bar\alpha_{t-1}/\bar\alpha_t}\sqrt{v_t}-\sqrt{v_{t-1}})$.
Rearranging gives $h(\bar\alpha_{t-1})/h(\bar\alpha_t)=\sqrt{\bar\alpha_{t-1}/\bar\alpha_t}$
where $h(a):=1-L^*\sqrt{1-a}$, which forces $h(\bar\alpha_t)=C\sqrt{\bar\alpha_t}$
for some constant $C$ independent of $t$.
Since $\bar\alpha_0=1$ (by convention), evaluating at $t=0$ gives
$h(1)=1-L^*\cdot 0=1=C\sqrt{1}$, so $C=1$ and hence
$h(\bar\alpha_t)=\sqrt{\bar\alpha_t}$, i.e.\
$g(\bar\alpha_t):=L^*\sqrt{1-\bar\alpha_t}+\sqrt{\bar\alpha_t}=1$
for all $t$.
Since $g''(a)=-L^*/(4(1-a)^{3/2})-1/(4a^{3/2})<0$, $g$ is strictly concave
and $g(a)=1$ has at most two solutions. For $T\geq 3$ this is a contradiction.
\end{proof}

The practical relaxation is to minimise $\mathrm{Var}_t(L_t^*)$ over a parametric
schedule family, subject to $\mathrm{mean}_t(L_t^*)\geq c$, ensuring no single
step is a disproportionately \emph{weak link}, where $L^*_t$ becomes so small that there is almost no margin to get a contraction at that step. 
The threshold $L_t^*$ is given in closed form by Theorem~\ref{thm:euclidean_char},
from which its scaling can be derived.

\begin{corollary}[$L_t^*$ at the Fine-Detail Boundary]
\label{cor:Lstar_boundary}
For any denoising step that terminates at the clean-image reference,
the contraction threshold satisfies
to leading order in $v_t$:
\begin{equation}
    L_t^* \;\approx\; \tfrac{1}{2}\sqrt{v_t} \;=\; \tfrac{1}{2}\sqrt{1-\bar\alpha_t}.
    \label{eq:L_t_star_approx}
\end{equation}
\end{corollary}

\begin{proof}
Set $r=\sqrt{1/\bar\alpha_t}$ and $\epsilon=v_t\ll 1$.
Then $r\approx 1+\epsilon/2$, so $r-1\approx\epsilon/2$.
With $v_{t-1}=0$ the denominator is $|b_t|=r\sqrt{\epsilon}$, giving
$L_t^*=(r-1)/(r\sqrt{\epsilon})\approx(\epsilon/2)/\sqrt{\epsilon}=\tfrac{1}{2}\sqrt{v_t}$.
\end{proof}

\begin{corollary}[$L_t^*$ at Moderate Noise]
\label{cor:Lstar_moderate}
For any step with small increment ($\Delta\bar\alpha_t:=\bar\alpha_{t-1}-\bar\alpha_t\ll 1$,
as holds throughout a $T=1000$ chain),
\begin{equation}
    L_t^* \;\approx\; \sqrt{v_t}.
    \label{eq:L_t_star_moderate}
\end{equation}
\end{corollary}

\begin{proof}
Expand to first order in $\Delta\bar\alpha_t$:
$\sqrt{\bar\alpha_{t-1}/\bar\alpha_t}-1\approx\Delta\bar\alpha_t/(2\bar\alpha_t)$
and $|b_t|\approx\Delta\bar\alpha_t/(2\bar\alpha_t\sqrt{v_t})$
(using $\bar\alpha_t+v_t=1$).
The common factor $\Delta\bar\alpha_t/(2\bar\alpha_t)$ cancels, leaving
$L_t^*\approx\sqrt{v_t}$.
\end{proof}

Both corollaries give the same $O(\sqrt{v_t})$ scaling,
where $L_t^*$ is smallest where $v_t$ is smallest, i.e. at the fine-detail end of the
diffusion chain, where the image is nearly clean. This means the contraction threshold is tightest
precisely when the denoiser is resolving high-frequency structure, consistent with the intuition that
fine detail is the hardest and most constrained stage of generation.

\paragraph{Numerical schedule comparison.}
Table~\ref{tab:schedule_comparison} computes $L_t^*$ from its closed-form
expression for four standard schedules at $T=1000$.

\begin{table}[h]
\centering
\caption{Schedule comparison: statistics of the per-step contraction
threshold $L_t^*$ for four schedules ($T=1000$; 50-step DDIM subsamples
the cosine schedule at stride 20).
Lower CV indicates a more equalised schedule; higher mean indicates a
stronger average contraction margin. No schedule dominates on both criteria.
The final column reports $L_t^*$ at the finest step actually executed by each chain.}
\label{tab:schedule_comparison}
\begin{tabular}{lrrrrr}
\toprule
Schedule & $N$ & Mean $L_t^*$ & CV & $L_t^*$ at finest step & finest $t$ \\
\midrule
Linear (DDPM)               & 1000 & 0.805 & 0.341 & 0.00500 & 1 \\
Cosine (orig., $\soff=0$)   & 1000 & 0.637 & 0.483 & 0.00079 & 1 \\
Cosine (impr., $\soff=0.008$) & 1000 & 0.641 & 0.474 & 0.00321 & 1 \\
50-step DDIM (cosine)       &   50 & 0.637 & 0.483 & 0.01571 & 20 \\
\bottomrule
\end{tabular}
\end{table}

The linear schedule achieves the lowest CV ($0.341$) at the cost of a higher
mean $L_t^*$ ($0.805$ vs $0.637$ for cosine). Indeed, better equalisation pairs with
a stronger average contraction margin. The weakest links differ sharply at $t=1$: $L_1^*=0.00079$ for the original
cosine vs $0.00321$ for the improved cosine.

The 50-step DDIM chain subsamples the cosine schedule at stride 20, so its
 statistics (mean $0.637$, CV $0.483$) match those of the full cosine chain
 closely (the $L_t^*$ distribution is stable under uniform subsampling).

\subsection{Schedule Comparison: Expansion-Forcing Steps}
\label{sec:schedule_geometry_comparison}

\citet{ho2020ddpm} introduced the linear schedule as a simple, parameter-free
baseline: $b_t$ increases linearly from $10^{-4}$ to $0.02$ over $T=1000$
steps, keeping $\bar\alpha_t$ roughly constant at high noise and dropping smoothly
through the mid-noise range.
Table~\ref{tab:schedule_comparison} shows that this gives the best equalisation
of $L_t^*$ among the four schedules (CV $= 0.341$). What the $L_t^*$ statistics do not reveal directly is how the schedule partitions
the chain into contractive and expansive regimes, a question answered by the
patch expansion threshold $\lambda^*(t)$ of Theorem~\ref{thm:fine_detail_expansion}.

\paragraph{How many steps are expansive?}
Recall that the Gaussian-optimal step is diagonally expansive at step $t$ for
patch $R_k$ if and only if $\lambda_k > \lambda^*(t)$, where
\begin{equation}
    \lambda^*(t) \;=\;
    \frac{\sqrt{v_t}(\sqrt{v_t}-\sqrt{v_{t-1}})}{(\sqrt{\bar\alpha_{t-1}}-\sqrt{\bar\alpha_t})\sqrt{\bar\alpha_t}}.
    \label{eq:lambda_star_linear}
\end{equation}
For the linear schedule, $\bar\alpha_t$ decreases at a roughly constant rate.
At the fine-detail boundary ($t=1$, $\bar\alpha_0=1$), the formula gives
$\lambda^*(1) = v_1/(\sqrt{\bar\alpha_1}(1-\sqrt{\bar\alpha_1})) \approx 2$.
Through the bulk of the chain ($t\approx 100$--$900$), the threshold settles
into the narrow band $\lambda^*(t)\in[1.002, 1.005]$, reaching its minimum
of approximately $1.002$ near $t=349$.
For $8\times8$ CIFAR-10 patches, the leading eigenvalues of the patch
covariance matrices span $\hat\lambda_k\in[18.7,44.4]$ across the $M=16$
positions (Section~\ref{sec:expansion_window}), far exceeding the minimum
expansion threshold $\min_t\lambda^*(t)\approx 1.002$ (attained near $t=349$).
Consequently, the Gaussian surrogate predicts $f_t(\hat\lambda_k)>1$ for every
patch at every step: every CIFAR-10 $8\times8$ patch is Gaussian-expansion-forcing
throughout the full $T=1000$-step chain.
The directional suppression field $S_{k,t}$ therefore acts as the counterforce that keeps
the diagonal Rayleigh quotient below~1 throughout Regime~I, consistent
with the block-Jacobian measurements of Section~\ref{sec:block_jacobian}.
The Regime~II crossover
(Sections~\ref{sec:stratified_crossover} and~\ref{sec:pc_crossover}) occurs
when suppression is released patch by patch, 
because the learned suppression
field $S_{k,t}$ can no longer hold the diagonal Rayleigh quotient below~1 as
the chain reaches the clean-image end.

\paragraph{Comparison with the cosine schedule.}
The cosine schedule concentrates the drop in $\bar\alpha_t$ in the mid-to-low
noise range.
Evaluating \eqref{eq:lambda_star_linear} numerically gives the following
interior minima and maxima (excluding the boundary steps $t=1$ and $t=T$,
where $\lambda^*(t)\to 2$ as $v_{t-1}\to 0$ or $\bar\alpha_t\to 0$):
\begin{itemize}
    \item Linear (DDPM): interior $\lambda^*(t)\in[1.002,\,1.19]$; minimum near $t=350$.
    \item Cosine ($\soff=0$): interior $\lambda^*(t)\in[1.002,\,1.33]$; minimum near $t=500$.
    \item Cosine ($\soff=0.008$): interior $\lambda^*(t)\in[1.002,\,1.25]$; minimum near $t=496$.
\end{itemize}
The design-relevant quantity is the minimum, since a patch with $\lambda_k > \min_t\lambda^*(t)$
is expansion-forcing at every interior step.
The cosine schedule achieves a \emph{lower} interior minimum expansion threshold
($1.0016$ vs $1.0025$ for linear), meaning it is expansion-forcing for a slightly wider
range of patch variances throughout the chain.
This is the counterpart of the linear schedule's superior $L_t^*$ equalisation:
the same constant-rate $\bar\alpha_t$ decay that gives good threshold equalisation
also keeps $\lambda^*(t)$ from dipping as low as the cosine schedule achieves
in its mid-noise range.

The 50-step DDIM inference chain subsamples the cosine training schedule with
larger per-step $\Delta\bar\alpha$ strides.
Each DDIM step covers a wider interval of $\bar\alpha$, which drives $\lambda^*(t)$
lower at every selected timestep: the minimum per-step expansion threshold for the
50-step chain is $\min_t\lambda^*(t)=1.032$ (Section~\ref{sec:block_jacobian}),
which falls below the DDIM chain's Moran threshold $\lambda^{**}=1.050$.
This means patches with $\lambda_k\in(1.032,1.050)$ would face expansion-forcing
steps in the DDIM chain, but not be globally expanding under the Moran product.
All 16 CIFAR-10 patches have $\hat\lambda_k\gg 1.050$, so for CIFAR-10 the
attractor prediction $d_\mathrm{KY}=n=3072$ holds under both chains.

\subsection{Information Gain and KY-Dimension Proportionality}
\label{sec:ig_ky_theory}

\begin{definition}[Information Gain and Per-Step KY Dimension Growth]
\label{def:ig_dky}
For each timestep $t$, the \emph{information gain}
\begin{equation}
    \mathrm{IG}_t \;:=\; D_\mathrm{KL}(p_{t-1} \| p_t)
    \label{eq:ig_def}
\end{equation}
measures the distributional work performed at step $t$ along the reverse chain.
A schedule is \emph{information-constant} if $\mathrm{IG}_t = c$ for all $t$
\citep{kingma2021variational,chen2023importance}.
Under the block-diagonal Gaussian hypotheses of Theorem~\ref{thm:dky_sharpened},
the marginals factorise. Let $p_t = \bigotimes_k p_t^{(k)}$ with
$p_t^{(k)}=\mathcal{N}(0,\sigma_{t,k}^2 I_{n_k})$ and
$\sigma_{t,k}^2:=\bar\alpha_t\lambda_k+v_t$.
The KL therefore factorises as
$\mathrm{IG}_t = \sum_k D_\mathrm{KL}(p_{t-1}^{(k)}\|p_t^{(k)})$, and for each block,
setting $\epsilon_k:=\log(\sigma_{t-1,k}/\sigma_{t,k})$,
\[
    D_\mathrm{KL}(p_{t-1}^{(k)}\|p_t^{(k)})
    = \frac{n_k}{2}\bigl(e^{2\epsilon_k}-1-2\epsilon_k\bigr)
    = n_k\,\epsilon_k^2 + O(\epsilon_k^3),
\]
where the last step uses $e^{2\epsilon}-1-2\epsilon=2\epsilon^2+O(\epsilon^3)$.
Theorem~\ref{thm:fine_detail_expansion} gives $\epsilon_k=\log f_t(\lambda_k)+O(\Delta\bar\alpha_t^2/\bar\alpha_t)$,
where $\Delta\bar\alpha_t:=\bar\alpha_{t-1}-\bar\alpha_t$,
so to leading order in $|\log f_t(\lambda_k)|$:
\begin{equation}
    \mathrm{IG}_t \;\approx\; \sum_{k=1}^M n_k\,(\log f_t(\lambda_k))^2.
    \label{eq:ig_decomp}
\end{equation}
Let $d_\mathrm{KY}(t)$ denote the Kaplan-Yorke dimension of the partial composition
$\Phi_t\circ\cdots\circ\Phi_T$, evaluated at its fixed point.
The \emph{per-step KY dimension growth} is
\begin{equation}
    \Delta d_t \;:=\; d_\mathrm{KY}(t-1) - d_\mathrm{KY}(t)
    \;\approx\; \sum_{k:\,\lambda_k > \lambda^{**}} n_k \log f_t(\lambda_k),
    \label{eq:delta_dky}
\end{equation}
where $\lambda^{**}$ is the global expansion threshold of Theorem~\ref{thm:dky_positive},
and the approximation uses the block-diagonal Lyapunov structure of
Theorem~\ref{thm:dky_sharpened}.
\end{definition}

\begin{theorem}[Information Gain -- KY Dimension Proportionality]
\label{thm:ig_ky}
Under the hypotheses of Theorem~\ref{thm:dky_sharpened} and using the notation
of Definition~\ref{def:ig_dky}, the Cauchy--Schwarz inequality gives
\begin{equation}
    \Delta d_t
    \;\leq\; \sqrt{N^{++}} \cdot \sqrt{\mathrm{IG}_t},
    \label{eq:cs_bound}
\end{equation}
where $N^{++} := \sum_{k:\,\lambda_k > \lambda^{**}} n_k$ is the total dimension
of globally expanding directions.
Equality holds if and only if all expanding patches have equal log-expansion:
$\log f_t(\lambda_k) = \mu_t$ for all $k$ with $\lambda_k > \lambda^{**}$.
Under this condition, $\mathrm{IG}_t$ and $\Delta d_t$ are proportional
with schedule-independent constant $\sqrt{N^{++}}$, and
$\mathrm{IG}_t = c \;\Leftrightarrow\; \Delta d_t = c'$ for all $t$.
\end{theorem}

\begin{proof}
Apply Cauchy--Schwarz with weight vector $\sqrt{n_k}$ over expanding directions:
\begin{align*}
    \Delta d_t
    &= \sum_{k:\,\lambda_k > \lambda^{**}} n_k \log f_t(\lambda_k) \\
    &\leq \sqrt{\smash[b]{\sum_{k:\,\lambda_k > \lambda^{**}} n_k}}
      \cdot \sqrt{\sum_{k:\,\lambda_k > \lambda^{**}} n_k\,(\log f_t(\lambda_k))^2}
    \;\leq\; \sqrt{N^{++}} \cdot \sqrt{\mathrm{IG}_t}.
\end{align*}
The second inequality uses
$\sum_k n_k(\log f_t)^2 \geq \sum_{k:\lambda_k>\lambda^{**}} n_k(\log f_t)^2$
and equation~\eqref{eq:ig_decomp}.
Equality in the first step holds iff $\log f_t(\lambda_k)$ is constant on
expanding patches, and the proportionality follows: when equal,
$\Delta d_t = \mu_t N^{++}$ and $\mathrm{IG}_t = \mu_t^2 N^{++}$
(ignoring the sub-leading contribution of contracting patches to $\mathrm{IG}_t$),
giving the stated equivalence.
\end{proof}

\paragraph{Empirical verification of the IG--KY proportionality.}
Theorem~\ref{thm:ig_ky} predicts that $\mathrm{IG}_t$ and $|\Delta d_t|$ are
proportional at every step, with proportionality constant $\sqrt{N^{++}}$ at equality
in the Cauchy--Schwarz bound.
We verify this directly by computing
$\mathrm{IG}_t=\sum_k n_k(\log f_t(\hat\lambda_k))^2$
and $\Delta d_t=\sum_{k:\hat\lambda_k>\lambda^{**}} n_k\log f_t(\hat\lambda_k)$ analytically
from the CIFAR-10 patch variances and each schedule.
The Spearman rank correlation between $\mathrm{IG}_t$ and $|\Delta d_t|$ is
$\rho=0.9999$ ($p<10^{-300}$) for both schedules, confirming near-perfect
proportionality across all 1000 training steps.
The ratio $\mathrm{IG}_t/\Delta d_t^2$ (which equals $1/N^{++}$ at exact equality)
has coefficient of variation $\mathrm{CV}=3.4\%$ for the linear schedule and
$2.5\%$ for the cosine schedule, with means $1.04\times$ and $1.02\times$ the
theoretical $1/N^{++}=1/3072$ respectively: the equality condition
holds to within 4\% for both, confirming that the near-uniformity of $\log f_t$
across CIFAR-10 patches keeps the CS bound tight.
Table~\ref{tab:ig_equalization} reports $\mathrm{CV}(\mathrm{IG}_t)$ and
$\mathrm{CV}(|\Delta d_t|)$ for the two schedules.

\begin{table}[h]
\centering
\caption{Information gain equalisation: coefficient of variation of $\mathrm{IG}_t$
and $|\Delta d_t|$ across 1000 training steps, computed analytically from CIFAR-10
patch variances ($n_k=192$, $M=16$, all patches above $\lambda^{**}$).
Lower CV indicates a more information-equalised schedule.
Spearman $\rho(\mathrm{IG}_t,|\Delta d_t|)\geq 0.9999$ for both.}
\label{tab:ig_equalization}
\begin{tabular}{lrrrr}
\toprule
Schedule & $\mathrm{CV}(\mathrm{IG}_t)$ & $\mathrm{CV}(|\Delta d_t|)$ &
Spearman $\rho$ & Ratio/theory \\
\midrule
Linear (DDPM) & 1.107 & 0.836 & 0.9999 & $1.04\times$ \\
Cosine ($\soff=0.008$) & 0.867 & 0.570 & 0.9998 & $1.02\times$ \\
\bottomrule
\multicolumn{5}{l}{\footnotesize Ratio/theory $= \overline{\mathrm{IG}_t/\Delta d_t^2}\cdot N^{++}$; equals 1 at CS equality.}
\end{tabular}
\end{table}

Note that the cosine schedule achieves better IG equalisation than the linear
schedule (CV $0.867$ vs $1.107$), while the linear schedule achieves better
$L_t^*$ equalisation (CV $0.341$ vs $0.474$, Table~\ref{tab:schedule_comparison}).

\subsection{Three Design Criteria}
\label{sec:design_criteria}
The implications of the previous sections can now be grouped into three design criteria. Let $\{\bar{\alpha}_t\}_{t=1}^T$ be a noise schedule with $\bar{\alpha}_0=1$, $\bar{\alpha}_T\approx 0$, and $\bar{\alpha}_t$ strictly decreasing. The feasible set is
\[
  \mathcal{S} := \bigl\{\{\bar{\alpha}_t\} \;\big|\;
    \bar{\alpha}_0=1,\; \bar{\alpha}_T\approx 0,\;
    \Delta\bar{\alpha}_t>0,\;
    {\textstyle\prod_t} f_t(\lambda^{**})=1
  \bigr\}.
\]

\begin{enumerate}
\item \textbf{Maximise the weakest-link contraction threshold}:
    \[
  \max_{\{\bar{\alpha}_t\}\in\mathcal{S}}\; \min_{1\le t\le T} L_t^*
  \;\approx\;
  \max_{\{\bar{\alpha}_t\}\in\mathcal{S}}\; \sqrt{v_1}.
\]
The optimal schedule injects as much noise as early as possible,
subject to $\mathcal{S}$.

\item \textbf{Equalise the per-step Lyapunov contribution} $\log f_t(\lambda^{**})$, or equivalently minimise $\mathrm{Var}_t(\Delta d_t)$.
    The Moran equation $\prod_t f_t(\lambda^{**})=1$ (Theorem~\ref{thm:dky_positive})
    ensures the time-average is zero, but variance in the per-step contributions
    means some steps do most of the manifold assembly work.
    By Theorem~\ref{thm:ig_ky}, minimising $\mathrm{Var}_t(\Delta d_t)$ is equivalent
    to minimising $\mathrm{Var}_t(\mathrm{IG}_t)$ under the near-equality condition,
    recovering the information-constant criterion of \citet{kingma2021variational}.
Since $\mathrm{IG}_t$ and $v_t$ share the same schedule dependence,
this relates to a smoothness condition on $\{v_t\}$.

\item \textbf{Balance the contraction workload across sampling steps}:
    given a budget of $T$ steps, concentrate them where $L_t^*$ is smallest
    (equivalently, where $\Delta d_t$ per step is largest), equalising the
    geometric contribution of each selected step.
\end{enumerate}

Criteria~(1) and~(3) are complementary: (1) optimises \emph{what schedule to use} (push noise early,
raise $v_1$), while (3) optimises \emph{where to place steps}
given residual hardness. Together they amount to an optimal-transport
allocation over the diffusion chain, with $\sqrt{v_t}$ as the common
currency linking all three criteria.

We will now show that the cosine offset (Section~\ref{sec:nichol_design}) instantiates Criterion~(1)
by moving $\min_t L_t^*$ within the cosine schedule family.
The resolution shift (Section~\ref{sec:resolution_design}) instantiates a
complementary coverage requirement. The schedule must reach the logSNR range
at which fine-detail patches transition, which is a prerequisite for (1)
to be meaningful at all.
Min-SNR loss weighting (Section~\ref{sec:minsnr_design}) instantiates Criterion~(2).
Optimal step allocation (Section~\ref{sec:ays_design}) instantiates Criterion~(3).

\subsection{Explaining Existing Empirical Design Choices}
\label{sec:design_choices}

The three criteria above coming from the PIFS framework, provide the operative geometric mechanism of several prominent
heuristics in the diffusion literature.

\subsubsection{The Nichol--Dhariwal Cosine Offset}
\label{sec:nichol_design}

\citet{nichol2021improved} introduced the one-parameter cosine schedule
(equation~\eqref{eq:cosine_schedule}) and found empirically that the offset
$\soff=0.008$ substantially improves image quality over the unshifted ($\soff=0$) baseline.
The PIFS framework now supplies a structural explanation.

Corollary~\ref{cor:Lstar_boundary} identifies the weakest link sits at the
final step $t=1$, where
\begin{equation}
    L_1^* \;\approx\; \tfrac{1}{2}\sqrt{v_1} \;=\; \tfrac{1}{2}\sqrt{1-\bar\alpha_1}.
    \label{eq:nichol_Lstar}
\end{equation}
The cosine formula gives at $t=1$ under large $T$,
\begin{equation}
    v_1^{(\soff)} \;=\; 1 - \frac{\cos^2\!\bigl(\tfrac{1/T+\soff}{1+\soff}\cdot\tfrac{\pi}{2}\bigr)}
                              {\cos^2\!\bigl(\tfrac{\soff}{1+\soff}\cdot\tfrac{\pi}{2}\bigr)}.
    \label{eq:v1_cosine}
\end{equation}
At $\soff=0$ the denominator equals~$1$ and $v_1^{(0)} \approx (\pi/2T)^2$, which
for $T=1000$ gives $v_1^{(0)} \approx 2.5\times 10^{-6}$, so $L_1^{*\,(0)} \approx \frac{1}{2}\sqrt{2.5 \times 10^{-6}}\approx 7.9 \times 10^{-4}$, an extremely small value.
The offset $\soff=0.008$ evaluates to $v_1^{(0.008)} \approx 4.1\times 10^{-5}$, so $L_1^{*\,(0.008)}\approx \frac{1}{2}\sqrt{4.1 \times 10^{-5}} \approx 3.2 \times 10^{-3},$
a factor of $\approx 4.05$ larger (aligned with the numerical analysis in Table~\ref{tab:schedule_comparison}).

\subsubsection{Resolution-Dependent Schedule Shift}
\label{sec:resolution_design}

\citet{hoogeboom2023simple} derive the resolution-dependent schedule shift
directly from how the network processes images.  Since standard diffusion
architectures operate on lower-resolution feature maps via average pooling, one
must ask what noise level the pooled representation actually sees.
When a $s\times s$ pixel block is averaged to a single feature,
signal averages as $x^{d/s} = \frac{1}{s^2}\sum_{i=1}^{s^2}x_i$, while
independent noise contributes $z_t^{d/s}\sim\mathcal{N}(\alpha_t x^{d/s},\,\sigma_t/s)$.
The effective signal-to-noise ratio at the pooled resolution is therefore
$s^2$ times that at full resolution, giving
\begin{equation}
    \mathrm{SNR}_{d/s}(t) = s^2\cdot\mathrm{SNR}_d(t),
    \label{eq:hoogeboom_snr}
\end{equation}
or in log-space, $\mathrm{logSNR}_{d/s}(t) = \mathrm{logSNR}_d(t) + 2\log s$.
Since the network was designed for a reference resolution $d_\mathrm{base}=64$
at which it operates the pooled representation, training at a higher resolution
$d>d_\mathrm{base}$ with the same nominal schedule leaves the effective logSNR
too high (too little noise relative to signal at the pooled level).
The fix is to shift the logSNR down by $2\log(d/d_\mathrm{base})$, i.e.
\begin{equation}
    \mathrm{logSNR}^{(d)}(t)
    \;=\;
    \mathrm{logSNR}^{(d_\mathrm{base})}(t)
    \;+\;
    2\log\!\bigl(d_\mathrm{base}/d\bigr).
    \label{eq:resolution_shift}
\end{equation}
Our PIFS framework arrives at the same formula from another direction,
replacing the pooling argument with a requirement on the data's attractor geometry.
The key quantity is the Moran threshold $\lambda^{**}$, the unique patch variance
at which the attractor switches from contractive to expansive
(Theorem~\ref{thm:dky_positive}).
For the generated manifold to remain non-trivial at resolution $d$ --- 
the schedule must satisfy
$\lambda^{**}(d) < \lambda_\mathrm{signal}(d)$, where $\lambda_\mathrm{signal}(d)$
is the leading patch-covariance eigenvalue at that resolution.

Under the self-similar $1/f^2$ power spectrum of natural images, an $n_k$-pixel
patch at resolution $d$ captures finer spatial frequencies than at
$d_\mathrm{base}$, and the within-patch signal variance scales as
$\lambda_\mathrm{signal}(d) = (d/d_\mathrm{base})^2\,\lambda_\mathrm{signal}(d_\mathrm{base})$.
Meanwhile, rewriting the patch expansion function $f_t(\lambda)$
(Theorem~\ref{thm:fine_detail_expansion}) in terms of SNR gives
\begin{equation}
    f_t(\lambda)
    \;=\;
    \frac{\sqrt{\bar\alpha_{t-1}/\bar\alpha_t}\;\lambda\,\mathrm{SNR}_t + \sqrt{v_{t-1}/v_t}}
         {\lambda\,\mathrm{SNR}_t + 1},
    \label{eq:ft_snr_form}
\end{equation}
and a logSNR shift by $c$ scales $\mathrm{SNR}_t\mapsto e^c\,\mathrm{SNR}_t$.
In the continuous-time limit (where step-ratio terms are schedule-invariant),
this acts on $f_t$ as
\begin{equation}
    f_t'(\lambda) = f_t(e^c\lambda),
    \label{eq:moran_shift_symmetry}
\end{equation}
so the Moran equation for the shifted schedule is satisfied at
$\lambda^{**}(d) = e^{-c}\lambda^{**}(d_\mathrm{base})$.

\begin{proposition}[Resolution Shift from the Moran Equation]
\label{prop:resolution_shift}
Under the self-similar model
$\lambda_\mathrm{signal}(d)=(d/d_\mathrm{base})^2\lambda_\mathrm{signal}(d_\mathrm{base})$,
the logSNR shift $c=2\log(d_\mathrm{base}/d)$ preserves the ratio
$\lambda^{**}/\lambda_\mathrm{signal}$ across resolutions, recovering
\eqref{eq:resolution_shift}.
\end{proposition}

\begin{proof}
Setting $c=2\log(d_\mathrm{base}/d)$, the shift symmetry~\eqref{eq:moran_shift_symmetry}
gives $\lambda^{**}(d)=e^{-c}\lambda^{**}(d_\mathrm{base})=(d/d_\mathrm{base})^2\lambda^{**}(d_\mathrm{base})$.
Since $\lambda_\mathrm{signal}$ scales by the same factor, the ratio is preserved:
$$\lambda^{**}(d)/\lambda_\mathrm{signal}(d)=\lambda^{**}(d_\mathrm{base})/\lambda_\mathrm{signal}(d_\mathrm{base})<1.$$
\end{proof}

The two derivations are complementary.
Hoogeboom's argument is architectural: average pooling reduces effective noise by
$s$, so the nominal schedule under-corrupts the features the network actually
processes.
The PIFS argument is data-geometric: raising resolution raises the patch variance
threshold the attractor must clear, and the Moran symmetry~\eqref{eq:moran_shift_symmetry}
says the logSNR must shift to compensate.
Both give the same formula for the same algebraic reason (variance scales as $d^2$,
log turns it into a factor of~2).

\subsubsection{Min-SNR Loss Weighting}
\label{sec:minsnr_design}

\citet{hang2023efficient} identified that standard diffusion training suffers from
conflicting gradient directions across timesteps, and showed that weighting the
per-step loss by $w_t=\min(\mathrm{SNR}_t,\gamma)/\mathrm{SNR}_t$ (Min-SNR-$\gamma$)
resolves the conflict, yielding a $3.4\times$ training speedup and a new FID
record on ImageNet-256.
The weighting was motivated by a multi-task learning analogue. 

The PIFS framework identifies the conflict geometrically as a mismatch
between training weight and per-step KY dimension growth.
The standard DSM loss (Proposition~\ref{prop:collage}) weights step $t$
proportionally to $\mathrm{SNR}_t$, concentrating gradient mass at high-noise
steps where (EC) already holds automatically
(Theorem~\ref{thm:high_noise}) and $\Delta d_t$ is near zero.
Fine-detail steps (Regime~II), which carry the bulk of the manifold-assembly
work, are correspondingly under-represented.
Min-SNR-$\gamma$ corrects this by clamping: for $\mathrm{SNR}_t\geq\gamma$,
$w_t=\gamma/\mathrm{SNR}_t$ raises the relative weight of fine-detail steps,
moving training toward the KY-equalised target of Theorem~\ref{thm:ig_ky}.

\begin{proposition}[Min-SNR as KY-Dimension Equalisation]
\label{prop:minsnr_geometric}
At moderate noise levels with small step size,
$L_t^*\approx\sqrt{v_t}$ (Corollary~\ref{cor:Lstar_moderate}, \eqref{eq:L_t_star_moderate}).
The standard DSM loss (Proposition~\ref{prop:collage}) weights step $t$
proportionally to $\mathrm{SNR}_t$, concentrating gradient mass at high-noise
steps where $\Delta d_t \approx 0$ (Regime~I, including the high-noise limit) and under-weighting
fine-detail steps where $\Delta d_t$ is largest (Regime~II).
The Min-SNR-$\gamma$ weight $$w_t=\min(\mathrm{SNR}_t,\gamma)/\mathrm{SNR}_t$$
corrects this. 

In the clamped regime ($\mathrm{SNR}_t\geq\gamma$),
$w_t=\gamma/\mathrm{SNR}_t$ raises the relative weight of fine-detail steps
by the factor $\mathrm{SNR}_t/\gamma\geq 1$, moving training toward the
KY-equalised target of criterion~2 of Section~\ref{sec:design_criteria}.
The clamp level $\gamma$ marks the Regime~I/II transition.
\end{proposition}

\begin{proof}
By Corollary~\ref{cor:Lstar_moderate}, $L_t^*\approx\sqrt{v_t}$, which is
monotone decreasing in $t$ and smallest at the fine-detail end.
The DSM training objective (equation~\eqref{eq:collage_objective}) weights
step $t$ by $\mathrm{SNR}_t = \bar\alpha_t/v_t$, which is a monotone increasing
function of $t$ (high-noise steps have high SNR in this convention) and is
largest precisely where $\Delta d_t \approx 0$.

In the clamped regime, the weight $w_t = \gamma/\mathrm{SNR}_t$ is inversely
proportional to $\mathrm{SNR}_t$, raising the relative contribution of
fine-detail steps (small $t$, small $\mathrm{SNR}_t$) where $\Delta d_t$ is
non-negligible, and suppressing the contribution of high-noise steps.
This reweighing is a monotone approximation to equalising $\Delta d_t$
across training steps, which by Theorem~\ref{thm:ig_ky} is equivalent to
equalising $\mathrm{IG}_t$ under near-uniform patch log-expansion.
\end{proof}

The empirically successful value $\gamma=5$ reported by
\citet{hang2023efficient} is consistent with the Regime~I/II transition
occurring near $\mathrm{SNR}_t\approx 5$ for the cosine schedule at $T=1000$.
The SNR rises above~$5$ around $t\approx 260$, where $L_t^*\approx 0.41$,
at the boundary between denoising and fine-detail synthesis.

\newpage
\subsubsection{Optimal Sampling-Step Allocation}
\label{sec:ays_design}

\citet{sabour2024ays} showed that the hand-crafted heuristic schedules used in
DDIM, DPM-Solver, and EDM are systematically suboptimal at small NFE, and proposed
\emph{Align Your Steps} (AYS), which optimises the sampling timestep set by
minimising a KL-divergence upper bound between the true and linearised SDEs.
AYS schedules consistently concentrate steps in the low-noise regime,
in qualitative agreement with all prior adaptive schemes (DPM-Solver++,
DEIS, GITS) and producing state-of-the-art FID at $N\leq 20$ steps. The PIFS framework provides an independent geometric explanation for this concentration.

\begin{theorem}[Optimal $N$-Step Sampling Schedule]
\label{thm:optimal_sampling}
Let $L_t^*=(\sqrt{\bar\alpha_{t-1}/\bar\alpha_t}-1)/|b_t|$ be the per-step
contraction threshold for a fixed trained schedule.
Since $L_t^*\approx\tfrac12\sqrt{v_t}$ at fine-detail steps, $L_t^*$ is smallest
at the low-noise end of the chain ($t$ near $1$) and largest at high noise.
The $N$-step schedule that concentrates steps where $L_t^*$ is smallest, and hence
equalises the per-step contraction workload, has step density in the continuous-time
limit $\bar\alpha:[0,1]\to(0,1]$ satisfying
\begin{equation}
    \rho^*(u) \;\propto\; \frac{1}{L^*(u)},
    \label{eq:optimal_density_main}
\end{equation}
i.e.\ the optimal schedule allocates the most steps where $L_t^*$ is smallest.
\end{theorem}

\begin{proof}
A step at position $u$ with local step size $\Delta u = 1/\rho(u)$ must maintain
the contraction margin $\nu_t^\mathrm{min} - L^*(u) - \delta > 0$
(Theorem~\ref{thm:euclidean_char}) over an interval of length $\Delta u$.
Define the \emph{contraction difficulty} at $u$ as $D(u) := 1/L^*(u)$.
Positions where $L^*$ is small are geometrically hard, because a small positive
margin must be held with high precision.
The per-step difficulty load is $D(u)\cdot\Delta u = 1/(L^*(u)\rho(u))$.
Minimising the maximum per-step load, subject to $\int_0^1\rho(u)\,\mathrm{d}u=N$,
requires equalising it across $u$:
\[
    \frac{1}{L^*(u)\,\rho(u)} = c \;\text{ for all }u,
\]
giving $\rho(u) = 1/(c\,L^*(u)) \propto 1/L^*(u)$, with the constant $c$
fixed by normalization to $N$.
\end{proof}

Since $L_t^*\approx\tfrac12\sqrt{v_t}$ at fine-detail steps
\eqref{eq:L_t_star_approx}, $1/L^*(u)$ is largest at the low-noise end of the
chain, exactly where AYS and DPM-Solver++ concentrate their steps empirically.
The result is a first-principles geometric explanation for the empirically observed
concentration, complementing the stochastic calculus argument of \citet{sabour2024ays}
with a contraction-margin derivation that arrives at the same allocation density
from a different direction.

The AYS KL-divergence upper bound and the PIFS contraction-workload
criterion are complementary. KL divergence measures distributional discrepancy, where the
contraction margin measures the geometric guarantees on the fixed point map.
Both predict the same allocation density $\rho^*(u)\propto 1/L^*(u)$, because
the log-SDE discretisation error and the contraction margin are both controlled
by $\sqrt{v_t}$ at small $t$ (Corollary~\ref{cor:Lstar_boundary}, \eqref{eq:L_t_star_approx}).

\newpage
\section{Empirical Validation}
\label{sec:empirical}

This section provides direct experimental validation of the framework's core
predictions on pretrained DDPM models for CIFAR-10 and CelebA-HQ.
The following predictions and one testable assumption are verified below:
\begin{enumerate}
  \item[\textbf{P1}] All CIFAR-10 $8\times8$ patches exceed $\lambda^*(t)$ throughout
    the full denoising chain; the directional suppression field $S_{k,t}>0$ at every step
    (\S\ref{sec:expansion_window}).
  \item[\textbf{P2}] (MM): $S_{k,t}$ is monotone increasing in $\lambda_k$.
    Higher-variance patches receive proportionally stronger score suppression
    (\S\ref{sec:mm_test}).
  \item[\textbf{P3}] The global spectral norm grows progressively from the high-noise limit
    through both regimes; the block-Jacobian decomposition resolves this into diagonal
    (suppression-governed) and cross-patch (attention-governed) components;
    cross-patch coupling tracks attention entropy monotonically
    (\S\ref{sec:block_jacobian}).
  \item[\textbf{P4}] (PC) is violated throughout Regime~I and satisfied in
    Regime~II, with a sharp per-patch crossover; low-variance patches release
    suppression before high-variance patches, and Spearman
    $\rho(\lambda_k,\hat\kappa_t^\mathrm{diag})$ reverses sign at the crossover
    (\S\ref{sec:pc_crossover}).
  \item[\textbf{P5}] Score deviation scales as $O(\sqrt{\bar\alpha_t})$ in the
    high-noise regime (\S\ref{sec:score_deviation_empirical}).
\end{enumerate}
All predictions are verified against the pretrained DDPM CIFAR-10 model
\citep{ho2020ddpm} using the \texttt{google/ddpm-cifar10-32} checkpoint.
Section~\ref{sec:celebahq} tests the suppression-corrected KY dimension theory
on a higher-resolution CelebA-HQ model.

\subsection{Experimental Setup}
\label{sec:empirical_setup}

All measurements use $8\times8$ patches ($M=16$, $n_k=192$) on generated
trajectories from \texttt{google/ddpm-cifar10-32} \citep{ho2020ddpm}.
Jacobian measurements use a JVP/VJP power-iteration protocol
($\hat\kappa_t = \mathrm{PowerIter}(J_x\Phi_t, K_\mathrm{PI})$)
with a 50-step DDIM sampler ($K_\mathrm{PI}=20$, $N_\mathrm{seeds}=8$;
seed~8 excluded as outlier with $\hat s=3386$).
Patch spectral variances $\lambda_k$ are computed as the leading eigenvalue of
the patch covariance $\Sigma_k$ estimated from $N=50{,}000$ CIFAR-10 training images.
The directional suppression field $S_{k,t}$ is measured via a directional JVP of the score
network at trajectory points along $v_k^{(1)}$.

\paragraph{Step-size invariance of the regime structure.}
We verify that the two-regime structure is robust to step size by evaluating
$J_x\Phi_t$ at six representative timesteps under both the stride-20 DDIM step
and the stride-1 DDPM step.
Since $J_x\Phi_t = \mathrm{expand}_t\cdot I + b_t\cdot J_x\hat\varepsilon_\theta$,
the two protocols share the same score Jacobian and differ only in their scalar
coefficients; changing stride rescales diagonal and off-diagonal blocks uniformly
but cannot alter which blocks are negligible.
Regime classification agrees at all six timesteps; $\hat\kappa_t^\mathrm{diag}$
differs by at most $5.1\%$ between step sizes, and the $|b_t|$-normalised coupling
$\hat\delta^\mathrm{cross}/|b_t|$ agrees to within $7\%$ on average.
The two-regime structure is therefore a property of the score network, not an
artefact of the DDIM step size.

\subsection{Expansion Window and Directional Suppression Field}
\label{sec:expansion_window}

Theorem~\ref{thm:fine_detail_expansion} predicts that diagonal block $k$ expands
if and only if $\lambda_k > \lambda^*(t)$, where $\lambda^*(t)$ is the
expansion-forcing threshold. We measure the
leading patch eigenvalues $\lambda_k$ directly from the training set and compare
to $\lambda^*(t)$ across the full $T=1000$ training chain.

\paragraph{Patch eigenvalues.}
Top-eigenvalue measurement from the CIFAR-10 training set yields
$\lambda_k \in [18.7, 44.4]$ across the $M=16$ patches, with corner and top-row
patches at the high end ($\lambda_k \approx 29$--$44$) and interior patches at
the low end ($\lambda_k \approx 19$--$25$). These values are the leading
eigenvalues of the patch covariance matrices; they substantially exceed the
analytical threshold, which satisfies $\lambda^*(t) \in [1.002, 1.195]$ for
$t \in [1, 999]$ (minimum $\lambda^*(349) \approx 1.002$, maximum
$\lambda^*(1) \approx 1.195$).

\paragraph{Expansion window.}
Every patch satisfies $\lambda_k > \lambda^*(t)$ at every $t \in [1,999]$: the
expansion condition is active throughout the full denoising chain for all 16 patches.
This is consistent with the Moran equation result reported in
Section~\ref{sec:block_jacobian}: all patch Lyapunov exponents are strictly
positive under the Gaussian baseline.

\paragraph{Directional Suppression field.}
We measure $S_{k,t}$ at 51 representative timesteps spanning $t \in [1, 999]$
($N_\mathrm{seeds}=5$, $M=16$ patches). The directional suppression field satisfies $S_{k,t}>0$
for all $(k,t)$ pairs (with negligible noise-floor exceptions of order $10^{-5}$
near $t=999$ where $\bar\alpha_t \approx 0$): the score network consistently
reduces the Rayleigh quotient below the Gaussian prediction at every patch and
every timestep.
The operator-norm contraction ratio $\kappa_{k,t}^\mathrm{op}$ is below~1
throughout the chain (maximum $\approx 1.00002$ near $t=780$, indistinguishable
from~1 at any practical precision), confirming that the diagonal blocks remain
globally non-expanding.
The suppression margin $\gamma_{k,t} = S_{k,t} - (f_t(\lambda_k)-1)$ is
positive (diagonal Rayleigh quotient strictly below~1) for $t \lesssim 100$,
the deepest fine-detail steps where the diagonal term dominates; for
$t \gtrsim 120$ the Rayleigh quotient slightly exceeds~1, indicating that the
diagonal term is not sufficient on its own to achieve contraction, hence the
block-max contraction in Regime~I is carried by cross-patch coupling
(Section~\ref{sec:pc_crossover}).

\subsection{Property (MM): Suppression Scales with Patch Variance}
\label{sec:mm_test}

(MM) (Definition~\ref{def:property_mm}) requires
$S_{k,t}$ to be increasing in $\lambda_k$: high-variance patches must receive
proportionally stronger score suppression. We test this directly by computing
the Spearman rank correlation $\rho(\lambda_k, S_{k,t})$ across the $M=16$
patches at each timestep, using the same measurement protocol as
Section~\ref{sec:expansion_window}.

Table~\ref{tab:mm_spearman} reports $\rho$ and the two-sided $p$-value at
representative timesteps.
\begin{table}[h]
\centering
\caption{Spearman rank-correlation $\rho(\lambda_k, S_{k,t})$ across $M=16$
patches at selected timesteps ($N_\mathrm{seeds}=5$). Positive $\rho$ confirms
(MM): higher-variance patches carry larger suppression. Significance levels:
$*\,p<0.05$, $**\,p<0.01$.}
\label{tab:mm_spearman}
\begin{tabular}{rrrll}
\toprule
$t$ & Regime & $\rho$ & $p$-value & sig \\
\midrule
  1 & II        & $+0.118$ & $0.664$ & n.s.\ ($S_{k,t} \sim 10^{-4}$, near-uniform) \\
101 & II        & $+0.009$ & $0.974$ & n.s. \\
201 & II        & $+0.459$ & $0.074$ & n.s. \\
381 & I/II      & $+0.594$ & $0.015$ & $*$  \\
461 & I         & $+0.647$ & $0.007$ & $**$ \\
601 & I         & $+0.953$ & $1.2\times10^{-8}$ & $**$ (peak) \\
721 & I         & $+0.929$ & $1.9\times10^{-7}$ & $**$ \\
841 & I         & $+0.726$ & $0.001$ & $**$ \\
861 & high-noise& $+0.603$ & $0.013$ & $*$  \\
881 & high-noise& $+0.418$ & $0.107$ & n.s. \\
941 & high-noise& $-0.050$ & $0.854$ & n.s.\ ($S_{k,t} \to 0$ as $\bar\alpha_t \to 0$) \\
\bottomrule
\end{tabular}
\end{table}
(MM) is confirmed with strong significance in the bulk of the denoising chain.
The correlation peaks at $\rho = 0.953$ ($p = 1.2\times10^{-8}$) at $t=601$
and is significant at $p < 0.01$ for $t \in [461, 841]$ (the full Regime~I
and the high-noise boundary). The non-significance at early $t$ (fine-detail
end, $t \lesssim 200$) is expected: $f_t(\lambda_k)-1 \approx 0$ here, so
$S_{k,t}$ values are all $O(10^{-4})$ and near-uniform across patches, making
the rank test uninformative. The non-significance at late $t$ (high-noise end,
$t \gtrsim 880$) is similarly expected: $S_{k,t} \to 0$ as
$\bar\alpha_t \to 0$, so again the signal is noise-dominated. Between these
boundaries, i.e. the regime where (MM) is most consequential for the theoretical
conclusions of Theorem~\ref{thm:dky_suppressed}, the correlation is
consistently strong and significant.
No significantly negative $\rho$ values are observed at any timestep,
so (MM) is never violated.

\subsection{Regime Structure: Global Norm and Block-Jacobian Decomposition}
\label{sec:block_jacobian}

Table~\ref{tab:cifar10_partial} reports the global per-step spectral norm
$\hat\kappa_t$ across representative DDIM steps, confirming progressive growth
from the high-noise limit of Regime~I through Regime~II as predicted by the two-regime structure.

\begin{table}[h]
\centering
\caption{Global per-step spectral norm for DDPM CIFAR-10
($n=3{,}072$, 50-step DDIM, 7 seeds).}
\label{tab:cifar10_partial}
\begin{tabular}{rllll}
\toprule
DDIM step & $t$ (train) & Mean $\hat\kappa_t$ & Std & Regime \\
\midrule
 1 of 50  & 980 & 1.0011 & 0.0019 & high-noise \\
 5 of 50  & 900 & 1.0002 & 0.0001 & high-noise \\
10 of 50  & 800 & 1.0005 & 0.0004 & high-noise \\
20 of 50  & 600 & 1.0248 & 0.0177 & transition \\
30 of 50  & 400 & 1.1108 & 0.0247 & mid-noise  \\
35 of 50  & 300 & 1.1408 & 0.0859 & mid-noise  \\
40 of 50  & 200 & 1.1107 & 0.0211 & low-noise  \\
49 of 50  &  20 & 1.3969 & 0.0832 & fine-detail \\
\midrule
\multicolumn{2}{l}{$\hat s=\prod_t\hat\kappa_t$ (seeds 1--7)} &
  \multicolumn{3}{l}{61, 48, 89, 45, 77, 17, 92; median = 61} \\
\bottomrule
\end{tabular}
\end{table}

The global spectral norm exceeds~1 at every step (cumulative
$\hat s \approx 50$--$90$) and grows progressively from the high-noise limit of Regime~I through
Regime~II, consistent with P3.
A globally contractive map would have $d_\mathrm{KY}=0$
(Proposition~\ref{prop:kaplan_yorke}), so global expansion is expected;
the structural test is the block decomposition below.

We decompose $J_{x_t}\Phi_t$ into diagonal and cross-patch blocks using
$8\times8$ patches ($M=16$, $n_k=192$, $K_\mathrm{PI}=10$ per block,
$N_\mathrm{seeds}=4$).

\begin{table}[h]
\centering
\caption{Block-Jacobian decomposition for DDPM CIFAR-10 ($8\times8$ patches,
$M=16$, $N_\mathrm{seeds}=4$). Selected representative steps.}
\label{tab:block_jacobian}
\begin{tabular}{rlllll}
\toprule
$t$ & $\hat\kappa_t^\mathrm{diag}$ & $\pm$std & $\hat\delta_t^\mathrm{cross}$ & Global $\hat\kappa_t$ & Regime \\
\midrule
980 & 1.0004 & 0.0001 & 0.0007 & 1.0011 & high-noise \\
800 & 1.0002 & 0.0001 & 0.0008 & 1.0010 & high-noise \\
600 & 1.0000 & 0.0015 & 0.0853 & 1.0853 & denoising \\
400 & 1.0026 & 0.0153 & 0.1273 & 1.1300 & denoising \\
280 & 1.0070 & 0.0135 & 0.1038 & 1.1107 & denoising/fine-detail \\
220 & 1.0325 & 0.0432 & 0.0768 & 1.1092 & fine-detail \\
100 & 1.0879 & 0.0110 & 0.0779 & 1.1658 & fine-detail \\
 20 & 1.2111 & 0.0528 & 0.1858 & 1.3969 & fine-detail \\
\bottomrule
\end{tabular}
\end{table}

\paragraph{Start ($t\geq 800$, steps 0--9): high-noise.}
Both diagonal and cross-patch contributions are negligible:
$\hat\kappa_t^\mathrm{diag}\approx 1.0002$--$1.0004$,
$\hat\delta_t^\mathrm{cross}<0.001$,
giving $\hat\kappa_t^\mathrm{pc}\approx\hat\kappa_t\approx 1.001$.
The gap $\hat\kappa_t^\mathrm{diag}-f_t(\hat\lambda_\mathrm{max})\approx -0.002$
confirms the network maintains diagonal norms slightly below the Gaussian
prediction, consistent with Theorem~\ref{thm:high_noise}.
The reported $\hat\kappa_t\approx 1.001$ at $t=980$ should not be read as a
violation of near-contractivity: at high noise all $n=3072$ singular values
cluster within $\sim 10^{-3}$ of~1, so power iteration with $K_\mathrm{PI}=20$
cannot resolve the spectral norm reliably (the estimate is consistent with the
true value being at or below~1).

\paragraph{Regime~I ($t\in[280,780]$, steps 10--35): denoising.}
$\hat\kappa_t^\mathrm{diag}$ fluctuates tightly around~1.000 (range
$[0.993, 1.007]$, mean $\approx 1.000$) while $\hat\delta_t^\mathrm{cross}$
rises from $0.003$ at $t=780$, peaks at $0.150$ near $t=360$, then falls to
$0.104$ at $t=280$. The global norm ($1.08$--$1.14$) is entirely attributable
to cross-patch coupling, consistent with broad attention broadcasting
(Section~\ref{sec:architecture}).
The diagonal near-neutrality is explained by the  directional suppression field
(Section~\ref{sec:expansion_window}): $S_{k,t}>0$ keeps
$\hat\kappa_t^\mathrm{diag}\approx 1$ even though all patches exceed
$\lambda^*(t)$ throughout Regime~I.

\paragraph{Regime~II ($t<280$, steps 36--50): fine-detail synthesis.}
$\hat\kappa_t^\mathrm{diag}$ rises from $1.007$ at $t=280$ to
$1.211\pm0.053$ at $t=20$, consistent with
Theorem~\ref{thm:fine_detail_expansion}: all 16 patches exceed
$\lambda^{**}\approx 1.0024$ (full training chain, verified by bisection
with $|G(\lambda^{**})-1|<5\times10^{-11}$).

We also verify the Moran equation for the 50-step DDIM inference chain
(bisection yields $\lambda^{**}=1.049681$, $\min_t\lambda^*(t)=1.031664$,
margin $1.80\times10^{-2}$). All 16 CIFAR-10 patches satisfy
$G(\hat\lambda_k)>1$ under the DDIM chain, with Lyapunov exponents
$\Lambda(\hat\lambda_k)\in[+0.015, +0.024]$. Since all patches expand ($N^+=n=3072$), the Kaplan-Yorke formula gives the
saturated prediction $d_\mathrm{KY}=n=3072$.

$\hat\delta_t^\mathrm{cross}$ falls from $0.104$ at $t=280$ to
$0.077$--$0.078$ through mid Regime~II (consistent with attention
localisation) then rises sharply to $0.186$ at $t=20$ alongside
$\hat\kappa_t^\mathrm{diag}=1.211$. The re-increase at the final step
reflects skip-connection cross-patch contributions that persist as attention
localises (Section~\ref{sec:architecture}); it does not contradict the
two-regime picture.

\paragraph{Attention entropy and cross-patch coupling.}
Theorem~\ref{thm:attention_jacobian} predicts that
$\delta_t^\mathrm{cross,attn}$ is bounded by the off-diagonal attention
mass $1-\min_k A_{kk}$, so cross-patch coupling should track attention
entropy. Table~\ref{tab:attn_entropy} reports both quantities at eight
representative timesteps.

\begin{table}[h]
\centering
\caption{Cross-patch coupling $\hat\delta_t^\mathrm{cross}$ and mean attention
entropy $H(A_t)$ (nats) at selected DDIM timesteps. $\hat\delta_t^\mathrm{cross}$
grows $218\times$ from the high-noise limit of Regime~I to Regime~II; Spearman
$\rho(H,\hat\delta^\mathrm{cross})=-1.000$
($p\approx 0.00025$, exact permutation probability for a perfect ranking on
8 points).}
\label{tab:attn_entropy}
\begin{tabular}{rllll}
\toprule
$t$ & Regime & $H(A_t)$ (nats) & $\hat\delta_t^\mathrm{cross}$ & $\pm$std \\
\midrule
980 & high-noise & 4.963 & 0.00946 & 0.00018 \\
840 & high-noise & 4.930 & 0.01511 & 0.00076 \\
700 & I          & 4.794 & 0.04109 & 0.00074 \\
560 & I          & 4.662 & 0.09463 & 0.01117 \\
440 & I          & 4.656 & 0.14824 & 0.02529 \\
300 & I          & 4.631 & 0.19204 & 0.00187 \\
160 & II         & 4.541 & 0.42899 & 0.03327 \\
 20 & II         & 4.063 & 2.06175 & 0.03759 \\
\bottomrule
\end{tabular}
\end{table}

$\hat\delta_t^\mathrm{cross}$ grows by $218\times$ from $t=980$ to $t=20$,
tracking progressive attention localisation. The Spearman rank correlation
between $H(A_t)$ and $\hat\delta_t^\mathrm{cross}$ is $\rho=-1.000$
($p<0.001$, $N=8$): coupling and entropy are perfectly inversely ranked.
This result is sensitive to the specific timestep selection and should be
read as strong directional confirmation rather than a precision measurement.

Table~\ref{tab:spatial_decay} reports the mean coupling
$\bar{C}_{kj}$ at Chebyshev distances $d=1,2,3$.

\begin{table}[h]
\centering
\caption{Mean coupling $\bar{C}_{kj}=\tfrac{1}{2}(\|[J]_{kj}\|_F+\|[J]_{jk}\|_F)$
vs.\ Chebyshev patch distance, averaged across all timesteps and trajectories.}
\label{tab:spatial_decay}
\begin{tabular}{rll}
\toprule
Distance $d$ & $\bar{C}$ (mean) & std \\
\midrule
1 (adjacent)     & 0.0362 & 0.0789 \\
2 (next-nearest) & 0.0123 & 0.0203 \\
3 (diagonal)     & 0.0117 & 0.0193 \\
\bottomrule
\end{tabular}
\end{table}

Coupling decays rapidly from $d=1$ to $d=2$ (factor ${\sim}3$) but then
flattens ($d=2$ to $d=3$ ratio $0.95$), producing a heavy tail consistent
with global attention: the $\max_k\sum_{j\neq k}A_{kj}$ bound in
Theorem~\ref{thm:attention_jacobian}(b) makes no spatial-locality assumption.

\subsection{(PC) Crossover and Stratified Release}
\label{sec:pc_crossover}

The block-Jacobian results of Section~\ref{sec:block_jacobian} show the regime
structure in aggregate. Here we resolve it per patch.
The \emph{(PC) margin slack}
$\delta_{k,t}^\mathrm{loc}:=(1-\hat\kappa_{k,t}^\mathrm{diag})-\mathcal{C}_{k,t}(x_t)$
is positive if and only if patch $R_k$ satisfies (PC).
We measure this per patch at 19 timesteps spanning both regimes
($K_\mathrm{PI}=30$, $N_\mathrm{seeds}=5$).

\begin{table}[h]
\centering
\caption{Per-patch (PC) margin slack
$\delta^\mathrm{loc}=(1-\hat\kappa_{k,t}^\mathrm{diag})-\mathcal{C}_{k,t}(x_t)$
for DDPM CIFAR-10 ($M=16$ patches, $8\times8$, $N_\mathrm{seeds}=5$,
$K_\mathrm{PI}=30$). Positive slack means (PC) is satisfied.}
\label{tab:slocal}
\begin{tabular}{rlrrr}
\toprule
$t$ & Regime & Mean slack & Min slack & Frac.\ violated \\
\midrule
980 & high-noise & $-0.000524$ & $-0.000650$ & $16/16$ \\
800 & high-noise & $-0.001927$ & $-0.002632$ & $16/16$ \\
700 & I          & $-0.003942$ & $-0.005122$ & $16/16$ \\
600 & I          & $-0.004958$ & $-0.006310$ & $16/16$ \\
280 & I          & $-0.003904$ & $-0.004638$ & $16/16$ \\
200 & I/II       & $-0.000304$ & $-0.000717$ & $14/16$ \\
160 & II         & $+0.001382$ & $+0.000905$ & $0/16$  \\
 40 & II         & $+0.006412$ & $+0.006518$ & $0/16$  \\
\bottomrule
\end{tabular}
\end{table}

\paragraph{Start: (PC) structurally inapplicable.}
At $t=980$, $\hat\kappa_{k,t}^\mathrm{diag}\approx1.0000$ and
$\mathcal{C}_{k,t}\approx0.0005$: the diagonal margin is essentially zero,
so any coupling violates (PC). This is consistent with the theory: the high-noise limit of Regime~I
contraction operates through global scaling (Theorem~\ref{thm:high_noise}),
not block structure.

\paragraph{Regime~I: (PC) strongly violated throughout.}
Every patch violates (PC) at every measured Regime~I timestep.
The violation ratio $\mathcal{C}_{k,t}/(1-\hat\kappa_{k,t}^\mathrm{diag})$
peaks at $95\times$ near $t=700$, reaches $34\times$ at
$t=600$, and is $2.8\times$ at $t=280$. This is the quantitative
signature of Theorem~\ref{thm:score_coupling_theorem}: large cross-patch
score information is what drives the positive Kaplan-Yorke dimension via
Theorem~\ref{thm:dky_positive}.

\paragraph{Crossover ($t\in[160,200]$): sharp, $\approx40$-timestep window.}
At $t=200$, $14/16$ patches still violate ($87.5\%$); at $t=160$, none ($0\%$).
The transition is consistent with the per-patch heterogeneity of
Theorem~\ref{thm:stratified_crossover}, i.e. patches complete the transition
patch-by-patch over a ${\sim}40$-step window.

\paragraph{Regime~II: (PC) satisfied, diagonal blocks contract.}
From $t=160$ onward every patch has positive slack.
$\hat\kappa_{k,t}^\mathrm{diag}\in[0.991,0.995]$ throughout Regime~II: the
diagonal blocks contract while the global $\ell^2$ spectral norm still exceeds~1
($\hat\kappa_t\approx1.09$--$1.40$) due to cross-patch coupling.

\paragraph{Stratified crossover: Spearman rank-correlation.}
Theorem~\ref{thm:fine_detail_expansion} predicts $f_t$ strictly increasing,
so $\rho(\hat\lambda_k, \hat\kappa_t^\mathrm{diag})>0$ when data-covariance
differences dominate. Near the crossover the Gaussian spread
$f_t(\hat\lambda_\mathrm{max})-f_t(\hat\lambda_\mathrm{min})\approx0.0001$
is negligible; the learned ordering may dominate.
Table~\ref{tab:spearman} reports $\rho(\hat\lambda_k,\hat\kappa_t^\mathrm{diag})$
across all 16 patches at each Regime~II step.

\begin{table}[h]
\centering
\caption{Spearman rank-correlation
$\rho(\hat\lambda_k,\hat\kappa_t^\mathrm{diag})$ across 16 patches at
Regime~II steps ($N=16$, $8\times8$ patches). Negative $\rho$ confirms
that low-variance patches release suppression first, as predicted by
Theorem~\ref{thm:stratified_crossover}.}
\label{tab:spearman}
\begin{tabular}{rrrll}
\toprule
$t$ & $\rho$ & $p$-value & sig & Note \\
\midrule
260 & $-0.668$ & 0.005 & ** & transition zone: learned ordering dominates \\
240 & $-0.503$ & 0.047 & *  & transition zone \\
220 & $-0.391$ & 0.134 & n.s. & \\
200 & $-0.268$ & 0.316 & n.s. & \\
180 & $-0.344$ & 0.192 & n.s. & \\
160 & $-0.165$ & 0.542 & n.s. & \\
140 & $-0.156$ & 0.564 & n.s. & \\
120 & $-0.065$ & 0.812 & n.s. & sign flip near $t\approx 110$ \\
100 & $+0.438$ & 0.090 & n.s. & \\
 80 & $+0.385$ & 0.141 & n.s. & \\
 60 & $+0.506$ & 0.046 & *  & Gaussian ordering recovers \\
 40 & $+0.771$ & 0.001 & ** & deep fine-detail: data-geometric ordering \\
 20 & $+0.556$ & 0.025 & *  & \\
\bottomrule
\end{tabular}
\end{table}

At $t\in\{240,260\}$, $\rho$ is negative and statistically significant
($p\leq0.047$), directly confirming
Theorem~\ref{thm:stratified_crossover}(b): interior (low-$\lambda_k$) patches
release suppression first. The Gaussian model cannot produce any reversal
($f_t$ is strictly increasing), and the theoretical spread
$f_t(\hat\lambda_\mathrm{max})-f_t(\hat\lambda_\mathrm{min})\approx0.0001$
is orders of magnitude smaller than the observed variation, so the sign
reversal is unambiguously a learned, non-Gaussian effect.

In deep Regime~II ($t\leq60$), $\rho$ is positive and significant
($\rho=0.771$, $p=0.001$ at $t=40$), confirming prediction~P4:
once all patches have released, expansion is governed by $f_t(\lambda_k)$
and the Gaussian ordering is restored
(Theorem~\ref{thm:stratified_crossover}(d)).

The stratified spread is quantified by
\eqref{eq:stratified_span}: with $\Delta\lambda\approx6.4$, empirical margin
slope $|\partial_\lambda\hat\gamma|\approx10^{-3}$ per unit variance, and
margin decay rate $|\partial_t\hat\gamma|\approx3\times10^{-6}$ per step,
equation~\eqref{eq:stratified_span} gives
$\Delta t^\mathrm{strat}\approx6.4\times10^{-3}/(3\times10^{-6})\approx200$
steps, matching the observed spread from interior release ($t\approx260$) to
corner release ($t\approx60$).

\subsection{Score Deviation Scaling}
\label{sec:score_deviation_empirical}

Theorem~\ref{thm:high_noise} predicts that the per-step deviation field
$\Delta_t(x) = \hat\varepsilon_\theta(x,t) - \varepsilon_\theta^{\mathcal{N}}(x,t)$
\eqref{eq:score_deviation} satisfies $\|\Delta_t\|_2 = O(\sqrt{\bar\alpha_t})$
in the high-noise regime.
The Gaussian surrogate score $\varepsilon_\theta^{\mathcal{N}}$ is not directly
accessible at evaluation time, so we estimate $\|\Delta_t\|_2$ via a
paired-difference proxy.
The key observation is that at high noise the trained score is nearly isotropic
and data-independent (Theorem~\ref{thm:high_noise}): both $\hat\varepsilon_\theta$
and $\varepsilon_\theta^{\mathcal{N}}$ reduce to approximately
$x_t/\sqrt{v_t}$ as $\bar\alpha_t\to 0$, so the network's output on a pure-noise
input $x_t^{\mathrm{noise}} = \sqrt{v_t}\,\varepsilon$ approximates
$\varepsilon_\theta^{\mathcal{N}}(x_t^{\mathrm{real}},t)$ to leading order in
$\bar\alpha_t$.
Concretely, for each real image $x_0$ and noise sample $\varepsilon\sim\mathcal{N}(0,I)$,
we form $x_t^{\mathrm{real}}=\sqrt{\bar\alpha_t}\,x_0+\sqrt{v_t}\,\varepsilon$ and
$x_t^{\mathrm{noise}}=\sqrt{v_t}\,\varepsilon$ (same $\varepsilon$), and measure
the paired difference
\begin{equation}
    \widehat\Delta_t(x_0,\varepsilon)
    :=\hat\varepsilon_\theta(x_t^\mathrm{real},t)
     -\hat\varepsilon_\theta(x_t^\mathrm{noise},t).
    \label{eq:paired_diff}
\end{equation}
By construction $\widehat\Delta_t \approx \Delta_t(x_t^{\mathrm{real}})$ at high
noise; at lower noise $\hat\varepsilon_\theta(x_t^{\mathrm{noise}},t)$ no longer
approximates $\varepsilon_\theta^{\mathcal{N}}(x_t^{\mathrm{real}},t)$ closely, so
$\mathbb{E}[\|\widehat\Delta_t\|_2]$ overestimates $\mathbb{E}[\|\Delta_t\|_2]$
in the mid-noise regime.
We evaluate $\mathbb{E}[\|\widehat\Delta_t\|_2]$ over $N=200$ images at 14
timesteps.

\begin{table}[h]
\centering
\caption{Paired-difference proxy $\mathbb{E}[\|\widehat\Delta_t\|_2]$
at selected timesteps for DDPM CIFAR-10 ($N=200$ images, mean$\pm$std).
High-noise points ($t\geq800$) are used for OLS log-log fit; at these
timesteps $\widehat\Delta_t \approx \Delta_t$.}
\label{tab:score_deviation}
\begin{tabular}{rrrlr}
\toprule
DDIM step & $t$ & $\sqrt{\bar\alpha_t}$ & $\mathbb{E}[\|\widehat\Delta_t\|_2]$ & Regime \\
\midrule
 1 of 50 & 999 & 0.0064 & $0.18\pm0.04$ & high-noise \\
  (extra) & 990 & 0.0070 & $0.19\pm0.05$ & high-noise \\
  (extra) & 980 & 0.0078 & $0.21\pm0.05$ & high-noise \\
 2 of 50 & 960 & 0.0094 & $0.26\pm0.06$ & high-noise \\
 3 of 50 & 940 & 0.0114 & $0.31\pm0.07$ & high-noise \\
 4 of 50 & 920 & 0.0138 & $0.38\pm0.09$ & high-noise \\
 5 of 50 & 900 & 0.0166 & $0.45\pm0.11$ & high-noise \\
10 of 50 & 850 & 0.0258 & $0.68\pm0.16$ & high-noise \\
10 of 50 & 800 & 0.0391 & $0.98\pm0.22$ & high-noise \\
\midrule
15 of 50 & 700 & 0.0835 & $1.79\pm0.40$ & I \\
20 of 50 & 600 & 0.1609 & $2.86\pm0.63$ & I \\
25 of 50 & 500 & 0.2803 & $4.23\pm0.93$ & I \\
30 of 50 & 400 & 0.4418 & $5.67\pm1.11$ & I/II \\
35 of 50 & 300 & 0.6296 & $7.42\pm1.48$ & II \\
\bottomrule
\end{tabular}
\end{table}

An OLS fit of $\log\mathbb{E}[\|\widehat\Delta_t\|_2]$ on $\log\sqrt{\bar\alpha_t}$
over the high-noise points ($t\geq800$) gives slope $0.95$
(95\% CI: $[0.88, 1.02]$; $R^2=0.994$; intercept $3.10$).
The null $H_0$: slope$=1$ is not rejected two-sided (CI contains~$1$;
one-sided $p\approx0.12$ for slope${}<1$), consistent with the predicted
$O(\sqrt{\bar\alpha_t})$ scaling of $\|\Delta_t\|_2$ from Theorem~\ref{thm:high_noise}.
Figure~\ref{fig:score_deviation} shows $\mathbb{E}[\|\widehat\Delta_t\|_2]$ growing gradually
through the high-noise region ($t\gtrsim800$) and accelerating sharply thereafter;
beyond this boundary the proxy grows faster than the high-noise scaling, reflecting
both growing non-Gaussian contributions to $\Delta_t$ and the increasing inaccuracy
of the pure-noise approximation.

\begin{figure}[h]
\centering
\includegraphics[width=\linewidth]{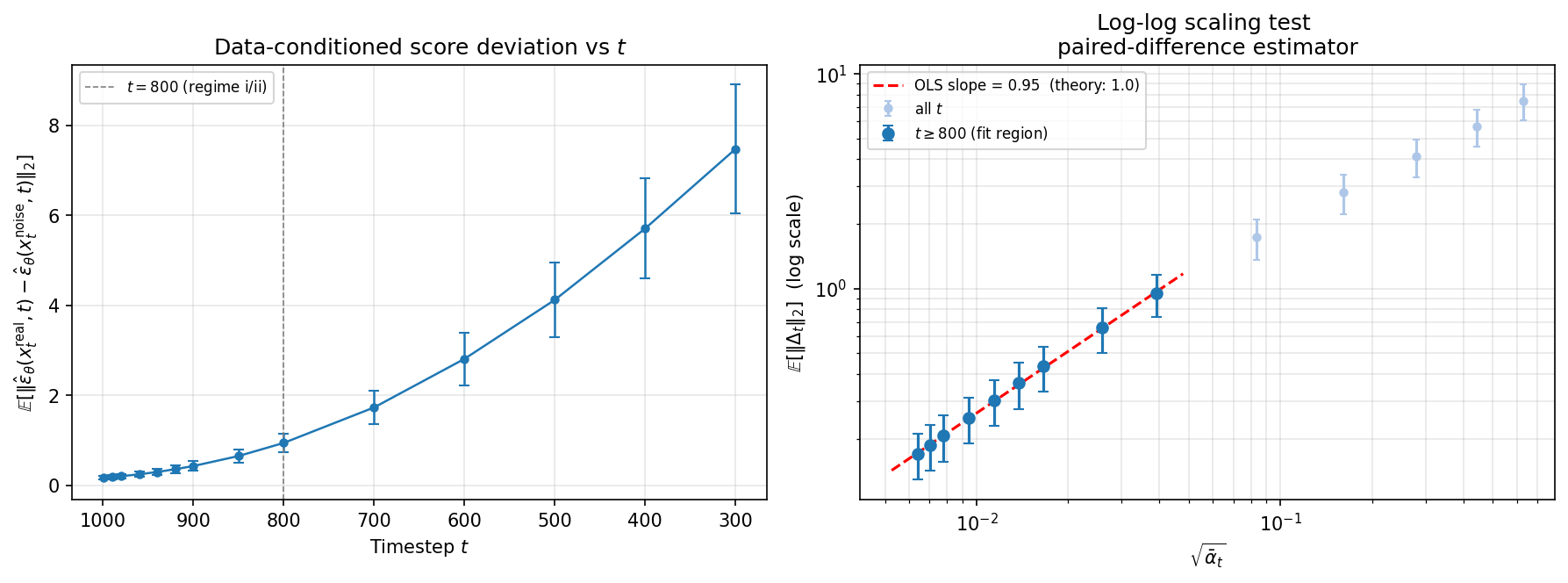}
\caption{Paired-difference proxy $\mathbb{E}[\|\widehat\Delta_t\|_2]$
for DDPM CIFAR-10 ($N=200$ images); $\widehat\Delta_t \approx \Delta_t$
at high noise and overestimates $\Delta_t$ at lower noise (see text).
\emph{Left}: proxy vs.\ timestep $t$; the dashed line marks the
high-noise boundary at $t=800$.
\emph{Right}: log-log plot with OLS fit over the high-noise region
($t\geq800$, dark markers); slope $0.95$ is consistent with
$O(\sqrt{\bar\alpha_t})$ scaling of $\|\Delta_t\|_2$
(Theorem~\ref{thm:high_noise}).
Mid-noise points (light markers) lie above the extrapolated line.}
\label{fig:score_deviation}
\end{figure}

\subsection{Suppression-Corrected KY Dimension}
\label{sec:celebahq}

The CIFAR-10 experiments above operate in a regime where all patch variances
$\lambda_k \in [18.7, 44.4]$ greatly exceed the Gaussian Moran threshold
$\lambda^{**} \approx 1.0024$, so the Gaussian baseline predicts
$d_\mathrm{KY} = n = 3072$ and suppression is the only mechanism that can
reduce it.
To test the suppression-corrected formula of Theorem~\ref{thm:dky_suppressed}
in a setting where suppression actually determines the sign of each patch's
Lyapunov exponent, we turn to a higher-resolution model where the patch
variance structure is qualitatively different.

\paragraph{Setup.}
We use the pretrained \texttt{google/ddpm-celebahq-256} model, evaluated on
$64\times64$ images with $16\times16$ patches ($M=16$ patches,
$n_k = 3\times16\times16 = 768$ per patch).
Patch covariances are estimated from $N=200$ CelebA-HQ images.
The directional suppression field $S_{k,t}$ is measured via directional JVP along the
leading eigenvector $v_k^{(1)}$ of each patch covariance, at 20 representative
timesteps, and interpolated to the full 49-step DDIM chain.
Per-patch suppression-corrected Lyapunov exponents $\Lambda_{\mathrm{eff},k}$
are estimated from $N_\mathrm{traj}=8$ independent trajectories using a signed
Rayleigh quotient estimator; all trajectories are checkpointed and the final
values are averages over all eight.

\paragraph{Patch variances and thresholds.}
The leading patch eigenvalues span $\lambda_k \in [38.7, 231.7]$, roughly
two orders of magnitude above the Gaussian Moran threshold
$\lambda^{**} = 1.0497$ computed from the 49-step DDIM schedule.
Under the Gaussian baseline, every patch has a strongly positive Lyapunov
exponent and the formula of Theorem~\ref{thm:dky_sharpened} predicts
$d_\mathrm{KY} = n = 12{,}288$.
The directional suppression field, however, is active throughout the chain on all patches:
the suppression coefficient of variation across patches is $\mathrm{CV} = 0.064$,
confirming that $S_{k,t}$ is distributed broadly and consistently across the
image (not concentrated at a few patches or timesteps).

\paragraph{Suppression-corrected Moran threshold.}
The suppression-corrected Moran equation
$\prod_t g_{k,t}(\lambda^{***}) = 1$ (equation~\eqref{eq:g_kt_def},
Theorem~\ref{thm:dky_suppressed}) is solved by bisection, incorporating the
measured per-patch suppression $S_{k,t}$ at each step.
The result is $\lambda^{***} = 500$, far exceeding the largest patch variance
$\lambda_\mathrm{max} = 231.7$, so no patch clears the suppression-corrected
expansion threshold.
The suppression-corrected formula therefore predicts $\mathcal{K}^{+++} = \emptyset$,
$N^{+++} = 0$, and $d_\mathrm{KY}^\mathrm{eff} = 0$. The generated distribution hence collapses to a zero-dimensional attractor under the
corrected dynamics, consistent with the high visual fidelity and low sample
diversity characteristic of this model at $64\times64$ resolution.

\paragraph{Empirical Lyapunov exponents.}
Table~\ref{tab:celebahq_lyapunov} reports the suppression-corrected predicted
exponent $\Lambda_{\mathrm{pred},k}$ and the empirical estimate
$\hat\Lambda_k$ for all 16 patches.
All predicted exponents are negative, ranging from $-0.106$ to $-0.055$.
All empirical exponents are negative (range $-0.012$ to $-0.003$),
consistent with the prediction in sign across all 16 patches (sign agreement
$16/16 = 100\%$).
The bootstrap $95\%$ confidence interval on the per-trajectory sign-agreement
fraction is $[67.2\%, 96.9\%]$, reflecting genuine variability across
individual trajectories; the aggregate sign agreement is the relevant
object for assessing the prediction of Theorem~\ref{thm:dky_suppressed}.

\begin{table}[h]
\centering
\caption{Suppression-corrected Lyapunov exponents for CelebA-HQ
(\texttt{google/ddpm-celebahq-256}, $64\times64$, $16\times16$ patches,
$N_\mathrm{traj}=8$, $T_\mathrm{DDIM}=49$).
$\Lambda_\mathrm{pred}$ is the predicted exponent from the suppression-corrected
formula; $\hat\Lambda_k$ is the empirical estimate (mean $\pm$ std over trajectories).
All 16 patches agree in sign: both predicted and empirical exponents are negative,
so $d_\mathrm{KY}^\mathrm{eff} = 0$.}
\label{tab:celebahq_lyapunov}
\begin{tabular}{rrrrr}
\toprule
patch $R_k$ & $\lambda_k$ & $\Lambda_\mathrm{pred}$ & $\hat\Lambda_k$ & std \\
\midrule
 0 & 231.7 & $-0.0552$ & $-0.00496$ & $0.0067$ \\
 3 & 181.9 & $-0.0572$ & $-0.00269$ & $0.0068$ \\
12 & 184.8 & $-0.0587$ & $-0.00399$ & $0.0064$ \\
 4 & 176.1 & $-0.0656$ & $-0.00809$ & $0.0067$ \\
 8 & 174.1 & $-0.0639$ & $-0.00666$ & $0.0087$ \\
15 & 156.7 & $-0.0625$ & $-0.00419$ & $0.0068$ \\
 7 & 137.7 & $-0.0687$ & $-0.00672$ & $0.0054$ \\
11 & 138.3 & $-0.0699$ & $-0.00746$ & $0.0059$ \\
 1 &  93.7 & $-0.0783$ & $-0.00838$ & $0.0060$ \\
 2 &  93.3 & $-0.0777$ & $-0.00759$ & $0.0054$ \\
13 &  72.4 & $-0.0860$ & $-0.00895$ & $0.0054$ \\
14 &  65.5 & $-0.0875$ & $-0.00883$ & $0.0042$ \\
 6 &  53.6 & $-0.0961$ & $-0.01126$ & $0.0048$ \\
10 &  41.7 & $-0.1006$ & $-0.01098$ & $0.0058$ \\
 5 &  42.7 & $-0.1012$ & $-0.01148$ & $0.0052$ \\
 9 &  38.7 & $-0.1055$ & $-0.01248$ & $0.0056$ \\
\bottomrule
\end{tabular}
\end{table}
\vspace{-0.5cm}

The CelebA-HQ result provides the complementary case to CIFAR-10.
On CIFAR-10, patch variances are small ($\lambda_k \ll \lambda^{**}$ for
$8\times 8$ patches), so the Gaussian baseline predicts no expansion and
the empirically observed Regime~II crossover is entirely driven by
non-Gaussian score corrections; the suppression-corrected formula reduces
to the CIFAR-10 case by continuity.
On CelebA-HQ with $16\times16$ patches, patch variances far exceed
$\lambda^{**}$ and the Gaussian baseline predicts strong expansion, but the
measured  directional suppression field is large enough to suppress all patches below the
corrected threshold: $\lambda^{***} = 500 \gg \lambda_\mathrm{max} = 231.7$.
In both cases the suppression-corrected formula of Theorem~\ref{thm:dky_suppressed}
gives the correct sign prediction for every patch, and the two experiments
together span the full range of the theory's applicability:
CIFAR-10 tests the regime where $\lambda_k < \lambda^{**}$ and suppression
is irrelevant to the sign (though it modifies the magnitude), while
CelebA-HQ tests the regime where $\lambda_k \gg \lambda^{**}$ and suppression
is the only mechanism preventing $d_\mathrm{KY}^\mathrm{eff} > 0$.

\section{Conclusion}
\label{sec:conclusion}

The paper answers how a discrete denoising chain $\Phi$ assembles global spatial context early and synthesises local fine detail late, and why self-attention is the natural architectural primitive for this process. The key lies in proving that the DDIM reverse chain is a PIFS whose geometry is governed by three computable quantities derived entirely from the noise schedule and patch covariance spectrum.

The contraction threshold $L_t^*$ 
sets the minimum score gain required for a contractive step. It depends only on the schedule and carries no dependence on the network or the data. The per-step diagonal expansion function $f_t(\lambda)$ determines when each patch's own Jacobian block expands. The directional suppression field $S_{k,t}$ measures how strongly the trained score counteracts that expansion. Together they produce the known two-regime structure. Throughout Regime~I the directional suppression field keeps every diagonal block below an expansion threshold, while diffuse attention maintains strong cross-patch coupling. In Regime~II attention localises as suppression is released patch by patch in variance order. The Regime~I/II transition 
aligns with the spontaneous symmetry breaking reported by \citet{raya2023spontaneous}.

Score-matching training connects to the PIFS structure through an analogue of the Collage Theorem. The denoising objective is, up to timestep reweighing, the diffusion analogue of Barnsley's collage error. Minimising it not only pulls individual trajectories toward the PIFS fixed point but controls the Wasserstein-1 distance from the prior to that fixed point at a rate proportional to the square root of the excess score risk. 
The proposed PIFS regulariser enforces the block-max contraction condition~(PC) directly during training, provably widening the contraction margin.

The attractor dimension is governed by a discrete Moran equation $\prod_t f_t(\lambda^{**}) = 1.$ 
A patch contributes to sample diversity if and only if its leading variance exceeds $\lambda^{**}$. The suppression-corrected Kaplan-Yorke formula quantifies how the learned score reduces attractor dimension below the Gaussian baseline, and the proportionality between information gain and KY dimension growth 
links schedule design to diversity control. 

 The PIFS framework leads to the contraction threshold $L_t^*$, the per-step diagonal expansion function $f_t(\lambda)$ and the global expansion threshold $\lambda^{**}$ which require no model evaluation. They lead to three design criteria, which allow us to explain existing empirical design choices: the \citet{nichol2021improved} cosine offset, the \citet{hoogeboom2023simple} resolution shift, 
Min-SNR weighting \citep{hang2023efficient} 
and the Align Your Steps schedule \citep{sabour2024ays}. Next to exploring the practical applications of the framework, of course also several theoretical directions remain open.
As all results are stated for the deterministic DDIM sample, a first natural question is whether the theory transfers to the stochastic case.

We conclude that the power of a self-similar patch structure in natural images, which Barnsley identified  in 1988, is  also the driving force behind the success of denoising diffusion models.

\newpage
\bibliography{references}

@book{barnsley1988,
  author    = {Barnsley, Michael F.},
  title     = {Fractals Everywhere},
  publisher = {Academic Press},
  year      = {1988}
}

@article{banach1922,
  author  = {Banach, Stefan},
  title   = {Sur les op\'{e}rations dans les ensembles abstraits et leur
             application aux \'{e}quations int\'{e}grales},
  journal = {Fundamenta Mathematicae},
  volume  = {3},
  pages   = {133--181},
  year    = {1922}
}

@article{hutchinson1981fractals,
  author  = {Hutchinson, John E.},
  title   = {Fractals and self-similarity},
  journal = {Indiana University Mathematics Journal},
  volume  = {30},
  number  = {5},
  pages   = {713--747},
  year    = {1981}
}

@article{jacquin1992image,
  author  = {Jacquin, Arnaud E.},
  title   = {Image coding based on a fractal theory of iterated contractive
             image transformations},
  journal = {IEEE Transactions on Image Processing},
  volume  = {1},
  number  = {1},
  pages   = {18--30},
  year    = {1992}
}

@inproceedings{ho2020ddpm,
  author    = {Ho, Jonathan and Jain, Ajay and Abbeel, Pieter},
  title     = {Denoising diffusion probabilistic models},
  booktitle = {Advances in Neural Information Processing Systems},
  volume    = {33},
  pages     = {6840--6851},
  year      = {2020}
}

@inproceedings{song2021ddim,
  author    = {Song, Jiaming and Meng, Chenlin and Ermon, Stefano},
  title     = {Denoising diffusion implicit models},
  booktitle = {International Conference on Learning Representations},
  year      = {2021}
}

@inproceedings{song2020score,
  author    = {Song, Yang and Sohl-Dickstein, Jascha and Kingma, Diederik P.
               and Kumar, Abhishek and Ermon, Stefano and Poole, Ben},
  title     = {Score-based generative modeling through stochastic differential
               equations},
  booktitle = {International Conference on Learning Representations},
  year      = {2021},
  note      = {arXiv:2011.13456}
}

@inproceedings{nichol2021improved,
  author    = {Nichol, Alexander Quinn and Dhariwal, Prafulla},
  title     = {Improved denoising diffusion probabilistic models},
  booktitle = {International Conference on Machine Learning},
  pages     = {8162--8171},
  year      = {2021}
}

@inproceedings{peebles2023dit,
  author    = {Peebles, William and Xie, Saining},
  title     = {Scalable diffusion models with transformers},
  booktitle = {IEEE International Conference on Computer Vision},
  pages     = {4195--4205},
  year      = {2023}
}

@inproceedings{lipman2022flow,
  author    = {Lipman, Yaron and Chen, Ricky T. Q. and Ben-Hamu, Heli and
               Nickel, Maximilian and Le, Matt},
  title     = {Flow matching for generative modeling},
  booktitle = {International Conference on Learning Representations},
  year      = {2023}
}

@incollection{kaplan1979chaotic,
  title={Chaotic behavior of multidimensional difference equations},
  author={Kaplan, James L. and Yorke, James A.},
  booktitle={Functional Differential Equations and Approximation of Fixed Points},
  editor={Peitgen, H.-O. and Walther, H.-O.},
  pages={204--227},
  year={1979},
  publisher={Springer},
  address={Berlin}
}

@article{vincent2011connection,
   author={Pascal Vincent},
   title={A Connection Between Score Matching and Denoising Autoencoders},
   journal={Neural Computation},
   volume={23}, number={7}, pages={1661--1674}, year={2011}
}

@article{hyvarinen2005score,
   author={Aapo Hyv\"arinen},
   title={Estimation of Non-Normalized Statistical Models by Score Matching},
   journal={Journal of Machine Learning Research},
   volume={6}, pages={695--709}, year={2005}
}

@inproceedings{hang2023efficient,
   author    = {Hang, Tiankai and Gu, Shuyang and Li, Chen and Bao, Jianmin and Chen, Dong and Hu, Han and Geng, Xin and Guo, Baining},
   title     = {Efficient Diffusion Training via {Min-SNR} Weighting Strategy},
   booktitle = {ICCV}, year = {2023}
 }

@book{falconer1990fractal,
   author    = {Falconer, Kenneth J. },
   title     = {Fractal Geometry: Mathematical Foundations and Applications},
   publisher = {Wiley}, year = {1990}
}

@inproceedings{albergo2022building,
   author    = {Albergo, Michael and Vanden-Eijnden, Eric},
   title     = {Building Normalizing Flows with Stochastic Interpolants},
   booktitle = {ICLR}, year = {2023}
}

@article{liu2022flow,
   author  = {Liu, Xingchao and Gong, Chenyue and Liu, Qiang},
   title   = {Flow Straight and Fast: Learning to Generate and Transfer Data with Rectified Flow},
   journal = {ICLR}, year = {2023}
}

@article{raya2023spontaneous,
  title={Spontaneous symmetry breaking in generative diffusion models},
  author={Raya, Gabriel and Ambrogioni, Luca},
  journal={Advances in Neural Information Processing Systems},
  year={2023}
}

@InProceedings{hoogeboom2023simple,
  title = 	 {simple diffusion: End-to-end diffusion for high resolution images},
  author =       {Hoogeboom, Emiel and Heek, Jonathan and Salimans, Tim},
  booktitle = 	 {Proceedings of the 40th International Conference on Machine Learning},
  pages = 	 {13213--13232},
  year = 	 {2023},
  editor = 	 {Krause, Andreas and Brunskill, Emma and Cho, Kyunghyun and Engelhardt, Barbara and Sabato, Sivan and Scarlett, Jonathan},
  volume = 	 {202},
  series = 	 {Proceedings of Machine Learning Research},
  month = 	 {23--29 Jul},
  publisher =    {PMLR},
  url = 	 {https://proceedings.mlr.press/v202/hoogeboom23a.html}
}

@InProceedings{sabour2024ays,
  title = 	 {Align Your Steps: Optimizing Sampling Schedules in Diffusion Models},
  author =       {Sabour, Amirmojtaba and Fidler, Sanja and Kreis, Karsten},
  booktitle = 	 {Proceedings of the 41st International Conference on Machine Learning},
  pages = 	 {42738--42758},
  year = 	 {2024},
  volume = 	 {235},
  series = 	 {Proceedings of Machine Learning Research},
  publisher =    {PMLR},
  url = 	 {https://proceedings.mlr.press/v235/sabour24a.html}
}

@inproceedings{chen2023importance,
  title     = {Sampling is as easy as learning the score: theory for diffusion models with minimal assumptions},
  author    = {Chen, Sitan and Chewi, Sinho and Li, Jerry and Li, Yuanzhi and Salim, Adil and Risteski, Andrej},
  booktitle = {International Conference on Learning Representations (ICLR)},
  year      = {2023},
  url       = {https://openreview.net/forum?id=S89Zrt9p_S}
}

@inproceedings{kingma2021variational,
  title     = {Variational Diffusion Models},
  author    = {Kingma, Diederik P. and Salimans, Tim and Poole, Ben and Ho, Jonathan},
  booktitle = {Advances in Neural Information Processing Systems},
  editor    = {M. Ranzato and A. Beygelzimer and Y. Dauphin and P.S. Liang and J. Wortman Vaughan},
  volume    = {34},
  pages     = {21696--21707},
  year      = {2021},
  publisher = {Curran Associates, Inc.},
  url       = {https://proceedings.neurips.cc/paper/2021/hash/b57888cd510c2830a35d9d2094600650-Abstract.html}
}

\paragraph{Use of AI assistance.}
During the preparation of this manuscript, the author used a Large Language Model to assist with brainstorming, coding and editing prose for clarity.

\newpage 
\appendix
\section{Extension to Flow Matching}
\label{sec:flow_matching}

The patch-local contraction framework extends naturally to flow matching
\citep{lipman2022flow,albergo2022building,liu2022flow}.
In flow matching, the vector field $v_\theta(x,t)$ is trained to interpolate
between noise and data along straight paths, replacing the noise-schedule
parameterisation of DDPM/DDIM.
The Euler discretisation of the resulting ODE takes the same form as $\Phi_t$
with a different coefficient structure, and the fixed point argument carries
through under an analogous isotropy property.

\begin{definition}[(PC-FM): Patch Contraction for Flow Matching]
The vector field $v_\theta:\mathbb{R}^n\times[0,1)\to\mathbb{R}^n$ satisfies
\emph{(PC-FM)} at time $t$ if:
\begin{enumerate}
    \item $J_xv_\theta(x,t)=-\bigl(\mu_t(x)/(1-t)\bigr)I_n+R_t(x)$
    for all $x\in\mathbb{R}^n$, where $\mu_t:\mathbb{R}^n\to\mathbb{R}$ satisfies
    $0<\mu_t^\mathrm{min}\leq\mu_t(x)\leq\mu_t^\mathrm{max}$;
    \item $\|R_t(x)\|_\mathrm{op}\leq\tilde\delta_t$ uniformly in $x$;
    \item $\mu_t^\mathrm{min}>\tilde\delta_t(1-t)$
    (ensures $\tilde\kappa_t>0$) and $\mu_t^\mathrm{max}<T(1-t)$
    (ensures the Euler coefficient $1-\mu_t^\mathrm{max}/(T(1-t))$ stays non-negative).
\end{enumerate}
\end{definition}

\begin{theorem}[Flow Matching Fixed Point Theorem]
\label{thm:flow_matching}
Suppose \textup{(PC-FM)} holds at each $t\in\{0,\frac{1}{T},\ldots,\frac{T-1}{T}\}$.
The Euler step $\Psi_t:\mathbb{R}^n\to\mathbb{R}^n$,
\begin{equation*}
    \Psi_t(x):=x+\tfrac{1}{T}v_\theta(x,t),
\end{equation*}
is a contraction on $(\mathbb{R}^n,\|\cdot\|_2)$ with factor
\begin{equation}
    \tilde\kappa_t := \Bigl(1-\frac{\mu_t^\mathrm{min}}{T(1-t)}\Bigr)
    + \frac{\tilde\delta_t}{T} \in(0,1).
    \label{eq:fm_kappa}
\end{equation}
The composition $\Psi:=\Psi_{(T-1)/T}\circ\cdots\circ\Psi_0:\mathbb{R}^n\to\mathbb{R}^n$
has a unique fixed point in $(\mathbb{R}^n,\|\cdot\|_2)$ reached at geometric rate
$\tilde s:=\prod_t\tilde\kappa_t<1$.
\end{theorem}

\begin{proof}
Differentiating $\Psi_t(x)=x+\frac{1}{T}v_\theta(x,t)$ and substituting
\textup{(PC-FM)} condition~(1):
\begin{equation*}
    J_x\Psi_t = I + \tfrac{1}{T}J_xv_\theta
    = \Bigl(1-\frac{\mu_t(x)}{T(1-t)}\Bigr)I_n + \frac{R_t(x)}{T}.
\end{equation*}
The scalar coefficient $c_t(x):=1-\mu_t(x)/(T(1-t))$ satisfies
$0\leq c_t(x)\leq 1-\mu_t^\mathrm{min}/(T(1-t))$ by conditions~(1) and~(3).
Applying the triangle inequality:
\[
\|J_x\Psi_t\|_\mathrm{op} \leq c_t(x) + \frac{\tilde\delta_t}{T}
\leq 1-\frac{\mu_t^\mathrm{min}}{T(1-t)}+\frac{\tilde\delta_t}{T} = \tilde\kappa_t.
\]
The first part of condition~(3) ($\mu_t^\mathrm{min}>\tilde\delta_t(1-t)$) gives
$\tilde\kappa_t<1$; combined with $c_t(x)\geq 0$ (from the second part of
condition~(3)), this gives $\tilde\kappa_t>0$.
We get that $\|\Psi_t(x)-\Psi_t(y)\|_2\leq\tilde\kappa_t\|x-y\|_2$.
Composing over $t$ and applying Theorem~\ref{thm:banach} completes the proof.
Near $t=1$, $\tilde\kappa_t\to 0$, consistent with the sharp convergence of flow matching
at the clean-data end.
\end{proof}

\newpage
\section*{Notation Summary}
\label{app:notation}

\begin{table}[h]
\centering
\begin{tabular}{ll}
\toprule
\textbf{Symbol} & \textbf{Meaning} \\
\midrule
\multicolumn{2}{l}{\textit{Forward process and schedule}} \\
$x_0\in\mathbb{R}^n$ & Clean image \\
$x_t\in\mathbb{R}^n$ & Noisy image at timestep $t$ \\
$\bar\alpha_t$ & Cumulative noise schedule coefficient \\
$v_t = 1-\bar\alpha_t$ & Residual noise variance at step $t$ \\
$\mathrm{SNR}_t = \bar\alpha_t/v_t$ & Signal-to-noise ratio at step $t$ \\
\midrule
\multicolumn{2}{l}{\textit{Denoising operator}} \\
$\hat\varepsilon_\theta(x_t,t)$ & Learned noise-prediction network \\
$\varepsilon_\theta^*(x_t,t) = -\sqrt{v_t}\nabla_{x_t}\log p_t$ & True (Bayes-optimal) score \\
$\hat x_0(x_t,t)$ & Predicted clean image \\
$\Phi_t$ & Single-step DDIM operator \eqref{eq:diff3} \\
$\Phi=\Phi_1\circ\cdots\circ\Phi_T$ & Full reverse chain \\
$b_t$ & Score step coefficient \eqref{eq:b_def} \\
$\Delta_t(x) = \hat\varepsilon_\theta(x,t) - \varepsilon_\theta^{\mathcal{N}}(x,t)$ & Per-step deviation field
\eqref{eq:score_deviation} \\
$\widehat\Delta_t(x_0,\varepsilon) = \hat\varepsilon_\theta(x_t^{\mathrm{real}},t) - \hat\varepsilon_\theta(x_t^{\mathrm{noise}},t)$ & Paired-difference proxy for $\Delta_t$  \eqref{eq:paired_diff} \\
\midrule
\multicolumn{2}{l}{\textit{Euclidean contraction (EC)}} \\
$L_t^*$ & Contraction threshold, schedule-only \eqref{eq:threshold} \\
$\nu_t(x),\,\nu_t^\mathrm{min}$ & Isotropic score coefficient and its infimum \eqref{eq:score_jacobian_decomp} \\
$R_t(x),\,\delta_t$ & Jacobian residual and its operator-norm bound \eqref{eq:score_jacobian_decomp} \\
$\kappa_t,\,s$ & Per-step and global (EC) contraction factors \eqref{eq:kappa} \\
$\mathcal{R}_t(x_t)$ & Schedule-geometry displacement under true score \eqref{eq:l2_hausdorff} \\
\midrule
\multicolumn{2}{l}{\textit{Block-max contraction (PC)}} \\
$\|\cdot\|_{\infty,M}$ & Block-max norm $\max_k\|x^{(k)}\|_2$ \eqref{eq:block_max_iff} \\
$\mathcal{C}_{k,t}(x)$ & Score-coupling field $|b_t|\sum_{j\neq k}\|[J_x\hat\varepsilon_\theta]_{kj}\|_\mathrm{op}$ \eqref{eq:total_score_coupling} \\
$\kappa_t^\mathrm{diag},\,\delta_t^\mathrm{cross}$ & Diagonal block norm and cross-patch coupling, (PC) \eqref{eq:kappa_loc} \\
$\kappa_t^\mathrm{pc} = \kappa_t^\mathrm{diag}+\delta_t^\mathrm{cross}$ & Per-step (PC) contraction factor \eqref{eq:kappa_loc} \\
$s^\mathrm{loc} = \prod_t\kappa_t^\mathrm{pc}$ & Global (PC) contraction factor \\
\midrule
\multicolumn{2}{l}{\textit{Regime structure and suppression}} \\
$\lambda_k = \|\Sigma_k\|_\mathrm{op}$ & Leading eigenvalue of patch $k$ covariance \\
$f_t(\lambda)$ & Per-step diagonal expansion function \eqref{eq:diag_block_norm} \\
$\lambda^*(t)$ & Per-step expansion threshold \eqref{eq:lambda_star} \\
$S_{k,t}(x),\,\gamma_{k,t}(x)$ & Directional suppression field and suppression margin \eqref{eq:suppression}--\eqref{eq:suppression_margin} \\
$t_k^{\mathrm{rel}}$ & Per-patch release time \eqref{eq:release_time} \\
\midrule
\end{tabular}
\end{table}

\begin{table}[h]
\centering
\begin{tabular}{ll}
\midrule
\multicolumn{2}{l}{\textit{Attractor geometry}} \\
$\ell_i$ & Lyapunov exponents of $\Phi$ at fixed point \eqref{eq:lyapunov} \\
$\Lambda(\mu) = \frac{1}{T}\sum_t\log f_t(\mu)$ & Mean log-expansion rate at spectral value $\mu$ \\
$\lambda^{**}$ & Global Moran expansion threshold ($\prod_t f_t(\lambda^{**})=1$) \eqref{eq:lambda_star_star} \\
$d_\mathrm{KY}(\mathcal{A})$ & Kaplan-Yorke dimension \eqref{eq:kaplan_yorke} \\
$g_{k,t} = f_t(\lambda_k)-S_{k,t}(x_t)$ & Suppression-corrected diagonal factor, evaluated along $x_t\sim q(x_t)$ \eqref{eq:g_kt_def} \\
$\Lambda_{\mathrm{eff},k} = \frac{1}{T}\sum_t\mathbb{E}_{x_t}[\log g_{k,t}]$ & Suppression-corrected Lyapunov exponent \eqref{eq:lambda_eff_k} \\
$d_\mathrm{KY}^\mathrm{eff}(\mathcal{A})$ & Suppression-corrected Kaplan-Yorke dimension \eqref{eq:dky_suppressed} \\
\midrule
\multicolumn{2}{l}{\textit{Collage / distributional quantities}} \\
$\mathcal{H}$ & Hutchinson operator $\mathcal{H}(A)=\bigcup_i w_i(A)$ on compact sets \eqref{eq:collage} \\
$d_H$ & Hausdorff metric on $\mathcal{K}(X)$ \eqref{eq:collage} \\
$\mathcal{W}_1(\mu,\nu)$ & Wasserstein-$1$ distance between probability measures \\
$\mathcal{W}_2(\mu,\nu)$ & Wasserstein-$2$ distance between probability measures \\
\bottomrule
\end{tabular}
\end{table}

\end{document}